\documentclass[12pt]{article}

\usepackage{fullpage}
\usepackage{amsfonts}
\usepackage{amssymb}
\usepackage{amstext}
\usepackage{amsmath}
\usepackage{natbib}
\usepackage{pstricks,pst-plot,pst-node}
\usepackage{graphicx}
\usepackage{booktabs}
\usepackage{array}
\usepackage{multirow}

\bibpunct{(}{)}{;}{a}{,}{,}
  
\usepackage{jstyle}
\usepackage{graphicx}
\usepackage{epsfig}

\def\reals{{\mathbb R}}
\def\P{{\mathbb P}}
\def\E{{\mathbb E}}
\def\supp{\mathop{\text{supp}\kern.2ex}}
\def\argmin{\mathop{\text{arg\,min}\kern.2ex}}
\let\hat\widehat
\let\tilde\widetilde

\def\argmin{\mathop{\text{arg\,min}}}

\let\hat\widehat
\let\tilde\widetilde
\def\given{{\,|\,}}
\def\ds{\displaystyle}

\def\1{{(1)}}
\def\2{{(2)}}
\def\pn{{(n)}}
\def\ip{{(i)}}
\def\Xbar{\overline{X}}
\def\except{\backslash}
\def\npn{\mathop{\textit{NPN\,}}}

\def\i{{(i)}}
\def\cE{{\mathcal{C}}}
\def\cM{{\mathcal{M}}}
\def\cF{{\mathcal{F}}}
\def\cP{{\mathcal{P}}}
\def\cG{{\mathcal{G}}}
\def\M{{\mathcal{M}}}
\def\tr{\mathop{\text{tr}}}
\long\def\comment#1{}

\begin{document}

\mbox{\ }
\vskip.0in
\begin{center}
\Large\bf 
The Nonparanormal: Semiparametric Estimation of High Dimensional Undirected Graphs
\end{center}
\vskip.1in
\begin{center}
\begin{tabular}{c}
{\large Han Liu, \ John Lafferty, \  Larry Wasserman} \\[15pt]
Department of Statistics and \\
School of Computer Science\\
Carnegie Mellon University \\
Pittsburgh, PA 15213 USA \\[20pt]
\today \\[20pt]
\end{tabular}
\end{center}

\begin{abstract}
\noindent\normalsize 
Recent methods for estimating sparse undirected graphs for
real-valued data in high dimensional problems rely heavily on the assumption of normality.
We show how to use a semiparametric Gaussian copula---or
``nonparanormal''---for high dimensional inference.  Just as additive
models extend linear models by replacing linear functions with a set
of one-dimensional smooth functions, the nonparanormal extends the
normal by transforming the variables by smooth functions.  We derive a
method for estimating the nonparanormal, study the method's
theoretical properties, and show that it works well in many examples.
\vskip30pt
\noindent
\begin{quote}
\begin{itemize}
\item[\bf Keywords:] Graphical models, Gaussian copula,  high dimensional inference, sparsity, $\ell_1$ regularization, graphical lasso,
paranormal, occult
\end{itemize}
\end{quote}
\end{abstract}


\newpage

\section{Introduction}

The linear model is a mainstay of statistical inference that has been extended
in several important ways.  An extension to high dimensions was achieved by 
adding a sparsity constraint, leading to the lasso
\citep{Tibshirani:96}.  An extension to nonparametric models was achieved 
by replacing linear functions with smooth functions,
leading to additive models \citep{Hast:Tibs:1999}.
These two ideas were recently combined, leading to an extension called
sparse additive models (SpAM)  \citep{Ravikumar:07,Ravikumar:08}.
In this paper we consider a similar nonparametric extension of
undirected graphical models based on multivariate Gaussian
distributions in the high dimensional setting.
Specifically, we use a high dimensional Gaussian copula with nonparametric marginals,
which we refer to as a nonparanormal distribution.

If $X$ is a $p$-dimensional random vector distributed according to
a multivariate Gaussian distribution with covariance matrix
$\Sigma$, the conditional independence relations between the random variables
$X_1, X_2, \ldots, X_p$ are encoded in a graph formed
from the precision matrix $\Omega = \Sigma^{-1}$.
Specifically, missing edges in the graph 
correspond to zeroes of $\Omega$.  To estimate the graph 
from a sample of size $n$, it is only necessary to estimate $\Sigma$, 
which is easy if $n$ is much larger than $p$.
However, when $p$ is larger than $n$, the problem is more challenging.
Recent work has focused on the problem of estimating the graph
in this high dimensional setting, which becomes feasible if $G$ is
sparse.  
\cite{Yuan:Lin:07} and
\cite{Banerjee:08} propose
an estimator based on regularized maximum likelihood using 
an $\ell_1$ constraint on the entries of $\Omega$, and \cite{FHT:07} develop 
an efficient algorithm for computing the estimator using a graphical
version of the lasso.   The resulting estimation procedure
has excellent theoretical properties, as shown
recently by \cite{Rothman:08} and \cite{Ravikumar:Gauss:09}.

While Gaussian graphical models can be useful, a reliance on exact
normality is limiting.  Our goal in this paper is to weaken this
assumption.  Our approach parallels the ideas behind sparse additive
models for regression \citep{Ravikumar:07,Ravikumar:08}.  
Specifically, we replace the Gaussian with a semiparametric Gaussian copula.
This means that we replace the random variable $X=(X_1,\ldots, X_p)$ by
the transformed random variable $f(X) = \left(f_1(X_1), \ldots,
f_p(X_p)\right)$, and assume that $f(X)$ is multivariate Gaussian.
This semiparametric copula results in a nonparametric extension of the normal that we call
the \textit{nonparanormal} distribution.  The nonparanormal
depends on the functions $\{f_j\}$,
and a mean $\mu$ and covariance matrix $\Sigma$, all of which
are to be estimated from data.  While the resulting family of
distributions is much richer than the standard parametric normal (the
paranormal), the independence relations among the variables are
still encoded in the precision matrix $\Omega = \Sigma^{-1}$.  
We propose a nonparametric estimator for the functions $\{f_j\}$, and show
how the graphical lasso can be used to estimate the graph
in the high dimensional setting.
The relationship between linear regression models, Gaussian graphical
models, and their extensions to nonparametric and high dimensional
models is summarized in Figure~\ref{fig::compare}.

Most theoretical results on semiparametric copulas
focus on low or at least finite dimensional models \citep{Tsukahara:05}.
Models with increasing dimension require
a more delicate analysis; in particular, simply plugging in the usual empirical distribution
of the marginals does not lead to accurate inference.
Instead we use a truncated empirical distribution.
We give a theoretical analysis of this estimator, 
proving consistency results with respect to risk, model
selection, and estimation of $\Omega$ in the Frobenius norm.

In the following section we review the basic notion of the graph
corresponding to a multivariate Gaussian, and formulate different
criteria for evaluating estimators of the covariance or inverse
covariance.  In Section~\ref{sec:npn} we present the nonparanormal,
and in Section~\ref{sec:estimation} we discuss estimation of the model.
We present a theoretical analysis of the
estimation method in Section~\ref{sec:theory}, with the
detailed proofs collected in an appendix.  In
Section~\ref{sec:experiments} we present experiments with both
simulated data and gene microarray data, where the problem is to  construct the isoprenoid biosynthetic pathway.

\begin{figure}[t]
\begin{center}
\renewcommand{\tabcolsep}{.3cm}
\renewcommand{\arraystretch}{1.2}
\begin{tabular}{|r|c||l|l|}
\multicolumn{1}{c}{\sf Assumptions} & \multicolumn{1}{c}{\sf Dimension} &
\multicolumn{1}{c}{\sf Regression} & \multicolumn{1}{c}{\sf Graphical Models} \\
\cline{1-4}
\multirow{2}{*}{\sf parametric\ }  & low    & linear model    & multivariate normal \\
              & high   & lasso           & graphical lasso\\
\cline{1-4}
\multirow{2}{*}{\sf nonparametric\ } & low    & additive model  & nonparanormal\\
              & high   & sparse additive model & $\ell_1$-regularized nonparanormal\\
\cline{1-4}
\end{tabular}
\end{center}
\caption{Comparison of regression and graphical models. The
  nonparanormal extends additive models to the graphical model
  setting.  Regularizing the inverse covariance leads to an extension
  to high dimensions, which parallels sparse additive models for
  regression. }
\label{fig::compare}
\end{figure}

\section{Estimating Undirected Graphs}
\label{sec:graphs}

Let $X=(X_1,\ldots, X_p)$ denote a random vector with
distribution $P=N(\mu,\Sigma)$.
The undirected graph $G=(V,E)$ corresponding to $P$
consists of a vertex set $V$ and an edge set $E$.
The set $V$ has $p$ elements, one for each component of $X$.
The edge set $E$ consists of ordered pairs $(i,j)$
where $(i,j)\in E$ if there is a edge between $X_i$ and $X_j$.
The edge between $(i,j)$ is excluded from $E$
if and only if $X_i$ is independent of $X_j$
given the other variables 
$O_{\except\{i,j\}} \equiv (X_s:\ 1\leq s \leq p, \ \ s\neq i,j)$, written
\begin{equation}\label{eq::ci}
X_i \amalg X_j \Bigm| O_{\except\{i,j\}}.
\end{equation}
It is well-known that, for multivariate Gaussian distributions,
\eqref{eq::ci} holds
if and only if
$\Omega_{ij}=0$
where $\Omega = \Sigma^{-1}$.

Let $X^\1, X^\2, \ldots, X^\pn$ be a random sample from $P$, where
$X^\ip \in\reals^p$.
If $n$ is much larger than $p$,
then we can estimate $\Sigma$ using maximum likelihood, leading to the
estimate $\hat\Omega = S^{-1}$, where 
\begin{equation*}
S =  \frac{1}{n} \sum_{i=1}^n \left(X^\ip - \Xbar\right) \left(X^\ip
  - \Xbar\right)^T
\end{equation*}
is the sample covariance, with $\Xbar$ the sample mean.
The zeroes of $\Omega$ can then be estimated by applying hypothesis testing to
$\hat\Omega$ \citep{Drton:07,Drton:08}.  

When $p > n$, maximum likelihood is no longer useful; in particular, the
estimate $\hat\Sigma$ is not positive definite, having rank no greater
than $n$.  Inspired by the success of the lasso for linear models,
several authors have suggested estimating $\Sigma$
by minimizing
\begin{equation}
-\ell(\mu,\Omega) + \lambda \sum_{j\neq k} |\Omega_{jk}|
\end{equation}
where 
\begin{equation}
\ell(\mu,\Omega) = \frac{1}{2}\left(\log |\Omega |  - \text{tr}(\Omega S) - p \log(2\pi)\right)
\end{equation}
is the average log-likelihood and $S$ is the sample covariance matrix.
The estimator $\hat\Omega$ can be computed efficiently using the
glasso algorithm \citep{FHT:07}, which is a block coordinate descent
algorithm that uses the standard lasso to estimate a single row and
column of $\Omega$ in each iteration.  Under appropriate sparsity
conditions, the resulting estimator $\hat\Omega$ has been shown to have good
theoretical properties \citep{Rothman:08,Ravikumar:Gauss:09}.

There are several different ways to judge the quality of an
estimator $\hat\Sigma$ of the covariance or inverse covariance $\hat\Omega$.
We discuss three in this paper, persistency, norm consistency, and sparsistency.
Persistency means consistency in risk, when the model is not assumed
to be correct.  Suppose the true distribution is $P$ has mean $\mu_0$,
and that we use a multivariate normal  $p(x;\mu_0,\Sigma)$ for
prediction.  We do not assume that $P$ is normal.
We observe a new vector $X\sim P$ and define the prediction risk to be
$$
R(\Sigma) = -\mathbb{E} \log p(X;\mu_0,\Sigma) = -\int \log p(x;\mu_0,\Sigma) \,dP(x).
$$
It follows that
$$
R(\Sigma) =  \frac{1}{2}\left(\text{tr}(\Sigma^{-1} \Sigma_0 ) + \log
  | \Sigma | - p\log (2\pi)\right)
$$
where $\Sigma_0$ is the covariance of $X$ under $P$.
If ${\cal S}$ is a set of covariance matrices,
the oracle is defined to be the covariance matrix $\Sigma_*$
that minimizes $R(\Sigma)$ over ${\cal S}$:
$$
\Sigma_* = {\argmin}_{\Sigma \in {\cal S}} R(\Sigma).
$$
Thus
$p(x;\mu_0,\Sigma_*)$ is the best predictor of a new observation
among all distributions
in $\{p(x;\mu_0,\Sigma):\ \Sigma\in {\cal S}\}$.
In particular, if ${\cal S}$ consists of covariance matrices with sparse graphs,
then $p(x;\mu_0,\Sigma_*)$ is, in some sense, the best sparse predictor.
An estimator $\hat\Sigma_n$ is \textit{persistent} if
$$
R(\hat\Sigma_n) - R(\Sigma_*) \stackrel{P}{\rightarrow} 0
$$
as the sample size $n$ increases to infinity.
Thus, a persistent estimator approximates the best estimator
over the class ${\cal S}$, but we do not assume that the true
distribution has a covariance matrix in ${\cal S}$, or even that it is
Gaussian. Moreover, we allow the dimension $p=p_n$ to increase with $n$.
On the other hand, norm consistency and sparsistency require that
the true distribution is Gaussian.
In this case, let $\Sigma_0$ denote the true covariance matrix.
An estimator is \textit{norm consistent} if
$$
\|\hat\Sigma_n - \Sigma\| \stackrel{P}{\to} 0
$$
where $\|\cdot\|$ is a norm.
If $E(\Omega)$ denotes the edge set corresponding to $\Omega$.
An estimator is \textit{sparsistent} if
$$
\mathbb{P}\Bigl( E(\Omega) \neq E(\hat\Omega_n) \Bigr) \rightarrow 0.
$$
Thus, a sparsistent estimator identifies the correct graph consistently.
We summarize known results on these properties for the multivariate
normal in Section~\ref{sec:theory}, before presenting our theoretical
analysis of the nonparanormal.

\section{The Nonparanormal}
\label{sec:npn}

We say that a random vector
$X=(X_1,\ldots,X_p)^T$ has a {\em nonparanormal} distribution if
there exist functions
$\{f_j\}_{j=1}^p$ such that
$Z \equiv f(X) \sim N(\mu,\Sigma)$, where $f(X) =
(f_1(X_1),\ldots, f_p(X_p))$.
We then write
$$
X \sim \npn(\mu,\Sigma,f).
$$
When the $f_j$'s are monotone and  differentiable, the joint
probability density function of $X$ is given by
\begin{equation}
p_X(x) =
\frac{1}{(2\pi)^{p/2}|\Sigma|^{1/2}}
\exp\left\{-\frac{1}{2}\left(f(x)-\mu\right)^{T}{\Sigma}^{-1}\left(f(x)-\mu\right)
\right\}\prod_{j=1}^p|f'_j(x_j)|.
\label{eq:npndensity}
\end{equation}

\begin{lemma}
The nonparanormal distribution $\npn(\mu,\Sigma,f)$ is a Gaussian copula when the $f_{j}$'s are monotone and  differentiable.
\end{lemma}

\begin{proof}
By Sklar's theorem \citep{Sklar:59},
any joint distribution can be written as
$$
F(x_1,\ldots,x_p) = C\{ F_1(x_1),\ldots, F_p(x_p)\}
$$
where the function $C$ is called a copula.
For the nonparanormal we have
$$
F(x_1,\ldots,x_p) = \Phi_{\mu,\Sigma}( \Phi^{-1}(F_1(x_1)),\ldots, \Phi^{-1}(F_p(x_p)))
$$
where
$\Phi_{\mu,\Sigma}$ is the multivariate Gaussian cdf and
$\Phi$ is the univariate standard Gaussian cdf.
Thus, the corresponding copula is
$$
C(u_1,\ldots, u_p) = \Phi_{\mu,\Sigma}( \Phi^{-1}(u_1),\ldots, \Phi^{-1}(u_p)).
$$
This is exactly a Gaussian copula with parameters $\mu$ and $\Sigma$.  If each $f_j$ is differentiable then
the density of $X$ has the same form as \eqref{eq:npndensity}.
\end{proof}

Note that the density in \eqref{eq:npndensity} is not identifiable; to
make the family identifiable we demand that $f_j$ preserve means and
variances:
\begin{equation}
\mu_{j} = \mathbb{E}(Z_{j}) = \mathbb{E}(X_j)
~~\mathrm{and}~~\sigma^2_{j} \equiv \Sigma_{jj} = 
\mathrm{Var}\left( Z_j\right) = \mathrm{Var}\left( X_j\right). 
\label{eq:identify}
\end{equation}
Note that these conditions only depend on $\mathrm{diag}(\Sigma)$ but
not the full covariance matrix.

Let $F_{j}(x)$ denote the marginal distribution function of $X_{j}$. Then
$$
F_{j}(x) = \mathbb{P}\left(X_{j} \leq x \right) = \mathbb{P}\left(Z_j \leq
f_j(x)\right) = \Phi\left( \frac{f_j(x)-\mu_{j}}{\sigma_{j}} \right)
$$
which implies that
\begin{equation}
f_j(x) = \mu_{j} + \sigma_{j} \Phi^{-1}\left(F_j(x) \right). \nonumber
\end{equation}

The following basic fact says that the independence graph of the
nonparanormal is encoded in $\Omega = \Sigma^{-1}$, as for the
parametric normal.

\begin{lemma}
If $X\sim \npn(\mu,\Sigma,f)$ is nonparanormal and each $f_j$ is differentiable, then 
$X_i \amalg X_j  \given O_{\except\{i,j\}}$ if and only if
$\Omega_{ij}=0$, where $\Omega = \Sigma^{-1}$.
\end{lemma}

\begin{proof}
From the form of the density \eqref{eq:npndensity}, it follows that
the density factors with respect to the graph of $\Omega$, and therefore
obeys the global Markov property of the graph.
\end{proof}

\begin{figure}[t]
\begin{center}
\begin{tabular}{ccc}
\includegraphics[width=.29\textwidth,angle=-90]{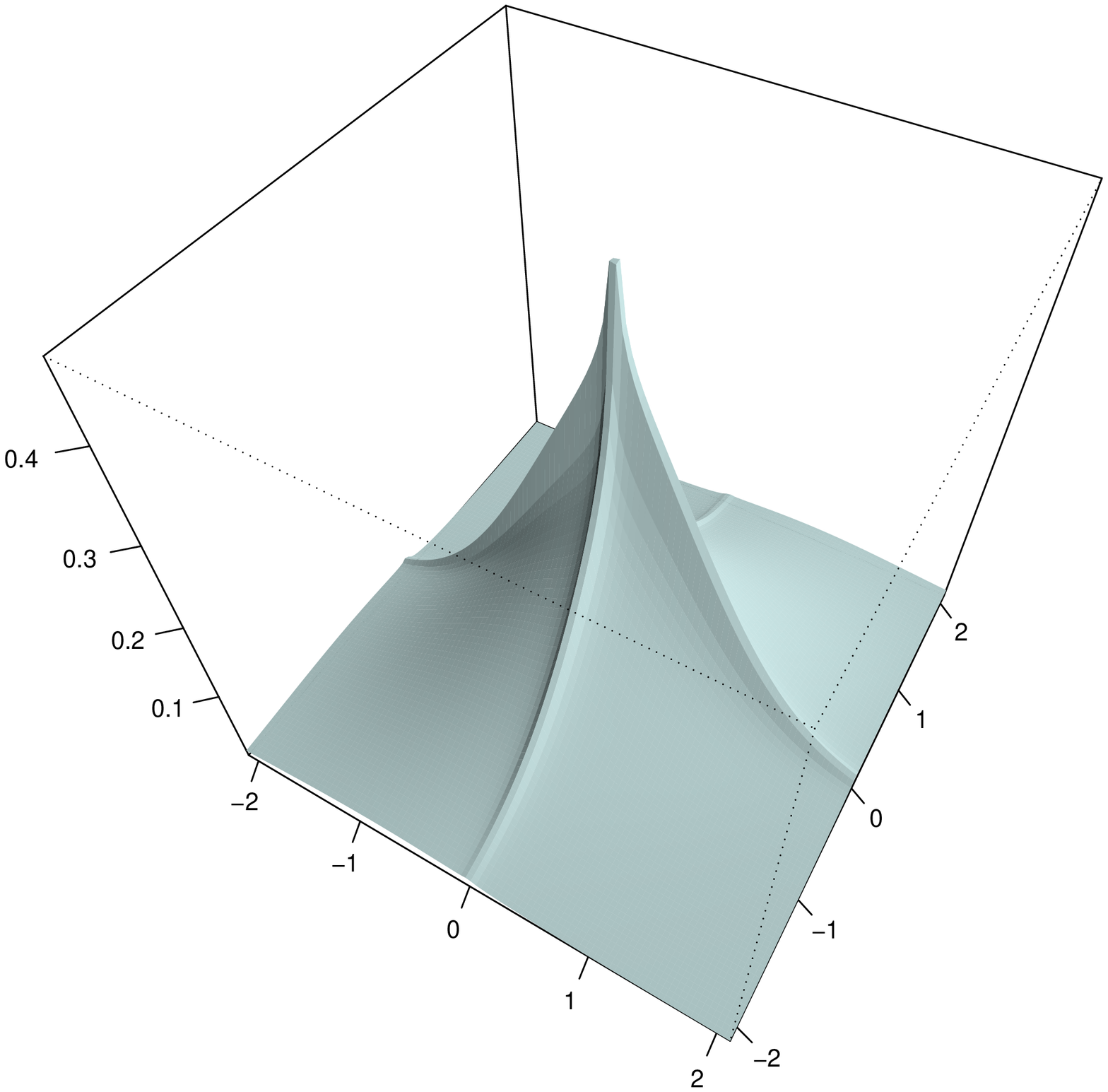} &
\includegraphics[width=.29\textwidth,angle=-90]{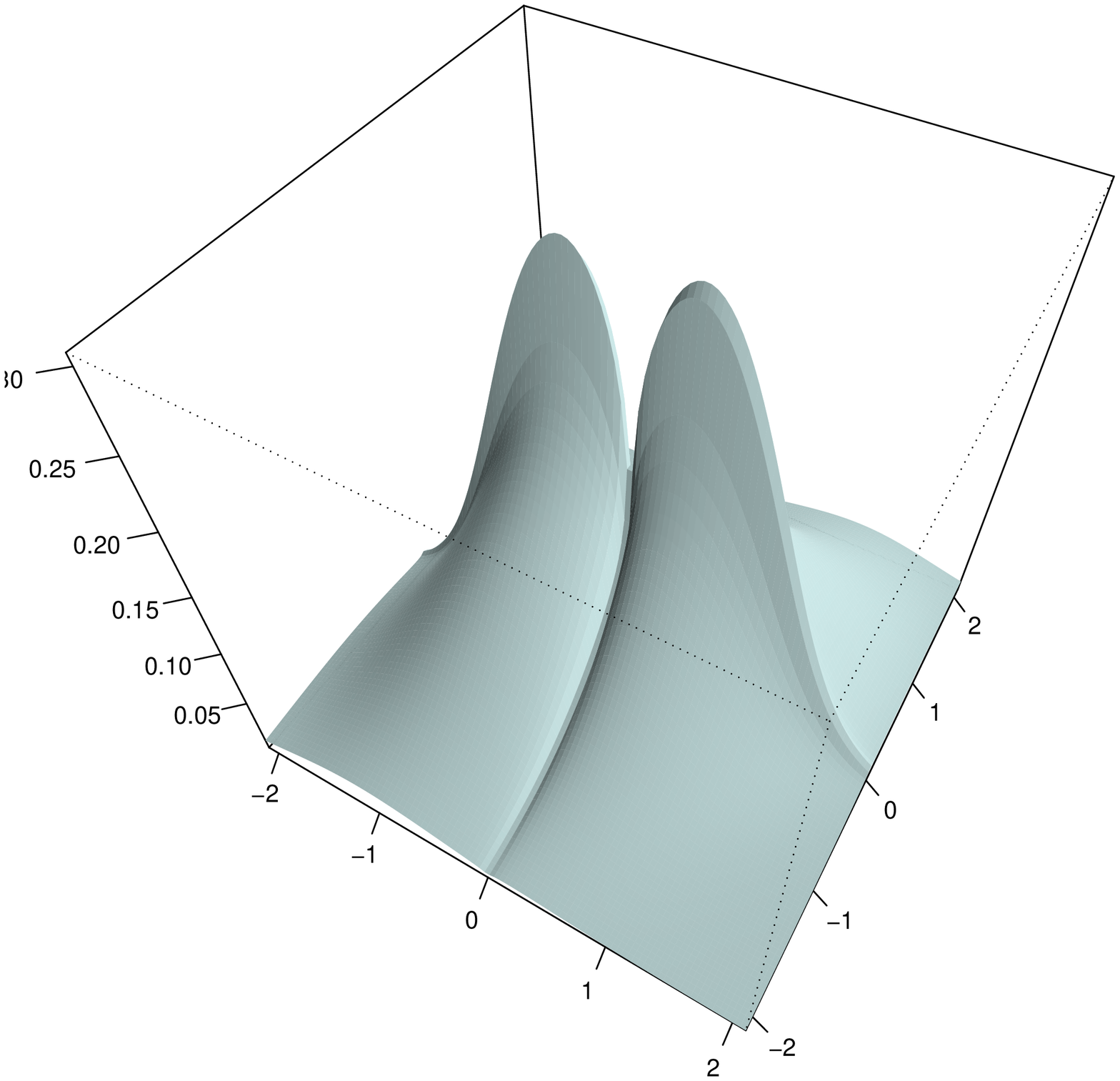} &
\includegraphics[width=.29\textwidth, angle=-90]{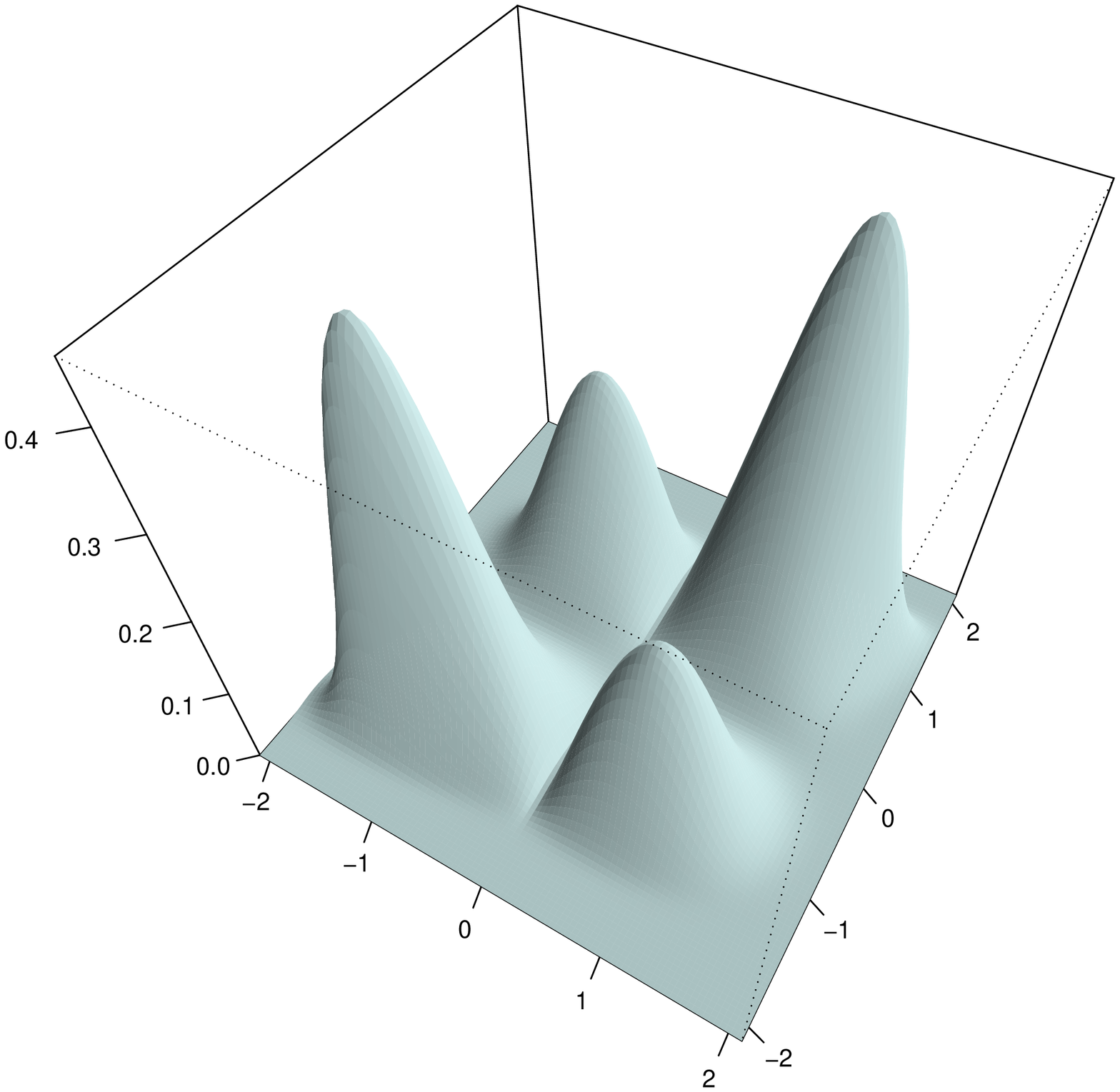} \\
\includegraphics[width=.29\textwidth,angle=-90]{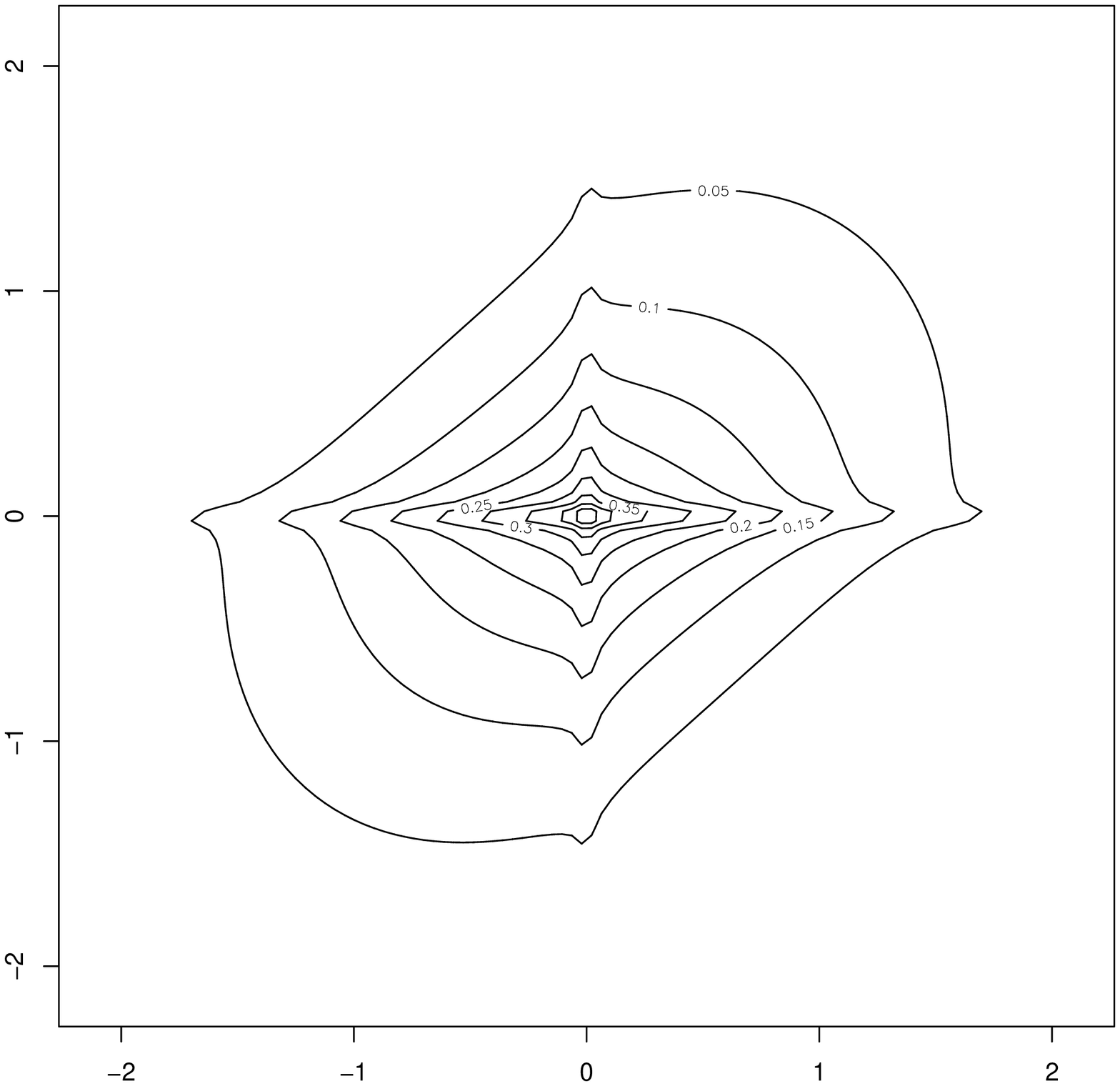} &
\includegraphics[width=.29\textwidth,angle=-90]{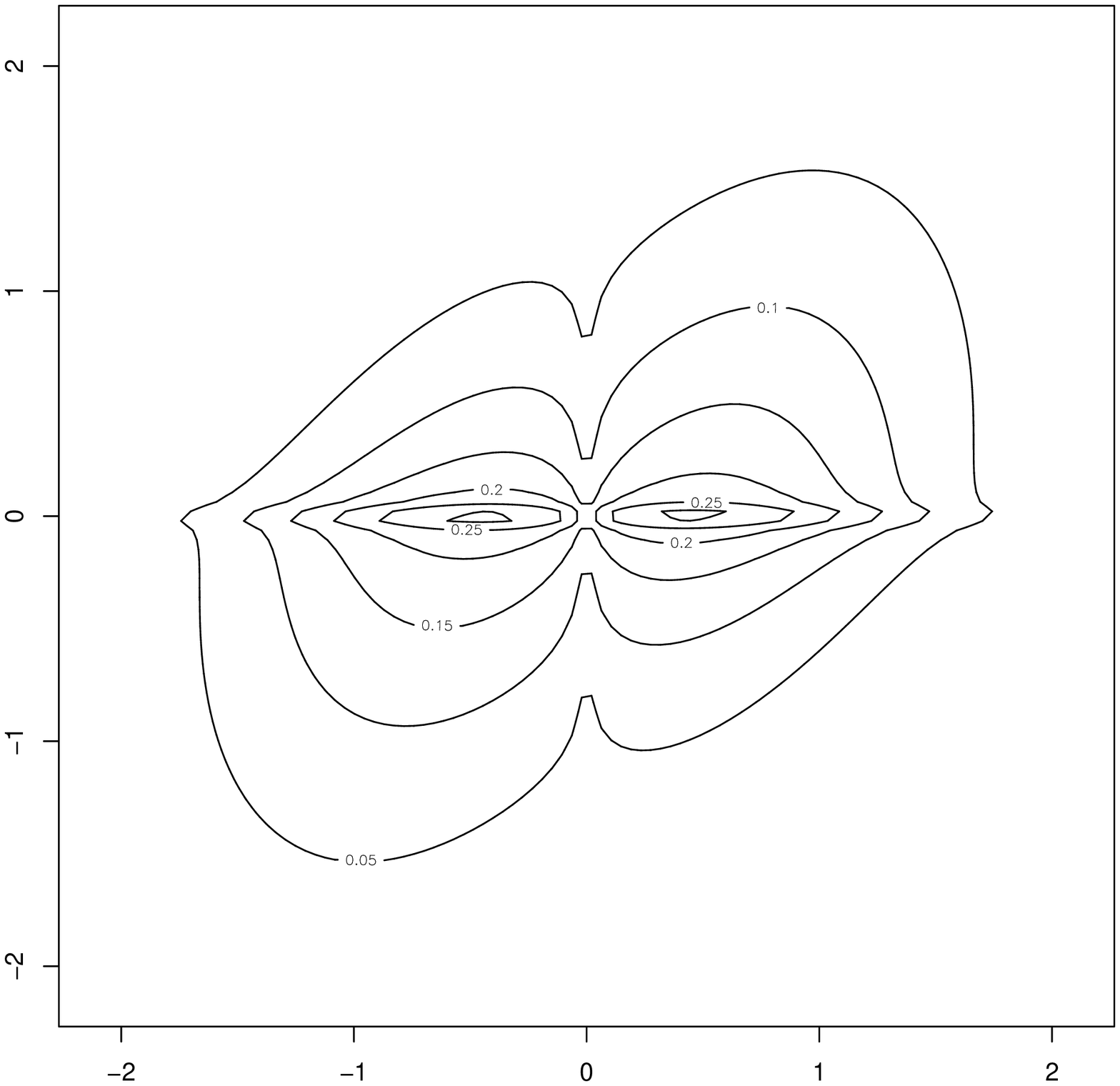} &
\includegraphics[width=.29\textwidth, angle=-90]{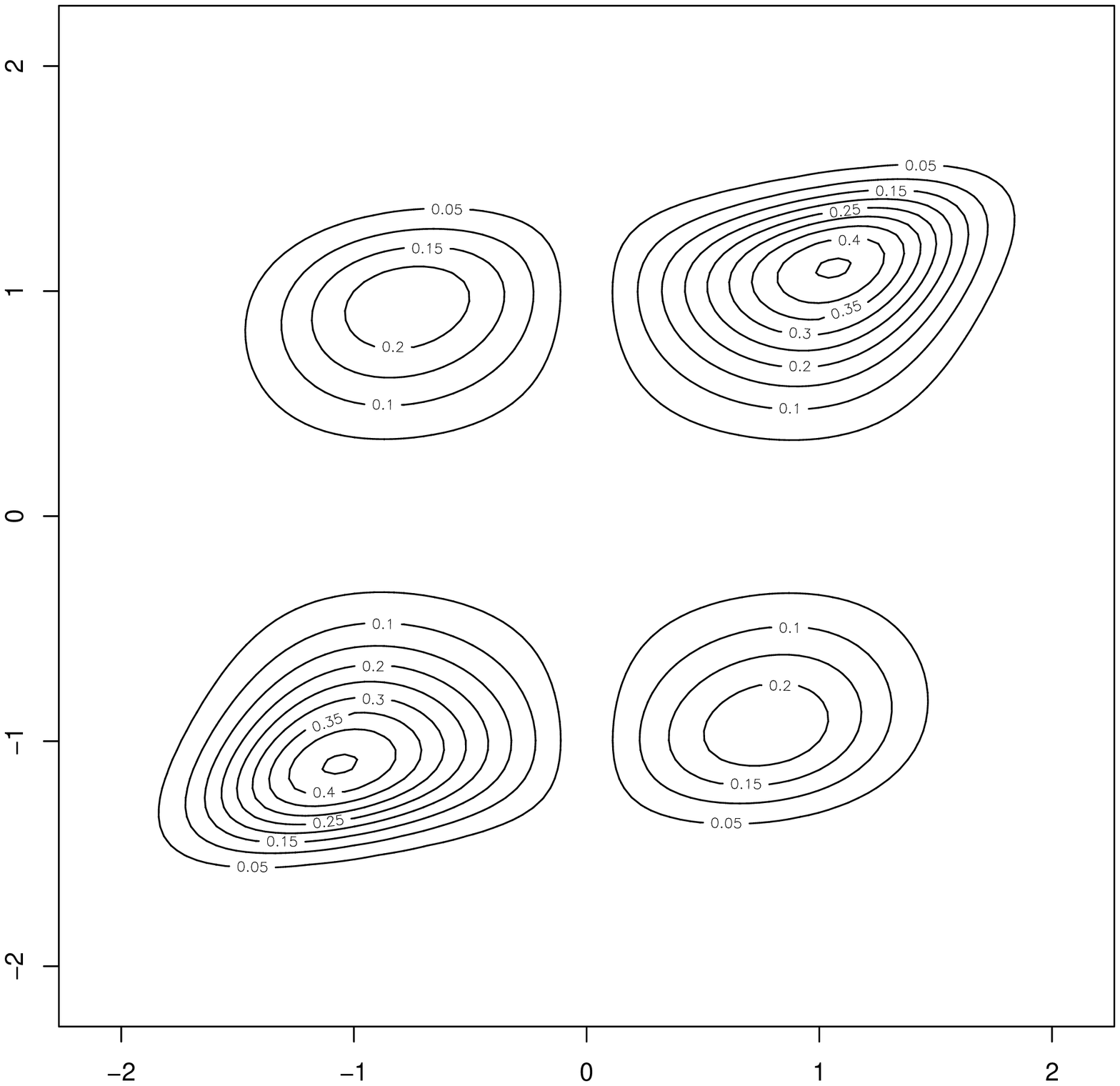} \\
\end{tabular}
\end{center}
\caption{Densities of three 2-dimensional nonparanormals.  The
  component functions have the form $f_j(x) =
  \text{sign}(x)|x|^{\alpha_j}$.   Left: $\alpha_1=0.9$,
  $\alpha_2=0.8$;   center: $\alpha_1=1.2$, $\alpha_2=0.8$; right
  $\alpha_1=2$, $\alpha_2=3$.  In each case $\mu=(0,0)$ and $\Sigma =
  \binom{1\ .5}{.5\ 1}$.}
\label{fig:densityex}
\vskip10pt
\end{figure}

Next we show that the above is true for any choice of
identification restrictions.

\begin{lemma}
\label{lemma:h}
Define 
\begin{equation}\label{eq::h}
h_j(x) = \Phi^{-1}(F_j(x))
\end{equation}
and let $\Lambda$ be the
covariance matrix of $h(X)$.  Then
$X_j \amalg X_k \given O_{\except\{j,k\}}$ if and only if
$\Lambda_{jk}^{-1} =0$.
\end{lemma}

\begin{proof}
We can rewrite the covariance matrix as
$$
\Sigma_{jk} = {\rm Cov}(Z_j,Z_k) = \sigma_{j}\sigma_{k} {\rm Cov}(h_j(X_j),h_k(X_k)).
$$
Hence $\Sigma = D \Lambda D$ and 
$$
\Sigma^{-1} = D^{-1} \Lambda^{-1} D^{-1}.
$$
where $D$ is the diagonal matrix with $\text{diag}(D) = \sigma$.  The
zero pattern of $\Lambda^{-1}$ is therefore identical to the zero
pattern of $\Sigma^{-1}$.
\end{proof}

Thus, it is not necessary to estimate $\mu$ or $\sigma$ to 
estimate the graph.

Figure~\ref{fig:densityex} shows three examples of 2-dimensional
nonparanormal densities.  In each case, the component functions
$f_j(x)$ take the form
\begin{equation*}
f_j(x) = a_j \text{sign}(x) |x|^{\alpha_j} + b_j
\end{equation*}
where 
the constants $a_j$ and $b_j$ are set to enforce the
identifiability constraints \eqref{eq:identify}.  The covariance in
each case is $\Sigma = \binom{1\ .5}{.5\ 1}$ and the mean is $\mu=(0,0)$.
The exponent $\alpha_j$ determines the nonlinearity.  It can be seen
how the concavity of the density changes with the exponent $\alpha$,
and that $\alpha > 1$ can result in multiple modes.

The assumption that $f(X) = (f_1(X_1), \ldots, f_p(X_p)$ is normal leads
to a semiparametric model where only one dimensional functions need
to be estimated.  But the monotonicity of the functions $f_j$, which map onto $\reals$,
enables computational tractability of the nonparanormal.  For more
general functions $f$, the normalizing constant for the density 
\begin{equation}
p_X(x) \propto
\exp\left\{-\frac{1}{2}\left(f(x)-\mu\right)^{T}{\Sigma}^{-1}\left(f(x)-\mu\right)
\right\} \nonumber
\label{eq:gpndensity}
\end{equation}
cannot be
computed in closed form.

%

\section{Estimation Method}
\label{sec:estimation}

Let $X^{(1)},\ldots, X^{(n)}$ be a sample of size $n$ where
$X^\i=(X^\i_{1},\ldots, X^\i_{p})^T \in \mathbb{R}^p$.
In light of \eqref{eq::h} we define
\begin{equation}
\hat{h}_j(x) = \Phi^{-1}(\tilde{F}_j(x)) \nonumber
\end{equation}
where
$\tilde{F}_j$ is an estimator of $F_j$. A natural candidate for $\tilde{F}_j$
is the marginal empirical distribution function
$$
\hat{F}_j(t) \equiv \frac{1}{n}\sum_{i=1}^n\mathbf{1}_{\left\{X^{(i)}_{j} \leq t\right\}}.
$$
Now, let $\theta$ denote the parameters of the copula.
Tsukahara (2005) suggests taking
$\hat\theta$ to be the solution of
$$
\sum_{i=1}^n \phi\left( \tilde{F}_1(X_1^\i),\ldots, \tilde{F}_p(X_p^\i),\theta\right) = 0
$$
where $\phi$ is an estimating equation
and
$\tilde{F}_j(t) = n \hat{F}_j(t)/(n+1)$.
In our case, $\theta$ corresponds to the covariance matrix.
The resulting estimator $\hat\theta$, called a rank approximate $Z$-estimator,
has excellent theoretical properties. However, 
we are interested in the 
high dimensional scenario where the dimension $p$ is allowed to increase
with $n$; the variance of
$\hat{F}_j(t)$ is too large in this case.
Instead, we use the following truncated 
or {\em Winsorized\/}\footnote{After Charles
  P. Winsor, whom John Tukey credited with converting him from
  topology to statistics \citep{Tukey:VI:90}.} estimator:
\begin{equation}\label{eq:truncatedestimator}
\tilde{F}_j(x) =  
\begin{cases}
\delta_{n} & \text{if $\hat{F}_j(x) < \delta_n$} \\
\hat{F}_j(x) & \text{if $\delta_n \leq \hat{F}_j(x) \leq 1-\delta_n$} \\
(1-\delta_n) & \text{if $\hat{F}_j(x) > 1- \delta_n$},
\end{cases}
\end{equation}
where $\delta_n$ is a truncation parameter. 
Clearly, there is a bias-variance tradeoff in choosing $\delta_n$.
In what follows we use
\begin{equation}
\delta_{n} \equiv \frac{1}{4n^{1/4}\sqrt{\pi\log n}}.
\label{eq:delta}
\end{equation}
This provides the right balance so that we can achieve the desired
rate of convergence in our estimate of $\Omega$ and the associated
undirected graph $G$.

Given this estimate of the distribution of variable $X_{j}$, we then
estimate the transformation function $f_j$ by
\begin{eqnarray}\label{eq:keyestimator}
\tilde{f}_j(x) \equiv \hat{\mu}_{j}+\hat{\sigma}_{j}\tilde{h}_j(x)
\end{eqnarray}
where
\begin{equation}
\tilde{h}_j(x) = \Phi^{-1}\left(\tilde{F}_j(x)\right) \nonumber
\end{equation}
and
$\hat{\mu}_{j}$ and $\hat{\sigma}_{j}$ are the sample mean and the standard deviation:
$$
\hat{\mu}_{j} \equiv 
\frac{1}{n}\sum_{i=1}^{n}X^\i_{j}~~\mathrm{and}~~
\hat{\sigma}_{j}= \sqrt{\frac{1}{n}\sum_{i=1}^{n}\left(X^\i_{j} - \hat{\mu}_{j} \right)^{2}}.
$$
Now, let $S_n(\tilde{f})$ be the sample covariance matrix of
$\tilde{f}(X^{(1)}),\ldots, \tilde{f}(X^{(n)})$; that is,
\begin{eqnarray}
\label{eq:covardef}
S_n(\tilde{f}) &\equiv& \frac{1}{n} 
\sum_{i=1}^n \left( \tilde{f}(X^\i) - \mu_n(\tilde{f})\right)\left( \tilde{f}(X^\i) -  \mu_n(\tilde{f})\right)^T \\
\mu_n(\tilde{f}) &\equiv& \frac{1}{n} \sum_{i=1}^n \tilde{f}(X^\i). \nonumber
\end{eqnarray}
We then estimate $\Omega$ using $S_n(\tilde f)$.  For instance, the
maximum likelihood estimator is $\hat\Omega_n = S_n(\tilde
  f)^{-1}$.  The $\ell_1$-regularized estimator is 
\begin{equation}
\hat\Omega_n = \mathop{\text{arg\,min}}_{\Omega} 
\left\{\text{tr}\left(\Omega S_n(\tilde f)\right)  - \log |\Omega |  + \lambda \|\Omega\|_1\right\}
\label{eq:winsorized-est}
\end{equation}
where $\lambda$ is a regularization parameter, 
and $\|\Omega\|_1=\sum_{j\neq k}|\Omega_{jk}|$.  The estimated graph
is then  $\hat{E}_n = \{ (j,k):\ \hat\Omega_{jk}\neq 0\}$.  In the
following section we analyze the theoretical properties of this $\ell_1$-regularized estimator.

\section{Theoretical Results}
\label{sec:theory}

In this section we present our theoretical results on risk consistency, model selection
consistency, and norm consistency of the covariance $\Sigma$ and inverse covariance
$\Omega$.  From Lemma~\ref{lemma:h}, the estimate of the graph does
not depend on $\sigma_j, \; j\in\{1,\ldots, p\}$ and $\mu$, so we assume that $\sigma_j=1$
and $\mu = 0$.   Our key technical result is an analysis of covariance of the Winsorized estimator
defined in \eqref{eq:truncatedestimator}, \eqref{eq:keyestimator}, and
\eqref{eq:covardef}.  In particular, we show that under appropriate conditions, 
\begin{equation*}
\max_{j,k} \left| S_n(\tilde{f})_{jk} - S_n(f)_{jk}\right| = O_{P}\left( \sqrt{\frac{\log p \log^2 n}{n^{1/2}}}\right)
\end{equation*}
where $S_n(\tilde f)_{jk}$ denotes the $(j,k)$ entry of the matrix.  This result allows us
to leverage the recent analysis of \cite{Rothman:08} and \cite{Ravikumar:Gauss:09} in the
Gaussian case to obtain consistency results for the nonparanormal.  More precisely, our
main theorem is the following.

\begin{theorem}\label{thm.keylemma}
Suppose that $p = n^\xi$ and let $\tilde{f}$ be the Winsorized estimator defined in 
\eqref{eq:keyestimator} with $\delta_{n} = \ds \frac{1}{4n^{1/4}\sqrt{\pi\log n}}$.  Define
\begin{eqnarray*}
C(M,\xi) \equiv  \frac{48}{\sqrt{\pi\xi}}\left(\sqrt{2M}-1 \right)(M+2)
\end{eqnarray*}
for $M,\xi > 0$. Then for any  $\ds \epsilon \geq C(M,\xi)\sqrt{\frac{\log p \log^2
    n}{n^{1/2}}}$ and  
sufficiently large $n$, we have
\begin{eqnarray}
\mathbb{P}\left(\max_{jk} \left| S_n(\tilde{f})_{jk} - S_n(f)_{jk}\right|> \epsilon
\right) 
\leq \frac{c_{1}p}{(n\epsilon^{2})^{2\xi}} + \frac{c_{2}p}{n^{M\xi-1}} + 
c_{3}\exp\left(-\frac{c_{4}n^{1/2}\epsilon^2}{\log p  \log^2 n} \right), \nonumber
\end{eqnarray}
where $c_{1}, c_{2}, c_{3},  c_{4}$ are positive constants.
\end{theorem}

The proof of the above theorem is given in Section~\ref{sec:proofs}.
The following corollary is immediate, which specifies the scaling of the dimension in terms of
sample size.  

\begin{corollary}
Let $M>\ds1+\frac{1}{\xi}$.  Then
\begin{eqnarray}
\mathbb{P}\left(\max_{jk} 
\left| S_n(\tilde{f})_{jk} - S_n(f)_{jk}\right|>C(M,\xi)\sqrt{\frac{\log p \log^2 n}{n^{1/2}}}\right) = 
o(1) + c_{3}\exp\left(-c_{4}C^{2}(M,\xi) \right).\nonumber
\end{eqnarray}
Hence,
\begin{eqnarray}
\max_{j,k} \left| S_n(\tilde{f})_{jk} - S_n(f)_{jk}\right| = O_{P}\left( \sqrt{\frac{\log p \log^2 n}{n^{1/2}}}\right).  
\label{eq.key}
\end{eqnarray}
\end{corollary}

The following corollary yields estimation consistency in both the
Frobenius norm and the $\ell_{2}$-operator norm.  The proof follows
the same arguments as the proof of Theorem 1 and Theorem 2 from
\cite{Rothman:08}, replacing their Lemma 1 with our
Theorem~\ref{thm.keylemma}. 

For a matrix $A = (a_{ij})$, the Frobenius norm $\| \cdot \|_{F}$ is
defined as $ \| A\|_{\rm F} \equiv \sqrt{\sum_{i,j}a^{2}_{ij}}$.  The 
$\ell_{2}$-operator norm $\|\cdot \|_{2}$ is defined as the magnitude of
the largest eigenvalue of the matrix, $ \|A \|_{2} \equiv \max_{\|x
  \|_{2} =1} \|Ax \|_2$.  In the following, we write $a_{n}\asymp b_{n}$ if
there are positive constants $c$ and $C$ independent of $n$ such that
$ c\leq {a_{n}}/{b_{n}} \leq C$.

\begin{corollary} \label{consistency}
Suppose that the data are generated as $X^{(i)}\sim \npn(\mu_0,\Sigma_0,
f_0)$, and let $\Omega_0 = \Sigma_0^{-1}$.  If the
regularization parameter $\lambda_n$ is chosen as
\begin{eqnarray*}
\lambda_n \asymp \sqrt{\frac{\log p \log^2 n}{n^{1/2}}}
\end{eqnarray*}
then the nonparanormal estimator $\hat{\Omega}_n$ of
\eqref{eq:winsorized-est} satisfies
\begin{eqnarray}\label{eq.Frobenius}
\|\hat{\Omega}_n -\Omega_0\|_{\rm F} = O_{P}\left( \sqrt{\frac{(s + p)\log p \log^2 n}{n^{1/2}} }\right) \nonumber
\end{eqnarray}
and
\begin{eqnarray}\label{eq.L2}
\|\hat{\Omega}_n -\Omega_{0}\|_{2} = O_{P}\left( \sqrt{\frac{s\log p \log^2 n}{n^{1/2}} }\right), \nonumber
\end{eqnarray}
where 
\begin{eqnarray}
s \equiv \mathrm{Card}\left(\left\{(i,j) \in \{1,\ldots, p \} \times \{1,\ldots, p \}  \given  {{\Omega_0}(i,j)}\neq 0,~ i\neq j\right\}\right) \nonumber
\end{eqnarray}
is the number of nonzero off-diagonal elements of the true precision matrix.
\end{corollary}

To prove the model selection consistency result, we need further assumptions.  We follow
Ravikumar (2009) and let the $p^{2} \times p^{2}$ Fisher information
matrix of $\Sigma_{0}$ be $\Gamma \equiv \Sigma_{0} \otimes \Sigma_{0}
$ where $\otimes$ is the Kronecker matrix product, and define
the support set $S$ of $\Omega_0= \Sigma_0^{-1}$ as
\begin{eqnarray}
S \equiv \left\{(i,j) \in \{1,\ldots, p \} \times \{1,\ldots, p \}  \given  \Omega_{0}(i,j)\neq 0\right\}. \nonumber
\end{eqnarray}

We use $S^{c}$ to denote the complement of $S$ in the set $\{1,\ldots, p\}
\times \{1,\ldots, p \}$, and for any two subsets $T$ and $T'$ of
$\{1,\ldots, p \} \times \{1,\ldots, p\}$, we use $\Gamma_{TT'}$ to
denote the sub-matrix with rows and columns of $\Gamma$ indexed by $T$
and $T'$ respectively. 

\begin{assumption}\label{assump.irrep}
There exists some $\alpha \in (0,1]$, such that 
$\left\| \Gamma_{S^{c}S} (\Gamma_{SS})^{-1} \right\|_{\infty} \leq 1 - \alpha. 
$
\end{assumption}

As in \cite{Ravikumar:Gauss:09}, we define two quantities
$K_{\Sigma_{0}} \equiv \|\Sigma_{0} \|_{\infty}$ and $K_{\Gamma}
\equiv \| (\Gamma_{SS} )^{-1}\|_{\infty}$.  Further, we define the
maximum row degree as
\begin{eqnarray}
d \equiv \max_{i=1,\ldots, p}\mathrm{Card}\left( \{ j \in {1,\ldots, p} \given \Omega_{0}(i,j) \neq 0\}\right). \nonumber
\end{eqnarray}

\begin{assumption}\label{assump.boundedquantities}
  The quantities $K_{\Sigma^{0}}$ and $K_{\Gamma} $ are bounded, and
  there are positive constants $C_{1}$ and $C_{2}$ such that
\begin{eqnarray}
\min_{(j,k) \in S}\left| \Omega_{0}(j,k)\right| \geq C_{1} \sqrt{\frac{\log n}{n}}~~\mathrm{and}~~n \geq C_{2}d ^{2}\log p \nonumber
\end{eqnarray}
for large enough $n$.
\end{assumption}

The proof of the following uses our Theorem~\ref{thm.keylemma} in place of 
equation (12) in the analysis of \cite{Ravikumar:Gauss:09},

\begin{corollary}  
Suppose the regularization parameter is chosen as
\begin{eqnarray}
\lambda_n \asymp \sqrt{\frac{\log p \log^2 n}{n^{1/2}}}. \nonumber
\end{eqnarray}
Then the nonparanormal estimator $\hat{\Omega}_n$ satisfies
\begin{eqnarray}
\mathbb{P}\left( \cG\left(\hat{\Omega}_n,  \Omega_{0}\right) \right) \geq 1-o(1) \nonumber
\end{eqnarray}
where $\cG( \hat{\Omega}_n,  \Omega_{0} )$ is the event
\begin{eqnarray}
\left\{ \mathrm{sign} \left(\hat{\Omega}_n(j,k)\right) = \mathrm{sign} \left(\Omega_{0}^{-1}(j,k)\right) ,~~~\forall j,k \in S \right\}. \nonumber
\end{eqnarray}
\end{corollary}

Our persistency (risk consistency) result parallels the persistency result for 
additive models given in \cite{Ravikumar:08}, and allows model dimension
that grows exponentially with sample size.  The definition in this theorem
uses the fact (from Lemma~\ref{lemma:quantile}) that
$\sup_x \Phi^{-1}\left(\tilde{F}_j(x)\right) \leq \sqrt{2\log n}$
when $\delta_{n} = 1/ (4n^{1/4}\sqrt{\pi\log n})$.

In the next theorem, we do not assume the true model is  nonparanormal and define the population and sample risks as
\begin{eqnarray*}
R(f, \Omega) &=&  \frac{1}{2}\left\{\tr\left[ \Omega \E(f(X) f(X)^T\right] - \log
  |  \Omega| - p\log (2\pi)\right\} \\
\hat R(f,  \Omega) &=&  \frac{1}{2}\left\{\tr\left[ \Omega S_n(f)\right] - \log
  | \Omega | - p\log (2\pi)\right\}.
\end{eqnarray*}

\begin{theorem}
\label{thm:persist}
Suppose that $p \leq e^{n^\xi}$ for some
$\xi < 1$, and define the classes
\begin{eqnarray}
\cM_n &=& 
\left\{ f \,:\,  \text{$f$ is monotone with $\|f\|_\infty \leq  C \sqrt{\log n}$}\right\} \nonumber\\
\cE_n &=& 
\left\{ \Omega  \,:\, \|\Omega^{-1}\|_1 \leq L_n\right\}. \nonumber
\end{eqnarray}
Let ${\hat\Omega_n}$ be given by
\begin{equation}
\hat\Omega_n = \mathop{\text{arg\,min}}_{\Omega \in \cE_n}  
\left\{\text{tr}\left(\Omega S_n(\tilde f)\right)  - \log |\Omega |\right\}. \nonumber
\end{equation}
Then
\begin{equation*}
R(\tilde f_n, \hat{\Omega}_n) - \inf_{(f,\Omega)\in\cM_n^p\oplus\cE_n} R(f,\Omega) =
O_P\left(L_n\sqrt{\frac{\log n}{n^{1-\xi}}}\right).
\end{equation*}
Hence the Winsorized estimator of $(f,\Omega)$ with 
$\delta_{n} = 1/ (4n^{1/4}\sqrt{\pi\log n})$ 
is persistent over
$\cE_n$ when $L_n = o\left(n^{(1-\xi)/2} / \sqrt{\log n}\right)$.
\end{theorem}

The proofs of Theorems~\ref{thm.keylemma} and~\ref{thm:persist} are
given in Section~\ref{sec:proofs}.

\section{Experimental Results}
\label{sec:experiments}

In this section, we report experimental results on synthetic and real
datasets.  We mainly compare the $\ell_1$-regularized nonparanormal
and Gaussian (paranormal) models, computed using the graphical lasso
algorithm (glasso) of \cite{FHT:07}. The primary conclusions are: (i) When the
data are multivariate Gaussian, the performance of the two methods is
comparable; (ii) when the model is correct, the nonparanormal performs
much better than the graphical lasso in many cases; (iii) even for
distributions that are not nonparanormal, the new method often
performs better; (iv) for gene microarray data, our method behaves
differently from the graphical lasso, and may support different
biological conclusions.

Note that we can reuse the glasso implementation to fit a sparse
nonparanormal.  In particular, after computing the Winsorized sample
covariance $S_n(\tilde f)$, we pass this matrix to the glasso
routine to carry out the optimization
\begin{equation}
\hat\Omega_n = \mathop{\text{arg\,min}}_{\Omega} 
\left\{\text{tr}\left(\Omega S_n(\tilde f)\right)  - \log |\Omega |  + \lambda_n \|\Omega\|_1\right\}. \nonumber
\end{equation}

\subsection{Neighborhood graphs }

We begin by describing a procedure to generate graphs as in \citep{Meinshausen:2006}, with respect to
which several distributions can then be defined. We generate a  $p$-dimensional sparse graph $G\equiv  (V,E)$ as follows:
Let $V=\{1,\ldots, p\}$
corresponding to variables $X=(X_1,\ldots, X_p)$.
We associate each index $j$ with a point
$(Y_j^{(1)},Y_j^{(2)})\in [0,1]^2$
where
$$ Y^{(k)}_1,\ldots, Y^{(k)}_n \sim \mathrm{Uniform}[0,1]$$ for $k=1,2$. 
Each pair of nodes $(i,j)$ is included in the edge set $E$ with
probability
\begin{eqnarray}
\mathbb{P}\biggl( (i,j)\in E \biggr) =
\frac{1}{\sqrt{2\pi}}\exp\left( -\frac{\|y_i - y_j \|^2_n}{2s}
\right) \nonumber
\end{eqnarray}
where $y_i \equiv (y^{(1)}_i,y^{(2)}_i)$ is the observation of $(Y^{(1)}_{i}, Y^{(2)}_{i})$ and $\|\cdot \|_n$
represents the Euclidean distance. Here, $s=0.125$ is a parameter that
controls the sparsity level of the generated graph.  We restrict the maximum
degree of the graph to be four and build the inverse covariance matrix
$\Omega_{0}$ according to
\begin{eqnarray}
\Omega_{0}(i,j) = \left\{\begin{array}{lr}
                            1 &  \mathrm{if}~i=j \\
                            0.245 & \mathrm{if}~(i,j)\in E \\
                            0 & \mathrm{otherwise,}
                          \end{array}
  \right.\label{exp.InvSigma}
\end{eqnarray}
where the value $0.245$ guarantees positive definiteness of the
inverse covariance matrix.

Given $\Omega_0$, $n$ data points are sampled from 
\begin{eqnarray}
X^{(1)}, \ldots, X^{(n)} \sim NPN(\mu_{0}, \Sigma_{0}, f_{0}) \label{eq.expsetup}
\end{eqnarray}
where $\mu_{0} = (1.5,\ldots, 1.5)$, $\Sigma_{0} =  \Omega^{-1}_{0}$. For simplicity,  the transformations functions for all dimensions are the same $f_{1} = \ldots = f_{p} = f_{0}$.  To sample data from the nonparanormal distribution, we also need $g_{0}  \equiv f^{-1}_{0}$,  two different transformations $g_{0}$ are employed:

 Next, we define the
following two transformation families:

\begin{definition}{\rm (Gaussian CDF Transformation)}
Let $g_{0}$ be a one-dimensional Gaussian cumulative distribution
function with mean $\mu_{g_{0}}$ and the standard deviation
$\sigma_{g_{0}}$, i.e.,
\begin{eqnarray}
g_{0}(t) \equiv  \Phi\left(\frac{t-\mu_{g_{0}}}{\sigma_{g_{0}}} \right). \nonumber
\end{eqnarray}
We define the
transformation function $g_j = f_j^{-1}$ for the $j$-th dimension as
\begin{eqnarray}
g_{j}(z_{j}) \equiv \sigma_{j} \left( \frac{g_0(z_{j}) - {\ds \int} g_0(t)\phi\left(\frac{t-\mu_{j} }{ \sigma_{j}  }\right)dt}{ \sqrt{{\ds\int} \left(g_0(y)- {\ds \int} g_0(t)\phi\left(\frac{t-\mu_{j} }{ \sigma_{j}  }\right)dt \right)^2 \phi\left(\frac{y-\mu_{j} }{ \sigma_{j}  }\right)dy}}\right) +
\mu_{j} \nonumber
\end{eqnarray}
where $\sigma_{j} = \Sigma_{0}(j,j)$.
\end{definition}

\begin{definition}{\rm (Symmetric Power Transformation)}
Let $g_{0}$ be the symmetric and odd transformation 
given by
\begin{eqnarray}
g_{0}(t) = \mathrm{sign}(t)|t|^{\alpha} \nonumber
\end{eqnarray}
where $\alpha>0$ is a parameter.
We define the power
transformation for the $j$-th dimension as
\begin{eqnarray}
g_{j}(z_{j}) \equiv \ds \sigma_{j} \left( \frac{\ds g_0(z_{j} - \mu_{j})}{\sqrt{{\ds\int} g^2_0(t-\mu_{j})
\phi\left(\frac{t-\mu_{j}}{\sigma_{j}}\right)dt}}\right) + \mu_{j}. \nonumber
\end{eqnarray}
\end{definition}

These transformation are constructed to preserve the marginal
mean and standard deviation. In the following experiments, we refer to
them as the cdf transformation and the power transformation, respectively. For
the cdf transformation, we set $\mu_{g_0} = 0.05$ and
$\sigma_{g_0}=0.4$. For the power transformation, we set $\alpha
=3$.

\begin{figure}[t]
\vskip-25pt
\begin{center}
\begin{tabular}{c}
\hskip-20pt
\includegraphics[width=.75\textwidth, angle=-90]{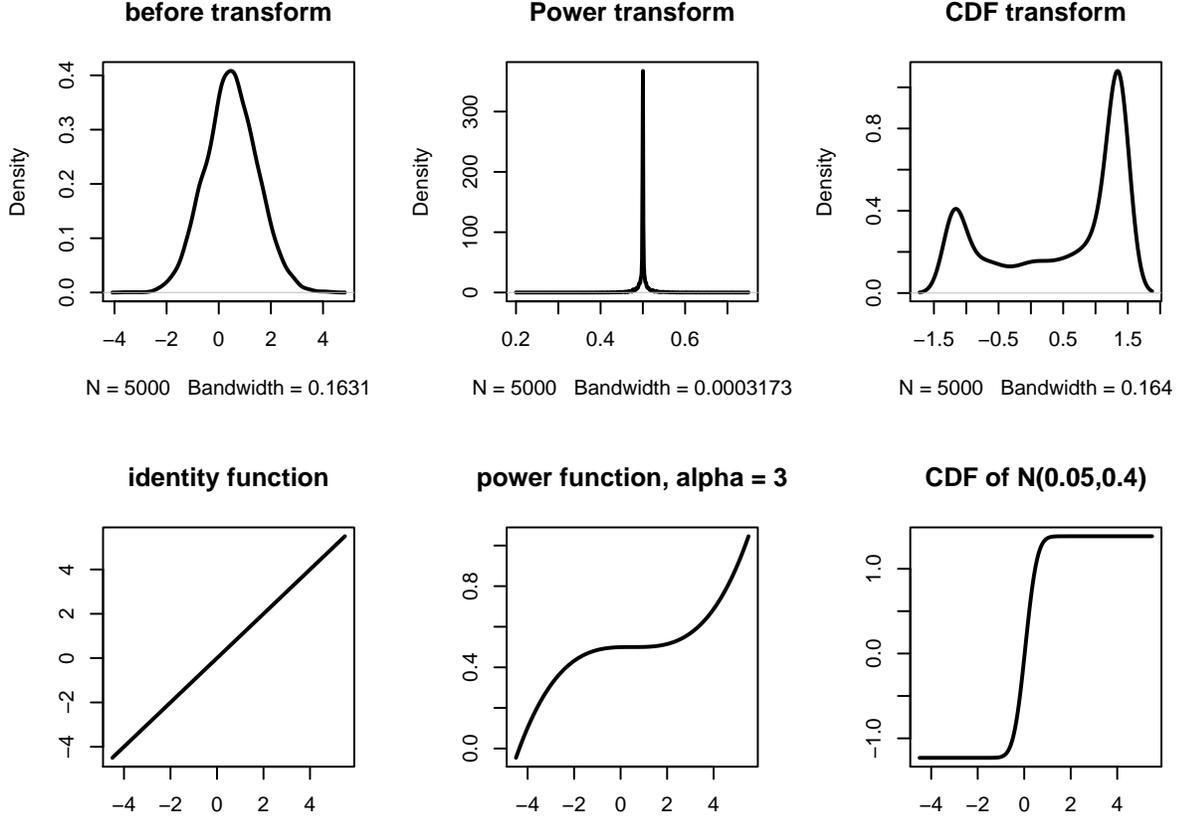}\\[-30pt]
\end{tabular}
\end{center}
\caption{\small The power and cdf transformations. The densities are estimated using kernel density estimator with bandwidths selected by unbiased cross-validation.}
\label{fig1}
\vskip15pt
\end{figure}

To visualize these two transformations, we sample $5000$ data points
from a one-dimensional normal distribution ${N}(0.5, 1.0)$
and then apply the above two transformations; the results are shown
in Figure \ref{fig1}.  It can be seen how the cdf and power
transformations map a univariate normal distribution into a highly skewed and a bi-modal
distribution, respectively.

To generate synthetic data, we set $p=40$, resulting in $\binom{40}{2}
+ 40 = 820$ parameters to be estimated, and vary the sample sizes from
$n=200$ to $n=1000$.  Three conditions are considered, corresponding
to using the cdf transform, the power transform, or no transformation.
In each case, both the glasso and the nonparanormal are applied to
estimate the graph.

\subsubsection{Comparison of regularization paths}

We choose a set of
regularization parameters $\Lambda$; for each $\lambda \in
\Lambda$, we obtain an estimate $\hat{\Omega}_n$ which is a
$40 \times 40$ matrix. The upper triangular matrix has 780 parameters; we
can vectorize it to get a 780-dimensional parameter vector. A
regularization path is trace of these parameters over
all the regularization parameters within $\Lambda$. The
regularization paths for both methods are plotted in Figure~\ref{fig.paths}.
For the cdf transformation and the power transformation, the
nonparanormal separates the relevant and the irrelevant
dimensions very well. For the glasso, relevant variables are
mixed with irrelevant variables. If no transformation is applied, the
paths for both methods are almost the same.

\begin{figure*}
\vskip-30pt
\begin{center}
\begin{tabular}{ccc}
\scriptsize \bf cdf & \scriptsize \bf power &\scriptsize \bf linear \\[-15pt]
\includegraphics[width=.30\textwidth,angle=-90]{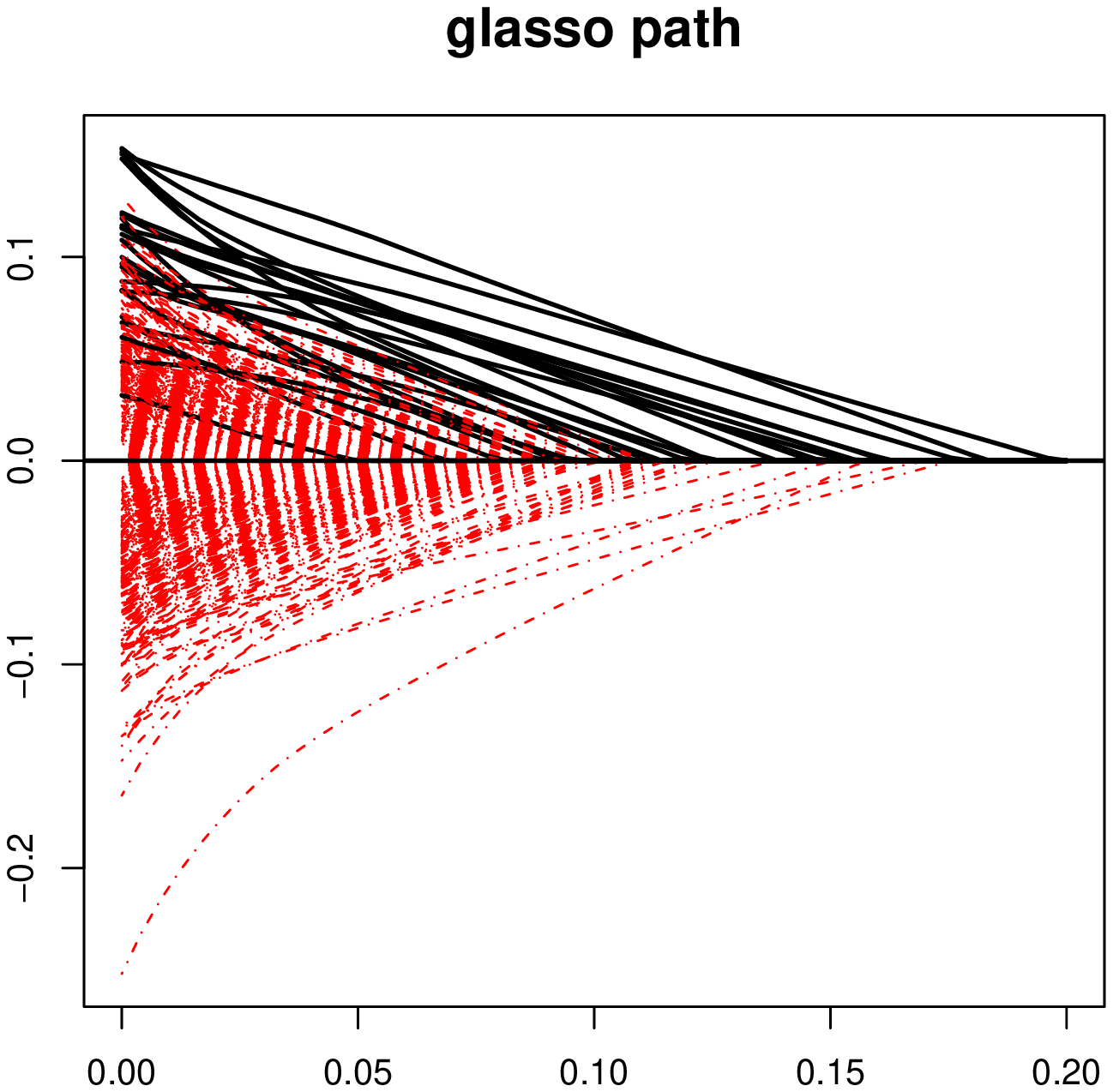} &
\includegraphics[width=.30\textwidth,angle=-90]{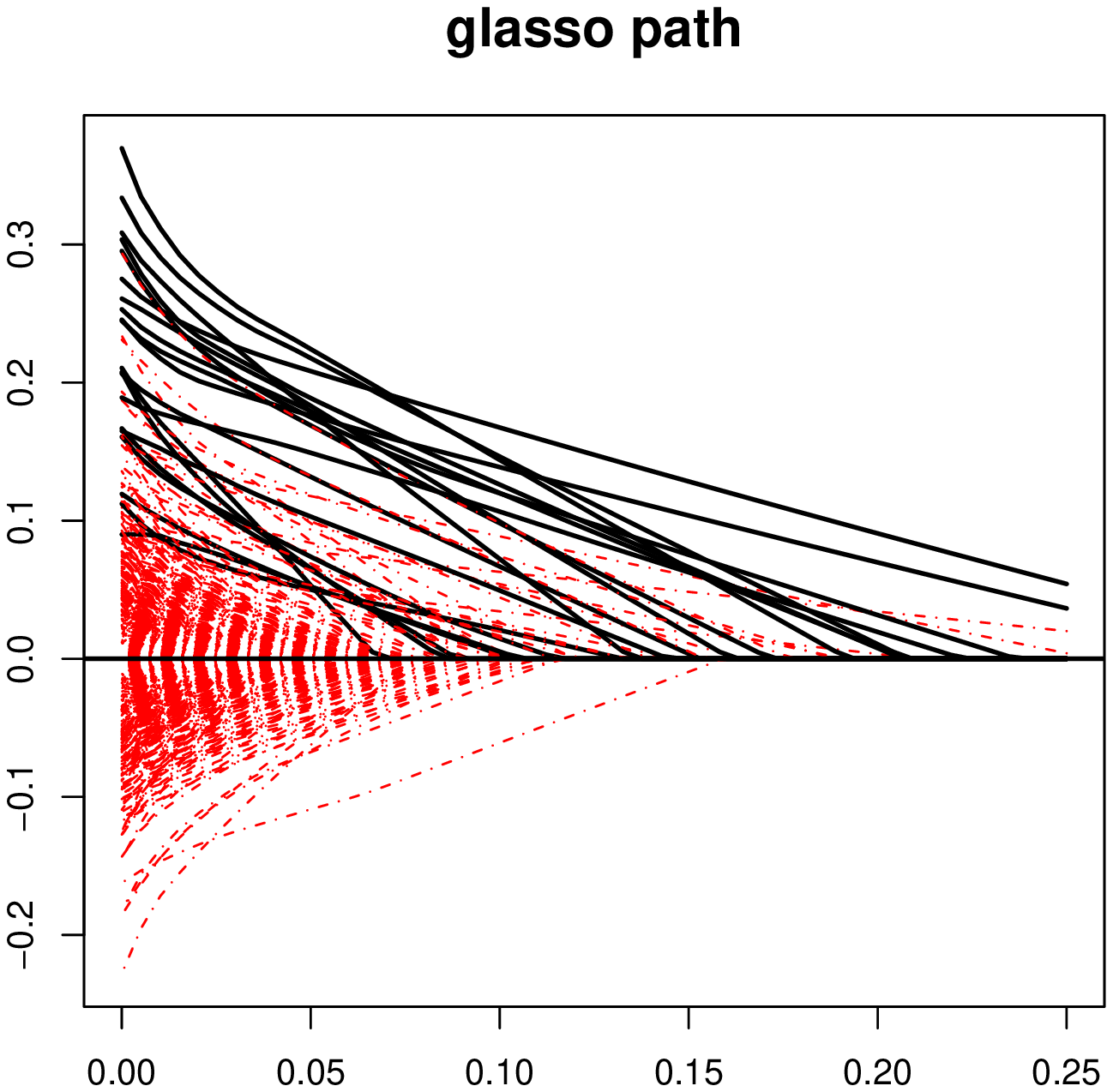} &
\includegraphics[width=.30\textwidth,angle=-90]{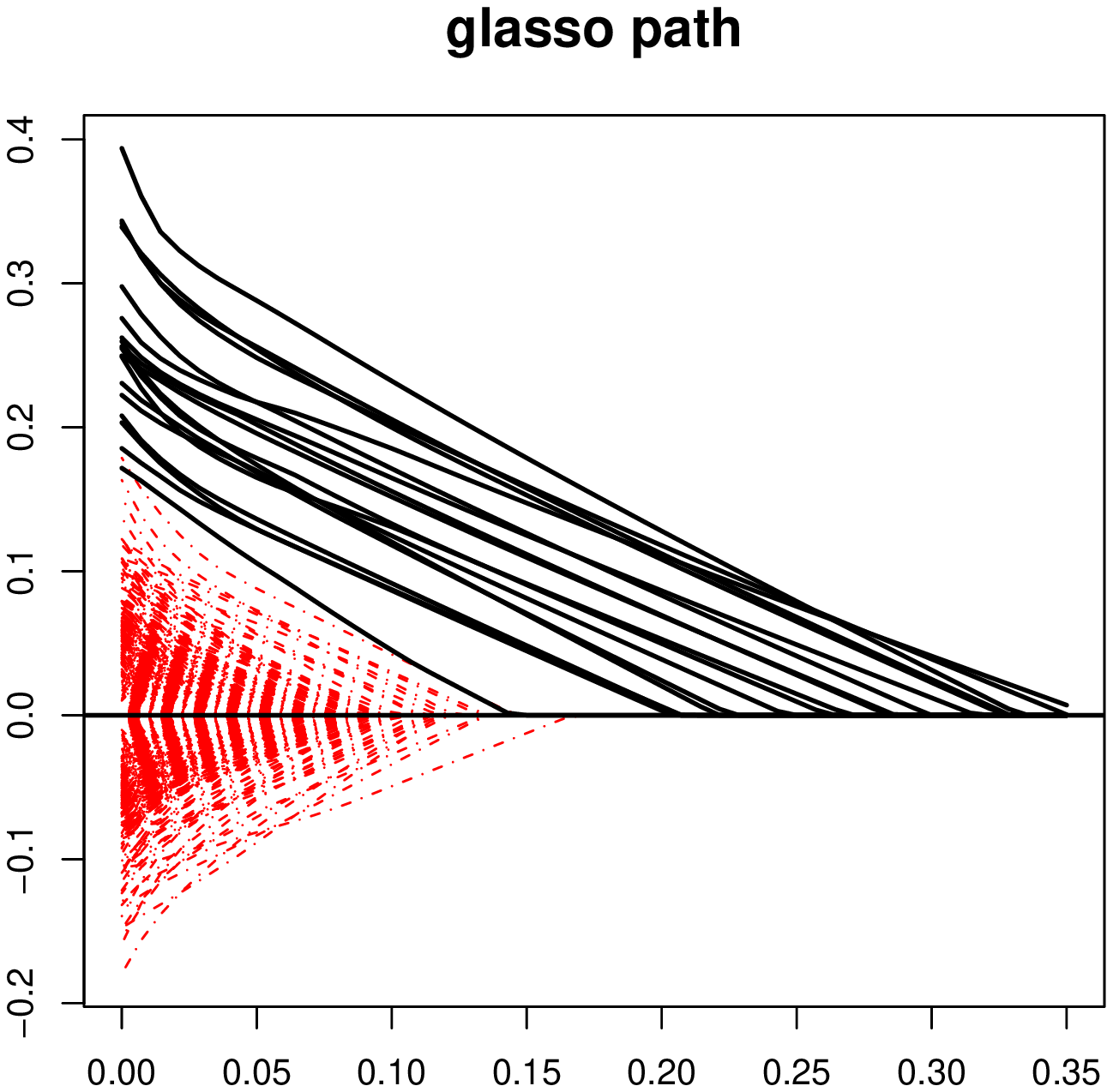} \\[-20pt]
\includegraphics[width=.30\textwidth,angle=-90]{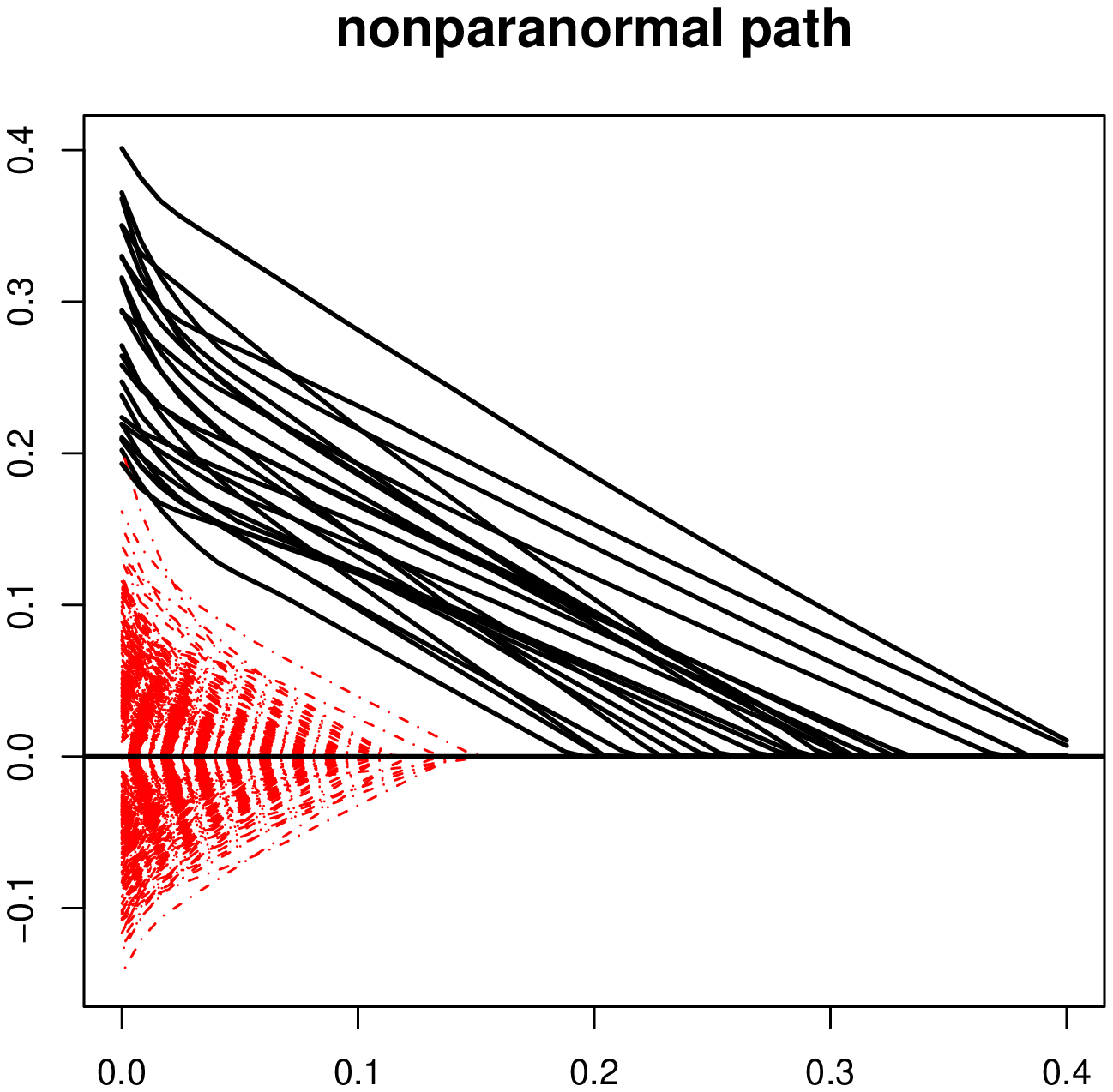} &
\includegraphics[width=.30\textwidth,angle=-90]{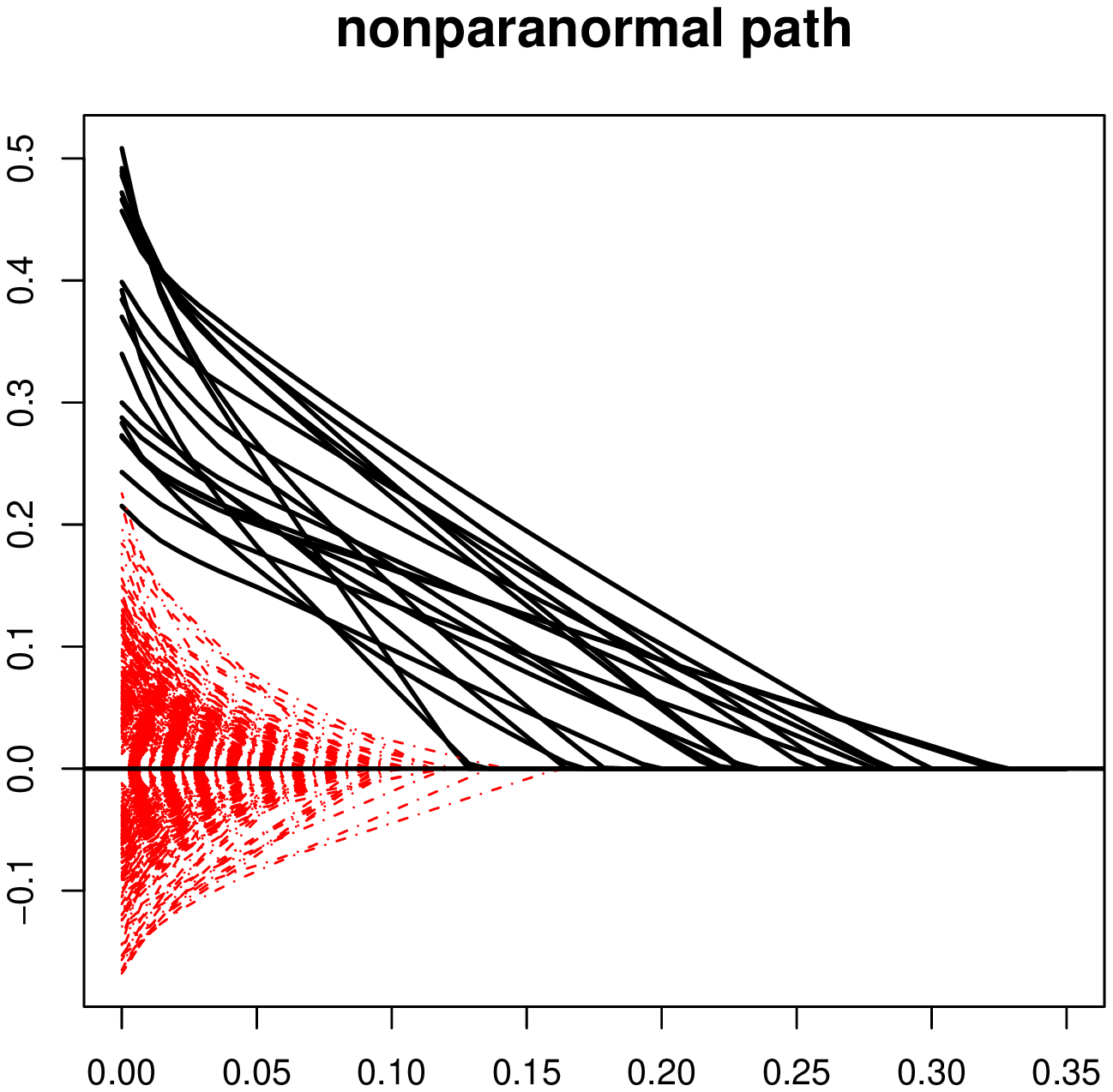} &
\includegraphics[width=.30\textwidth,angle=-90]{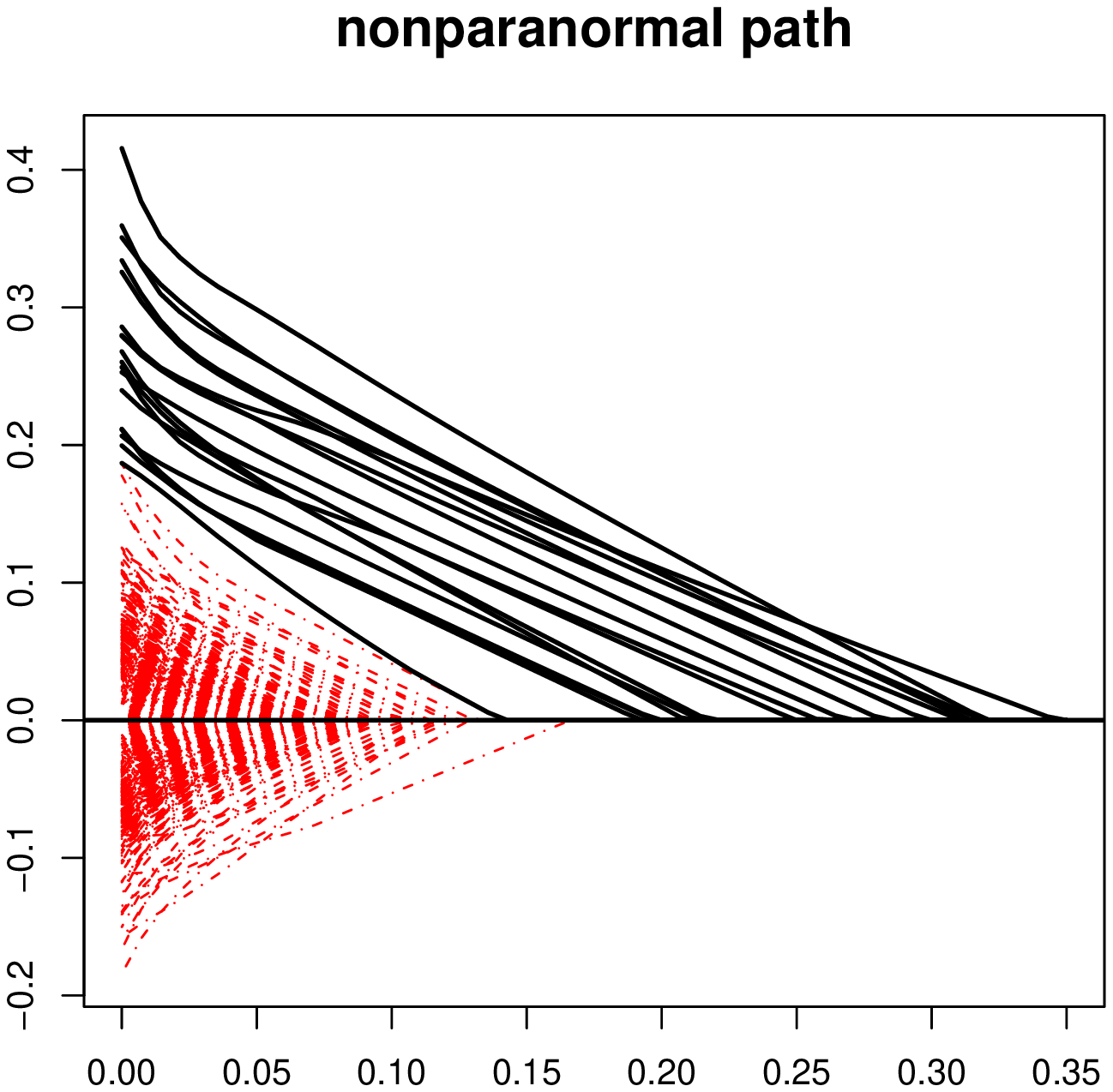}\\[-5pt]
& $n=500$ & \\[10pt]
\hline \\[-5pt]
\scriptsize \bf cdf & \scriptsize \bf power &\scriptsize \bf linear \\[-15pt]
\includegraphics[width=.30\textwidth,angle=-90]{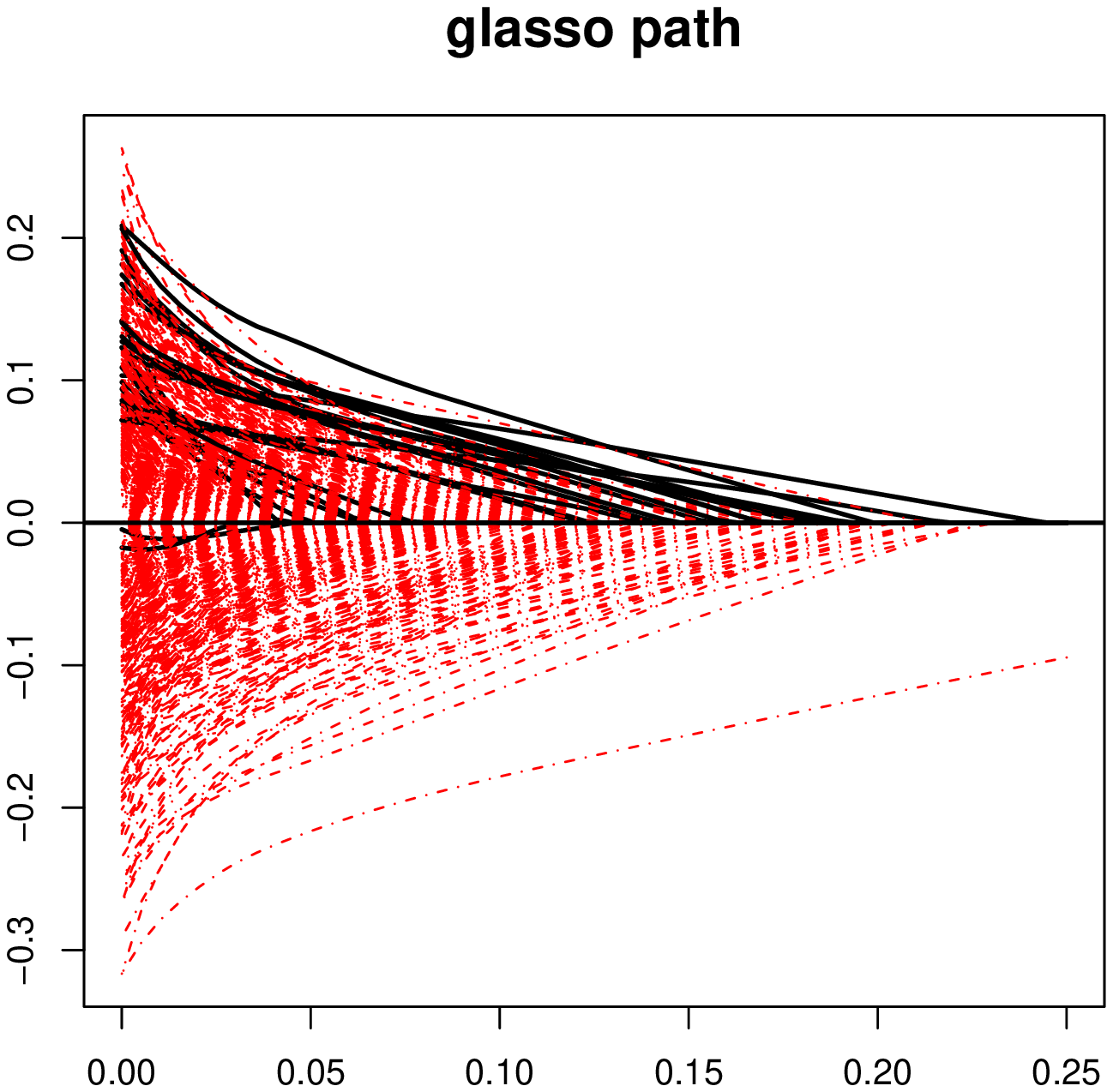} &
\includegraphics[width=.30\textwidth,angle=-90]{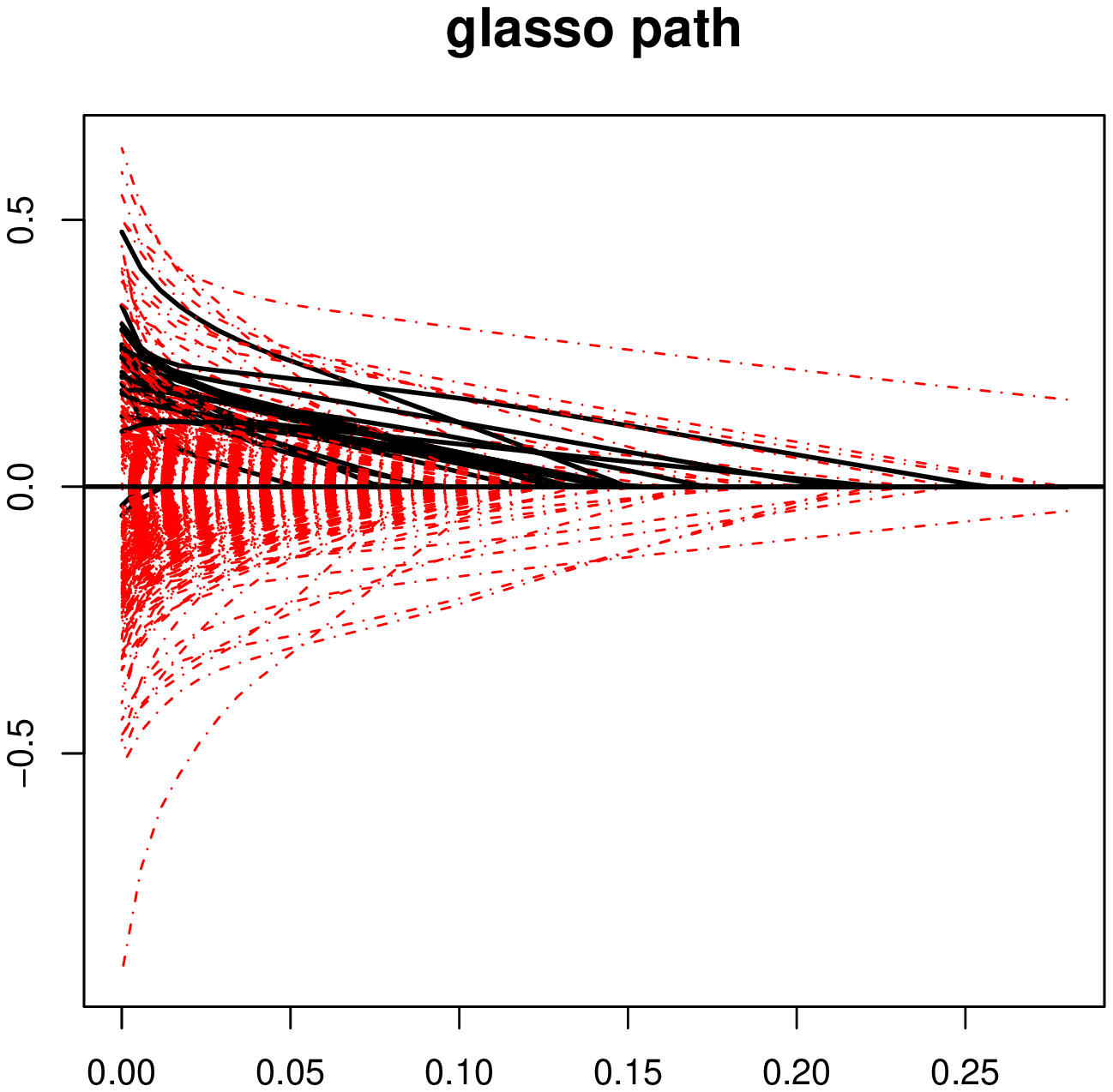} &
\includegraphics[width=.30\textwidth,angle=-90]{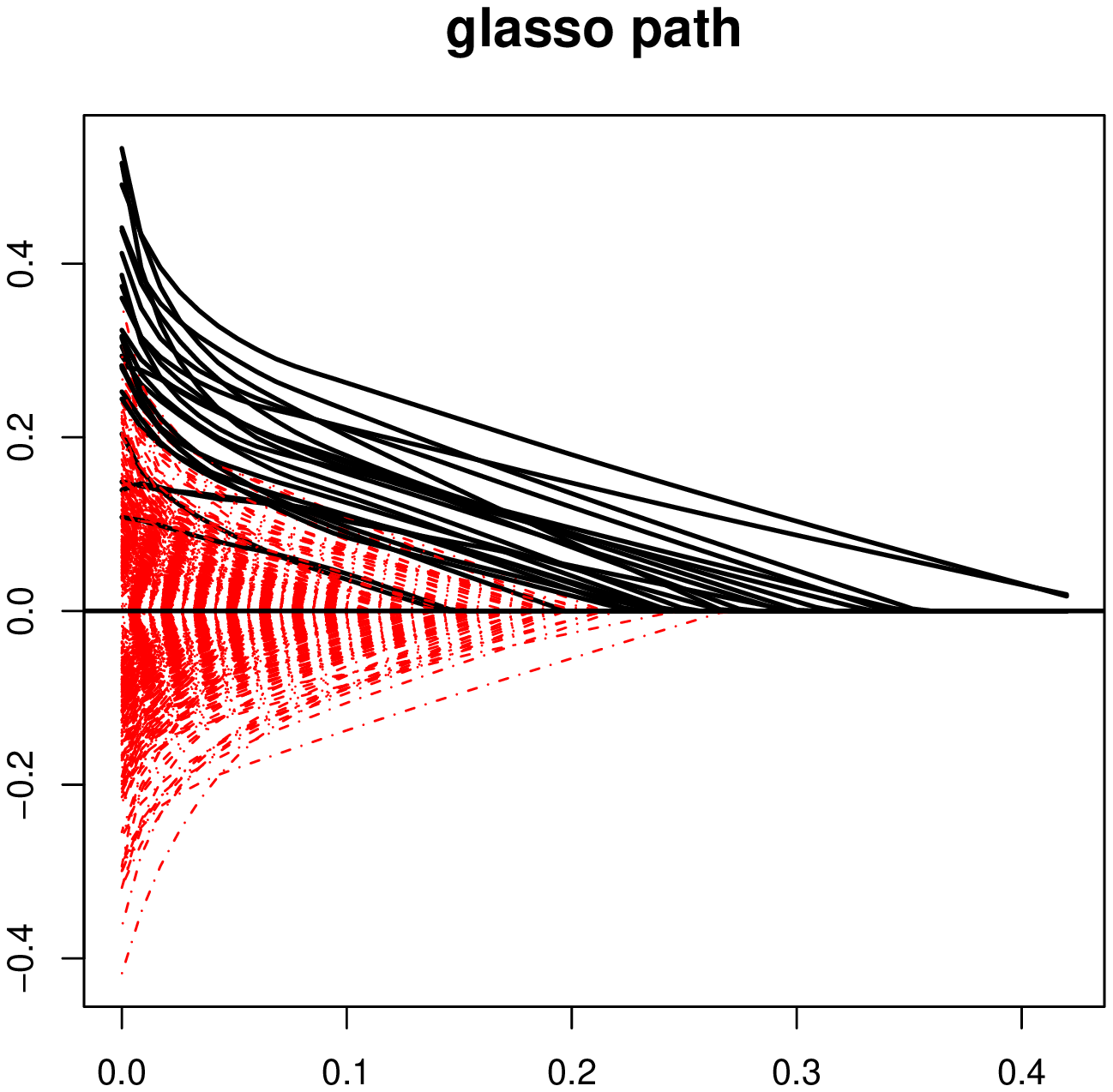} \\[-20pt]
\includegraphics[width=.30\textwidth,angle=-90]{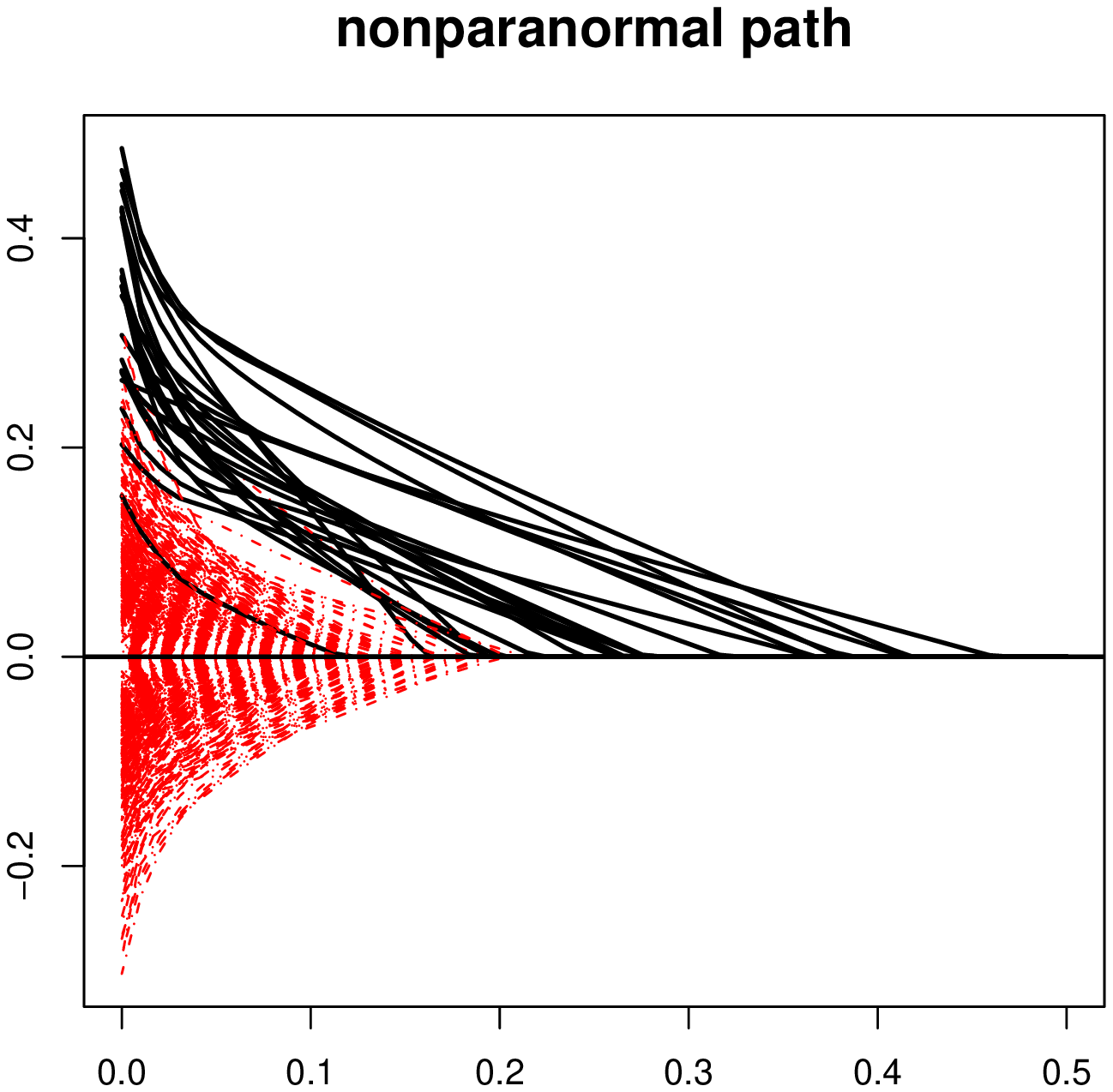} &
\includegraphics[width=.30\textwidth,angle=-90]{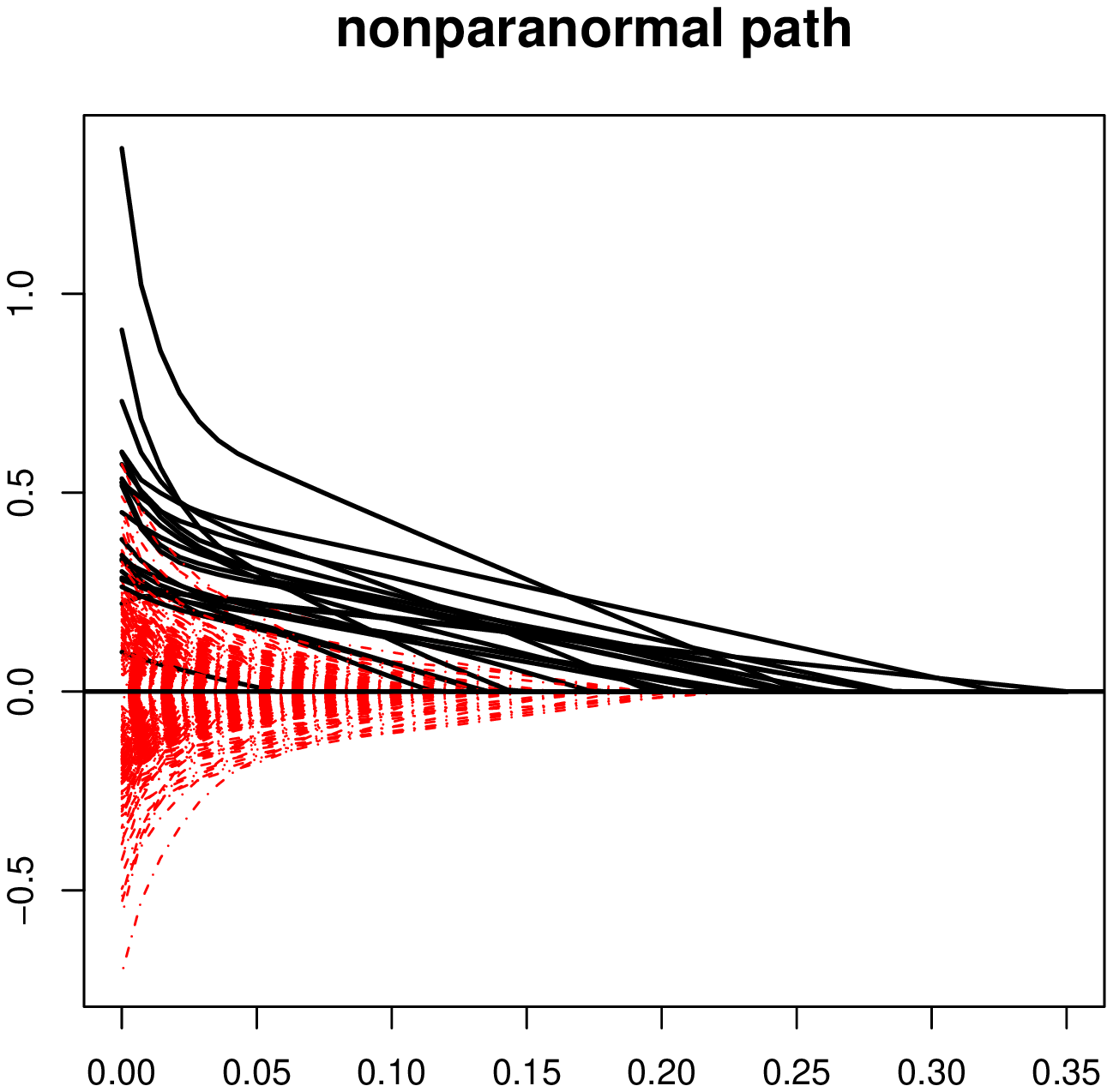} &
\includegraphics[width=.30\textwidth,angle=-90]{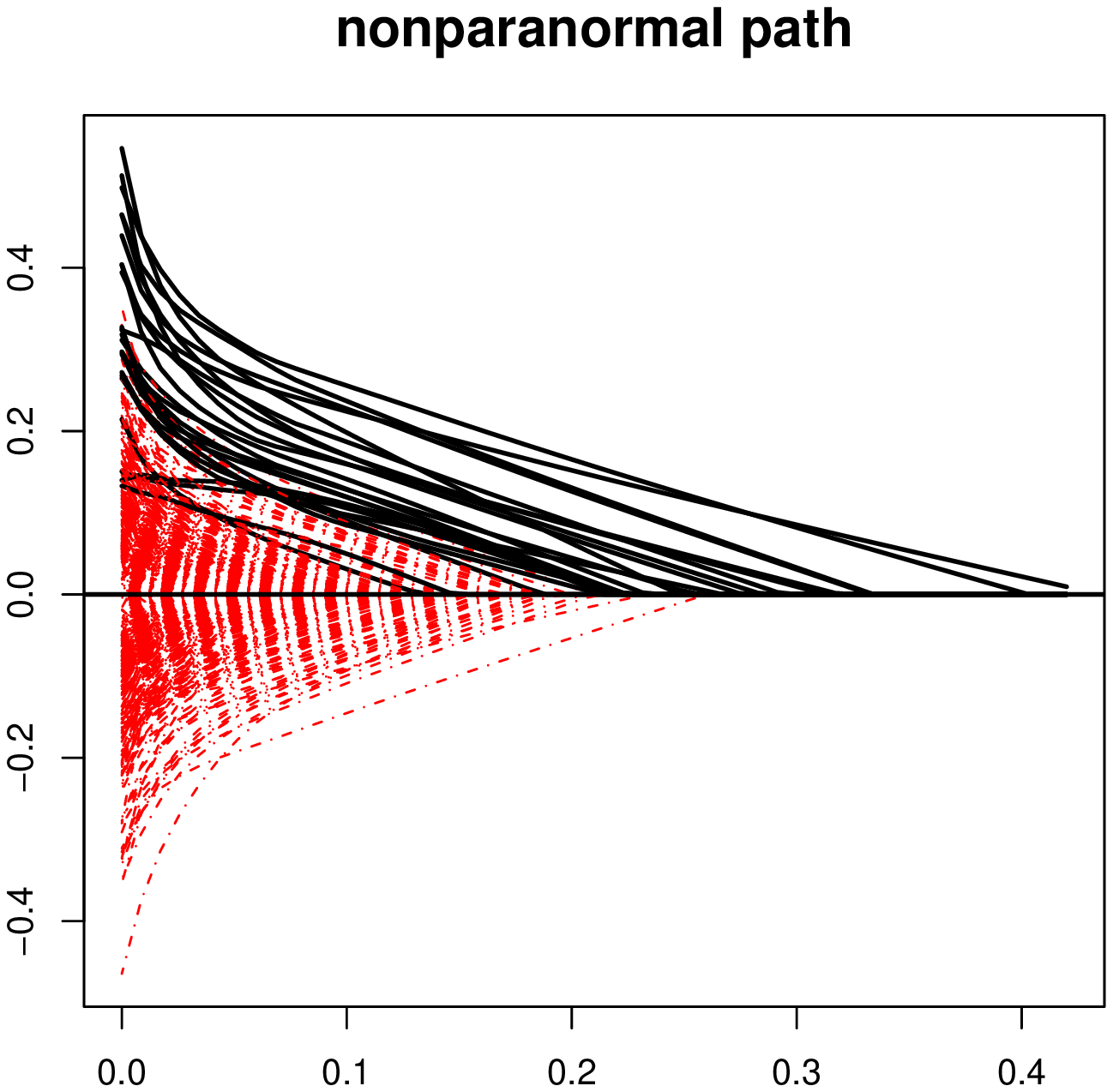}\\[-5pt]
& $n=200$ & \\[-20pt]
\end{tabular}
 \end{center}
\caption{\small Regularization paths for the glasso and nonparanormal
  with $n=500$ (top) and $n=200$ (bottom). 
The paths for the relevant variables (nonzero inverse covariance entries) are
plotted as solid (black) lines; the paths for the irrelevant variables are
plotted as dashed (red) lines.}
\label{fig.paths}
\end{figure*}

\subsubsection{Estimated transformations}

For sample size $n=1000$, we plot the estimated transformations
for three of the variables in Figure \ref{fig.components}.
It is clear that Winsorization plays a significant role for the power transformation. This is intuitive due to the high skewness of the nonparanormal distribution resulting from the power transformations.

\begin{figure}[ht]
\begin{center}
\def\hs{\hskip-11pt}
\begin{tabular}{ccc}
\\[-20pt]
\scriptsize \bf cdf & \scriptsize \bf power &\scriptsize \bf linear \\[-30pt]
\hs\includegraphics[width=.33\textwidth=-90,angle=-90]{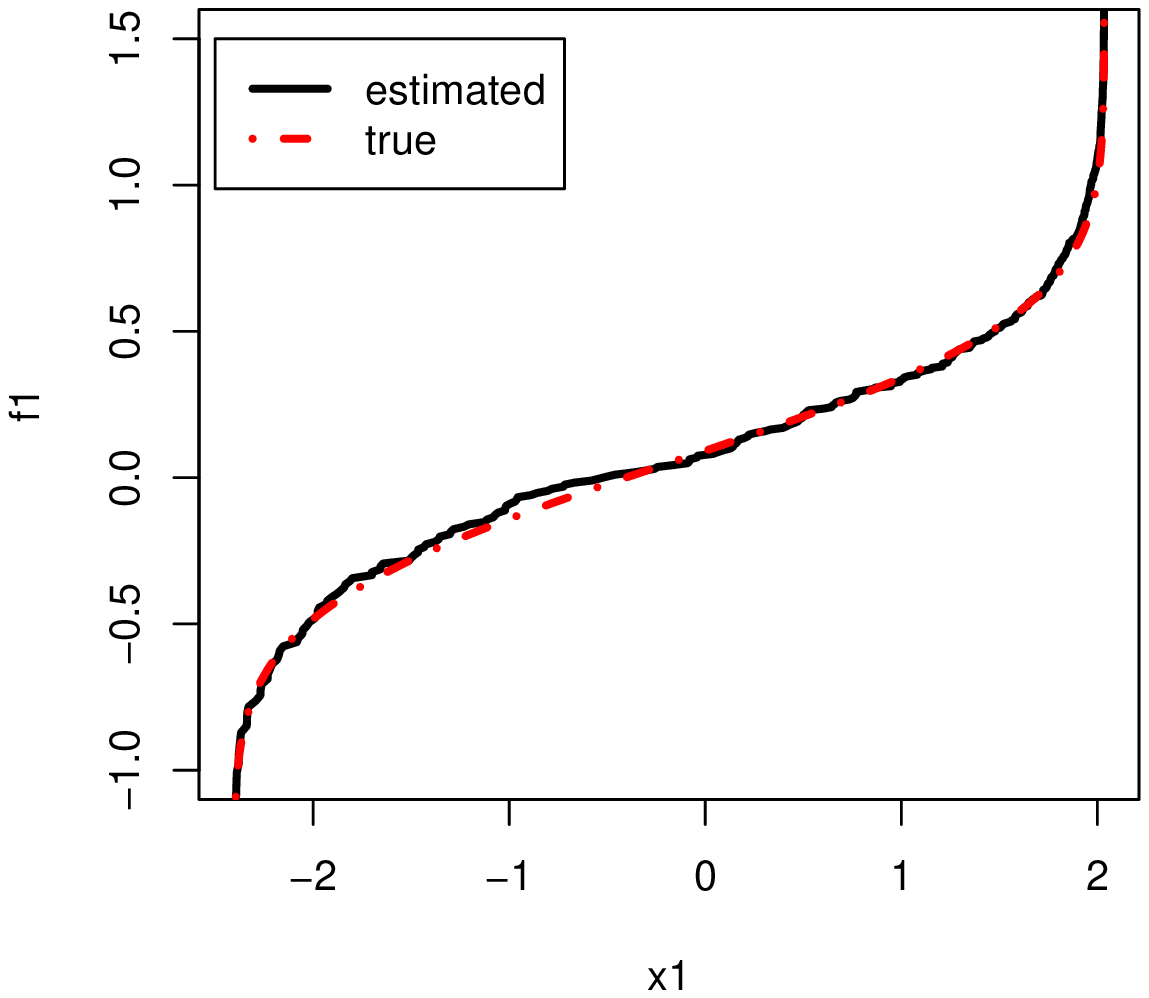} &
\hs\includegraphics[width=.33\textwidth=-90,angle=-90]{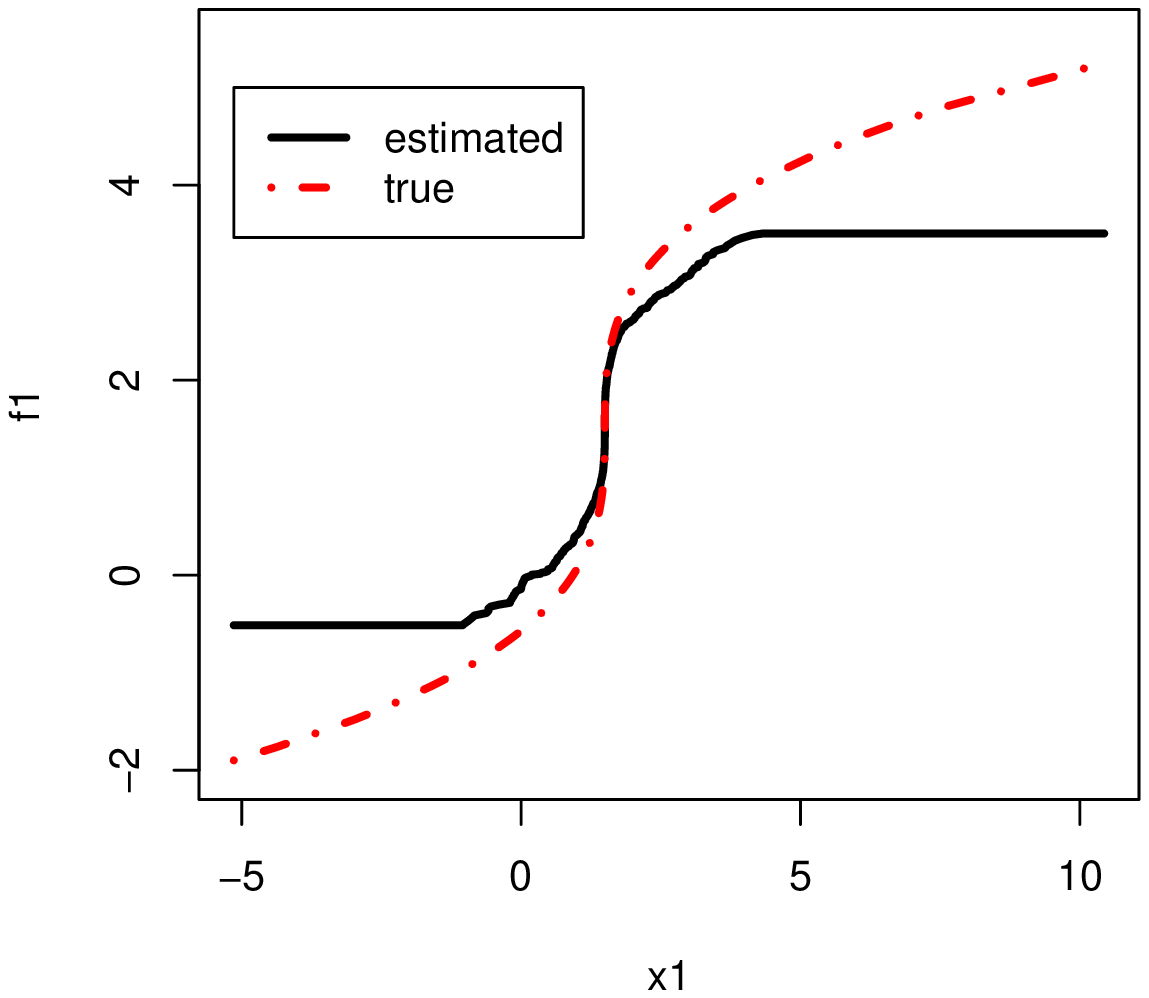} &
\hs\includegraphics[width=.33\textwidth=-90,angle=-90]{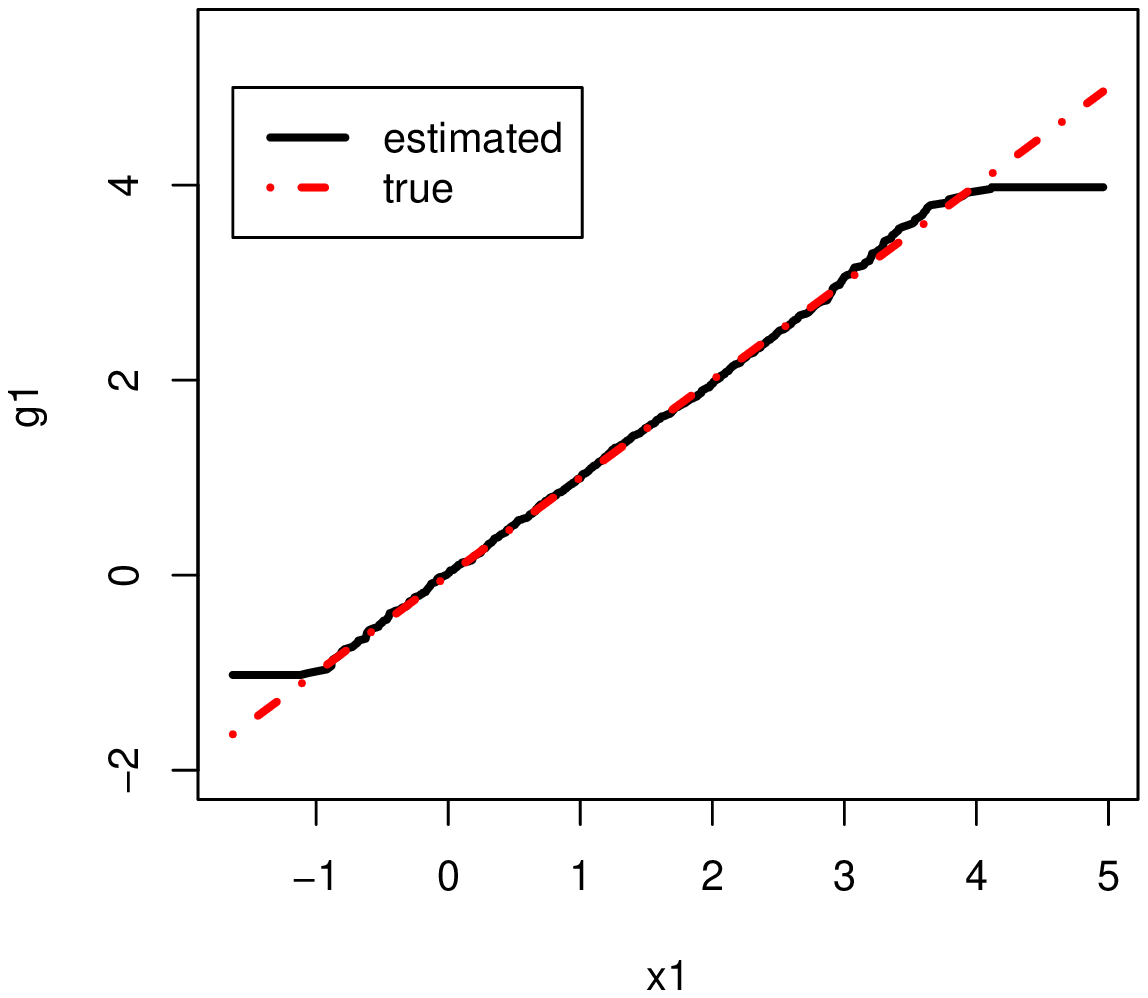} \\[-40pt]
\hs\includegraphics[width=.33\textwidth=-90,angle=-90]{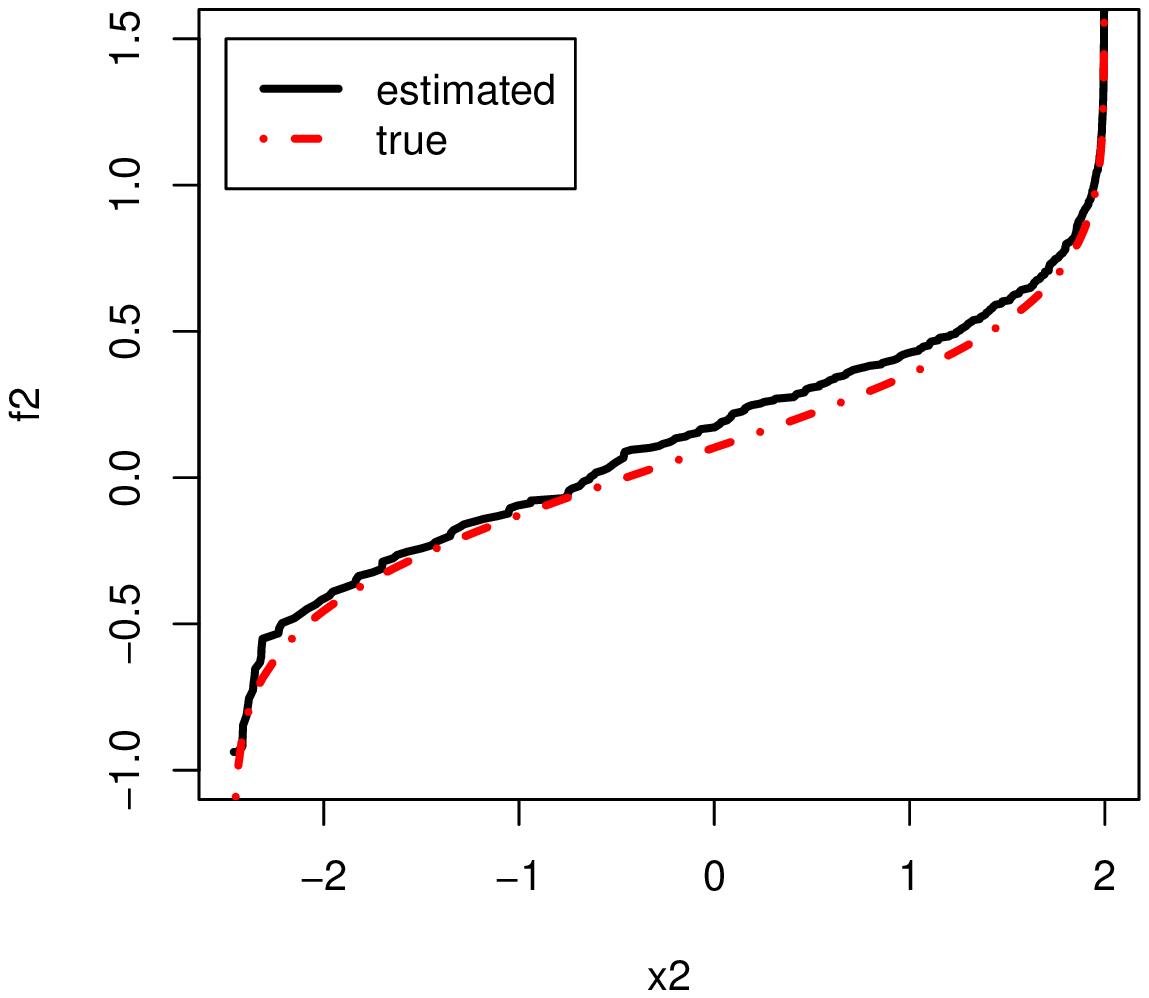} &
\hs\includegraphics[width=.33\textwidth=-90,angle=-90]{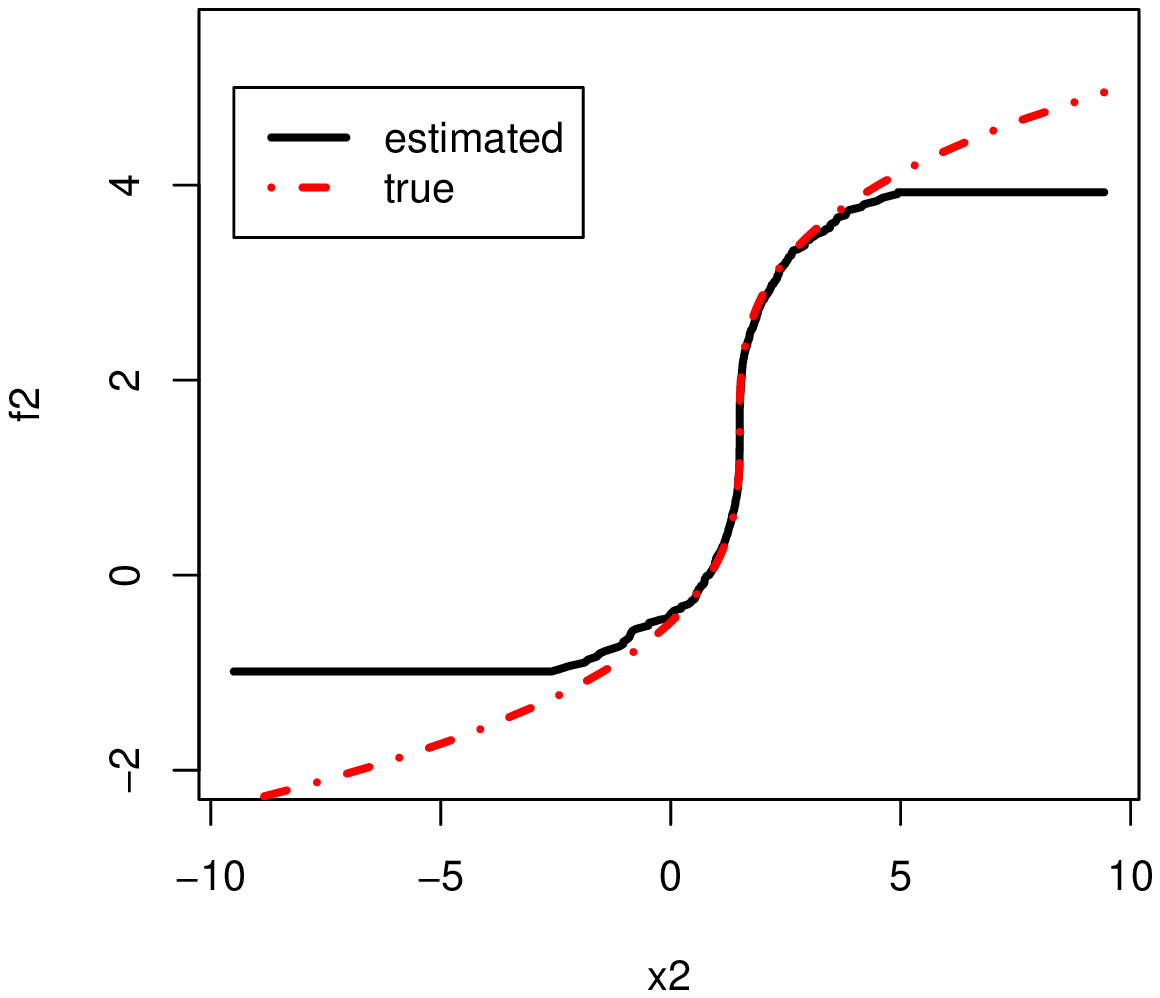} &
\hs\includegraphics[width=.33\textwidth=-90,angle=-90]{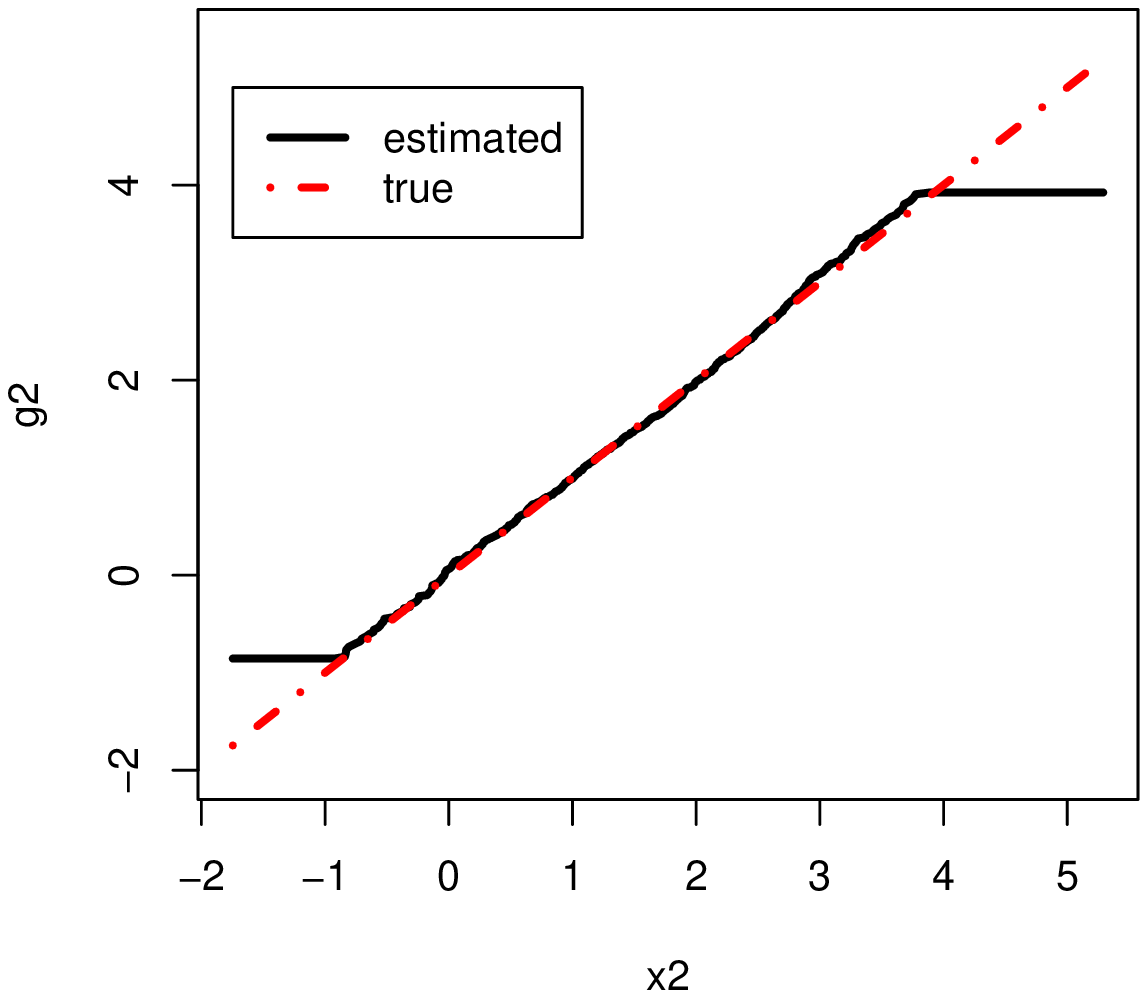} \\[-40pt]
\hs\includegraphics[width=.33\textwidth=-90,angle=-90]{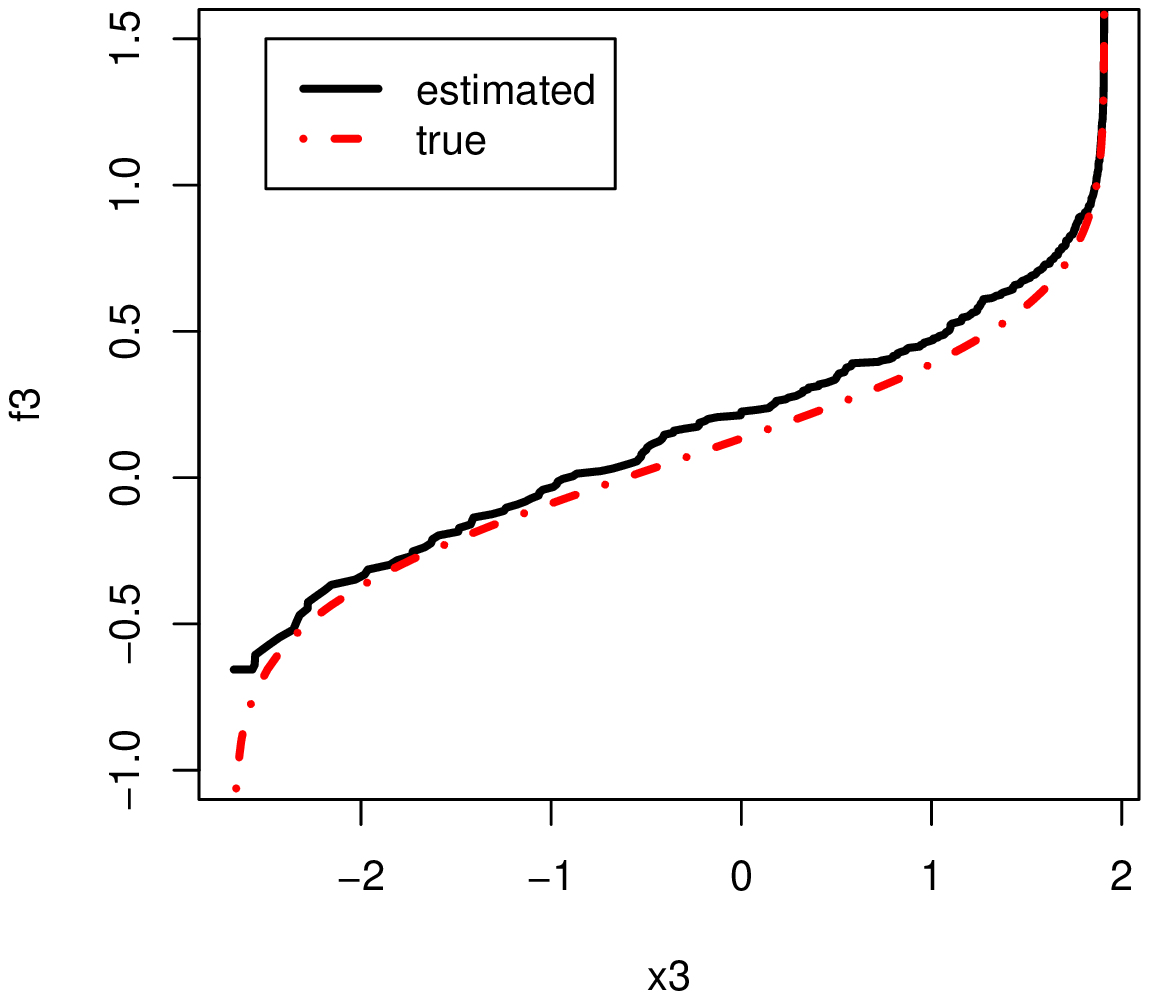} &
\hs\includegraphics[width=.33\textwidth=-90,angle=-90]{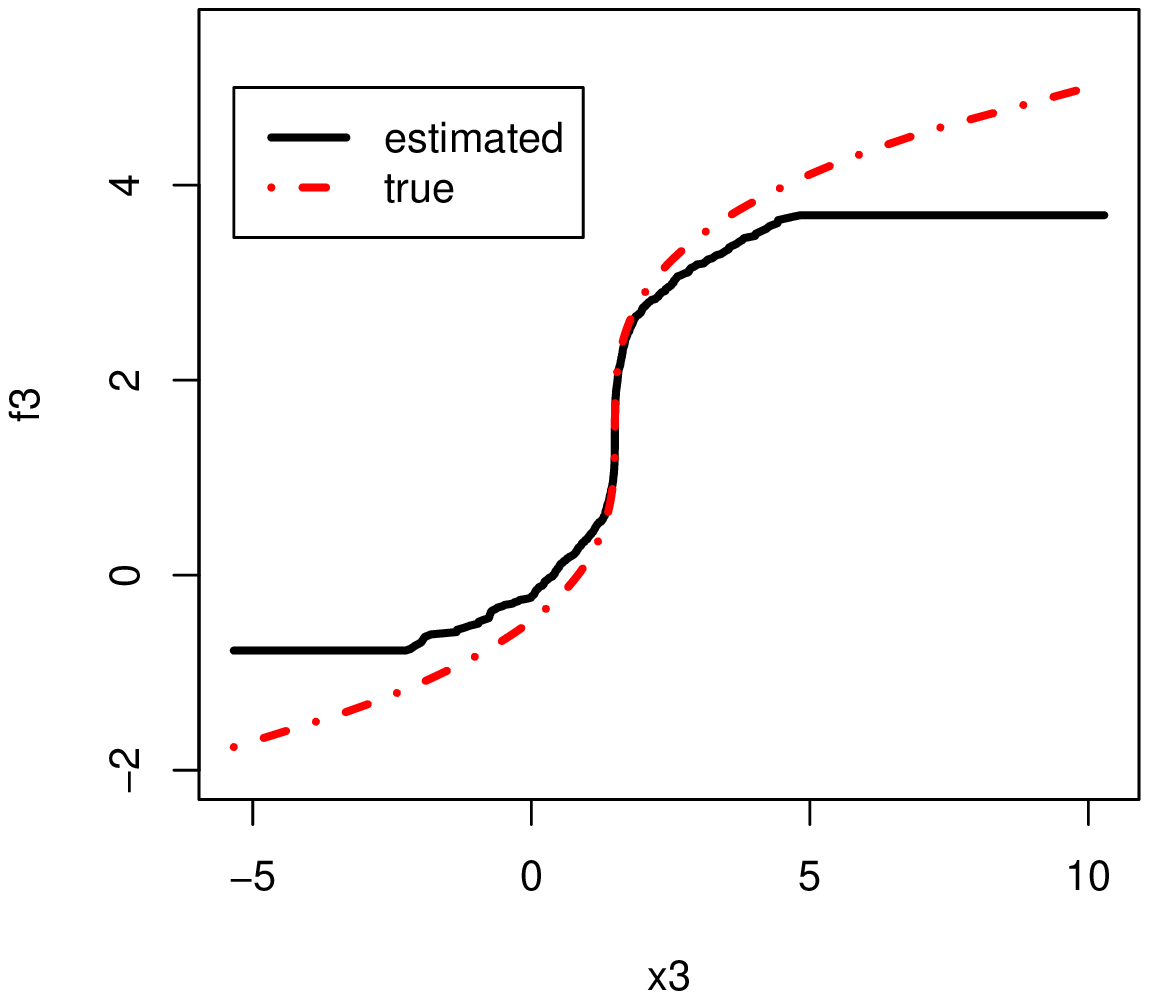} &
\hs\includegraphics[width=.33\textwidth=-90,angle=-90]{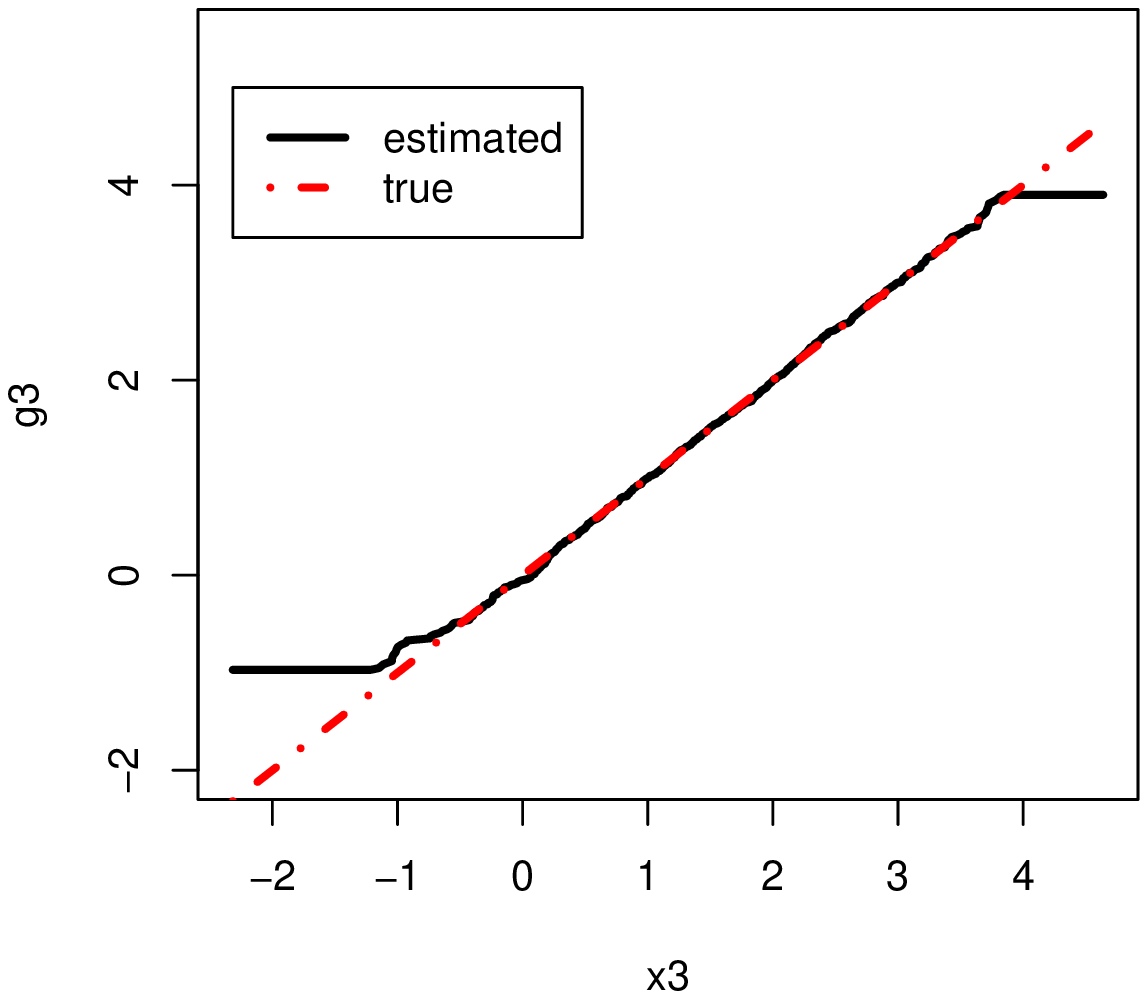} \\[-20pt]
\end{tabular}
\end{center}
\caption{\small Estimated transformations for the first three variables.}
\label{fig.components}
\end{figure}

\begin{figure}[ht]
\begin{center}
\begin{tabular}{ccc}
\\[-30pt]
\scriptsize \bf cdf & \scriptsize \bf power &\scriptsize \bf linear \\[-15pt]
\vspace{0in} 
\hskip-5pt
\includegraphics[width=.3\textwidth,angle=-90]{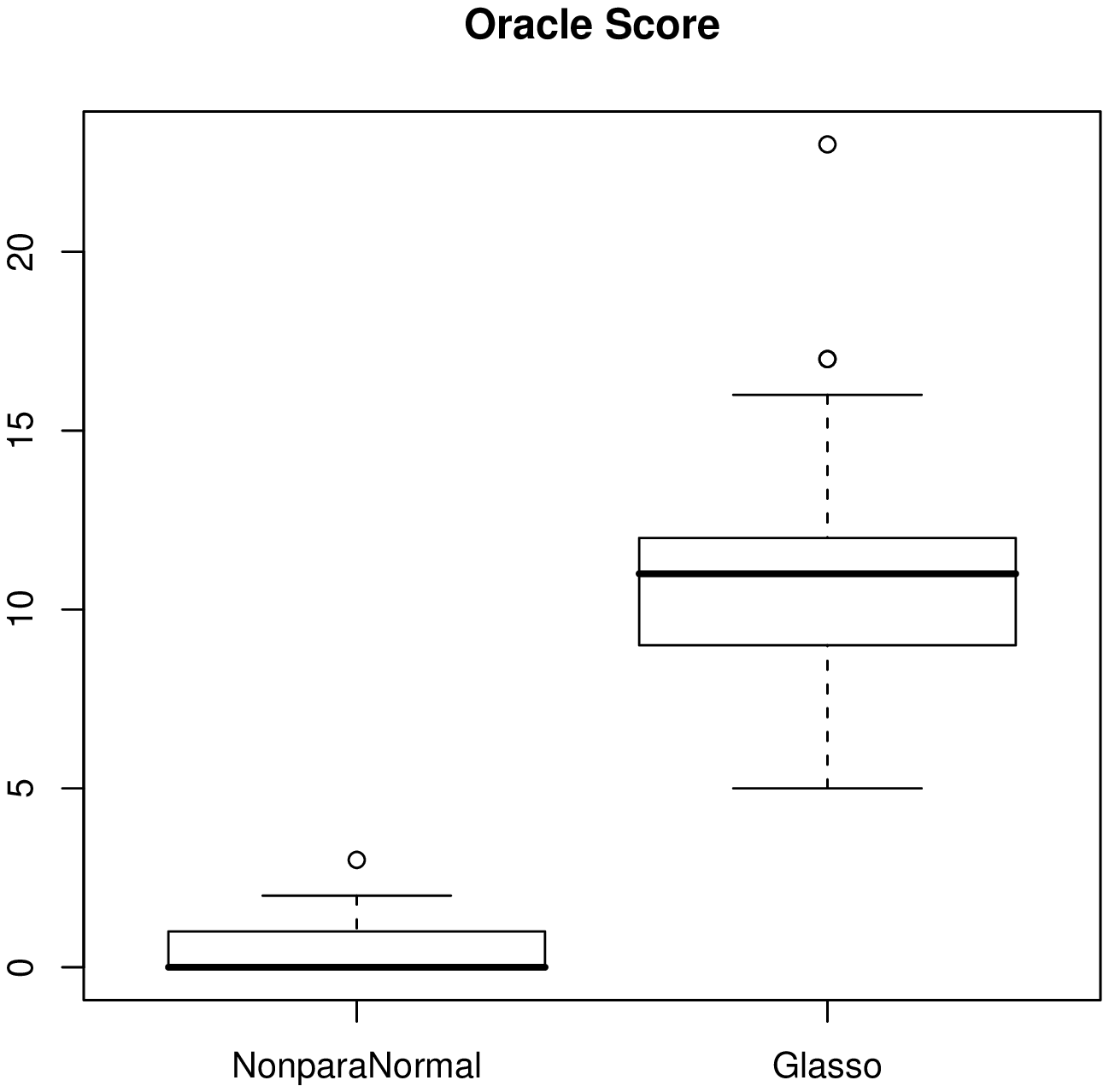} &
\hskip-5pt
\includegraphics[width=.3\textwidth,angle=-90]{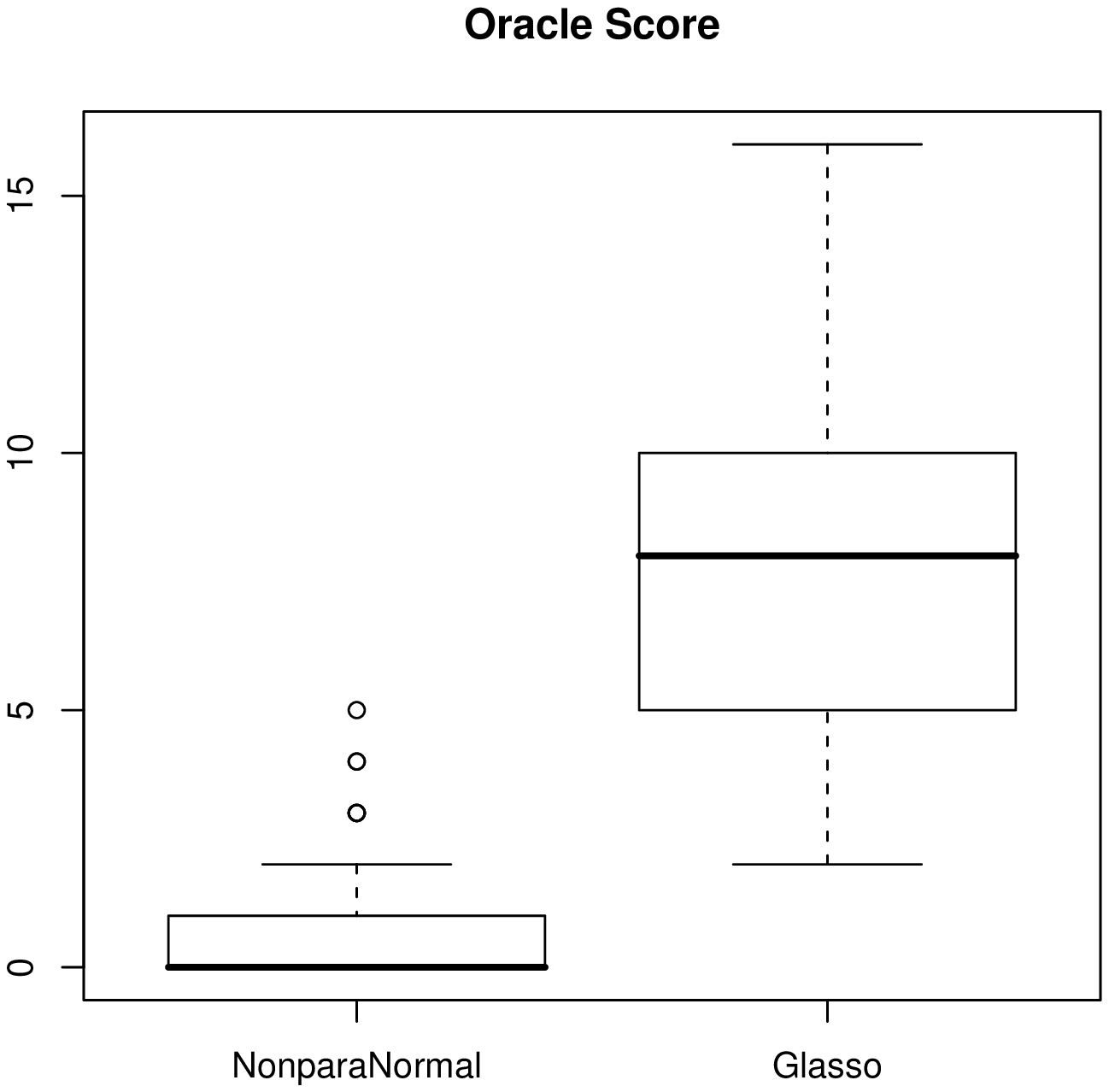}&  
\hskip-5pt
\includegraphics[width=.3\textwidth,angle=-90]{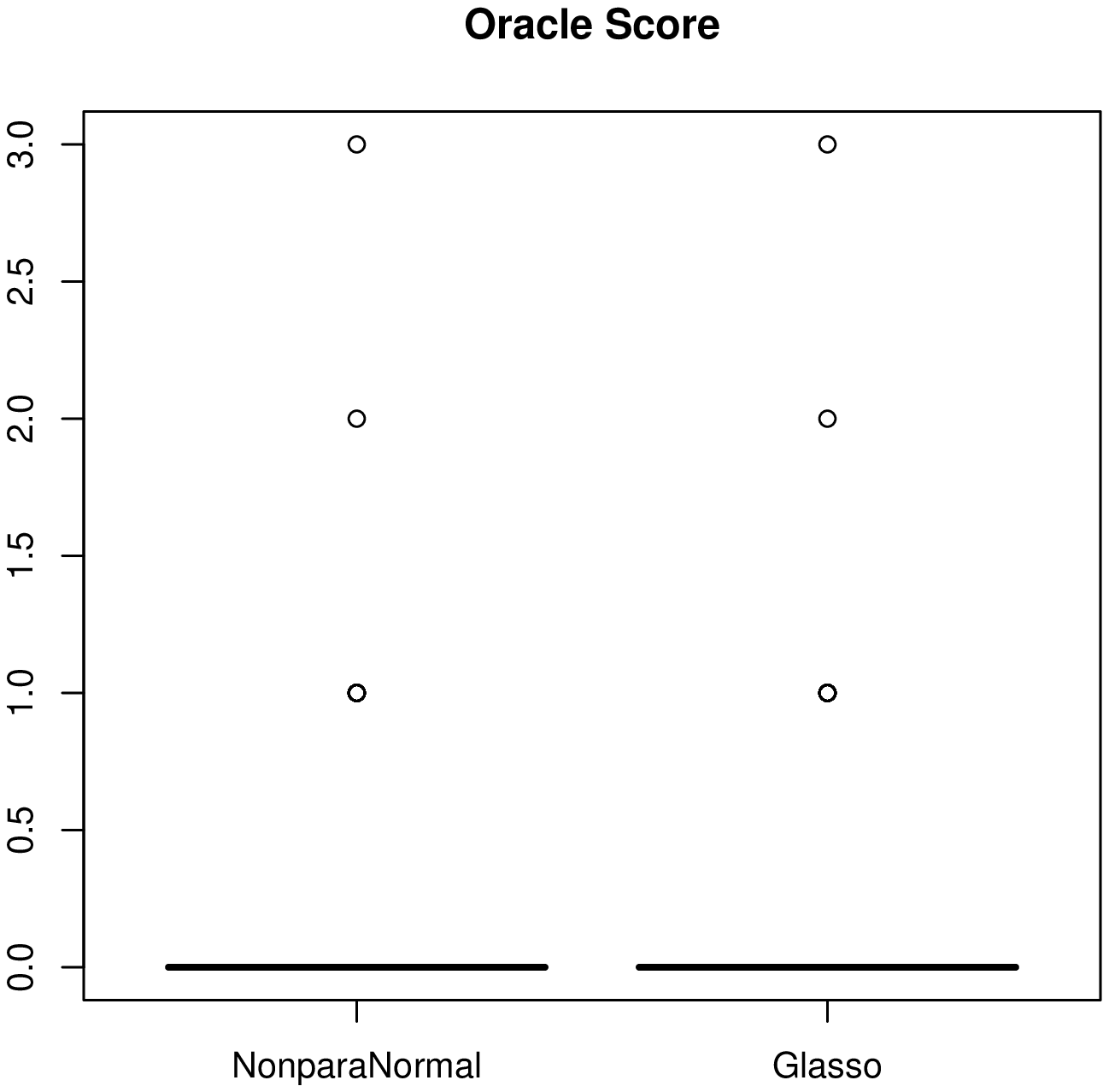}
\\[-20pt]
\hskip-5pt
\includegraphics[width=.3\textwidth,angle=-90]{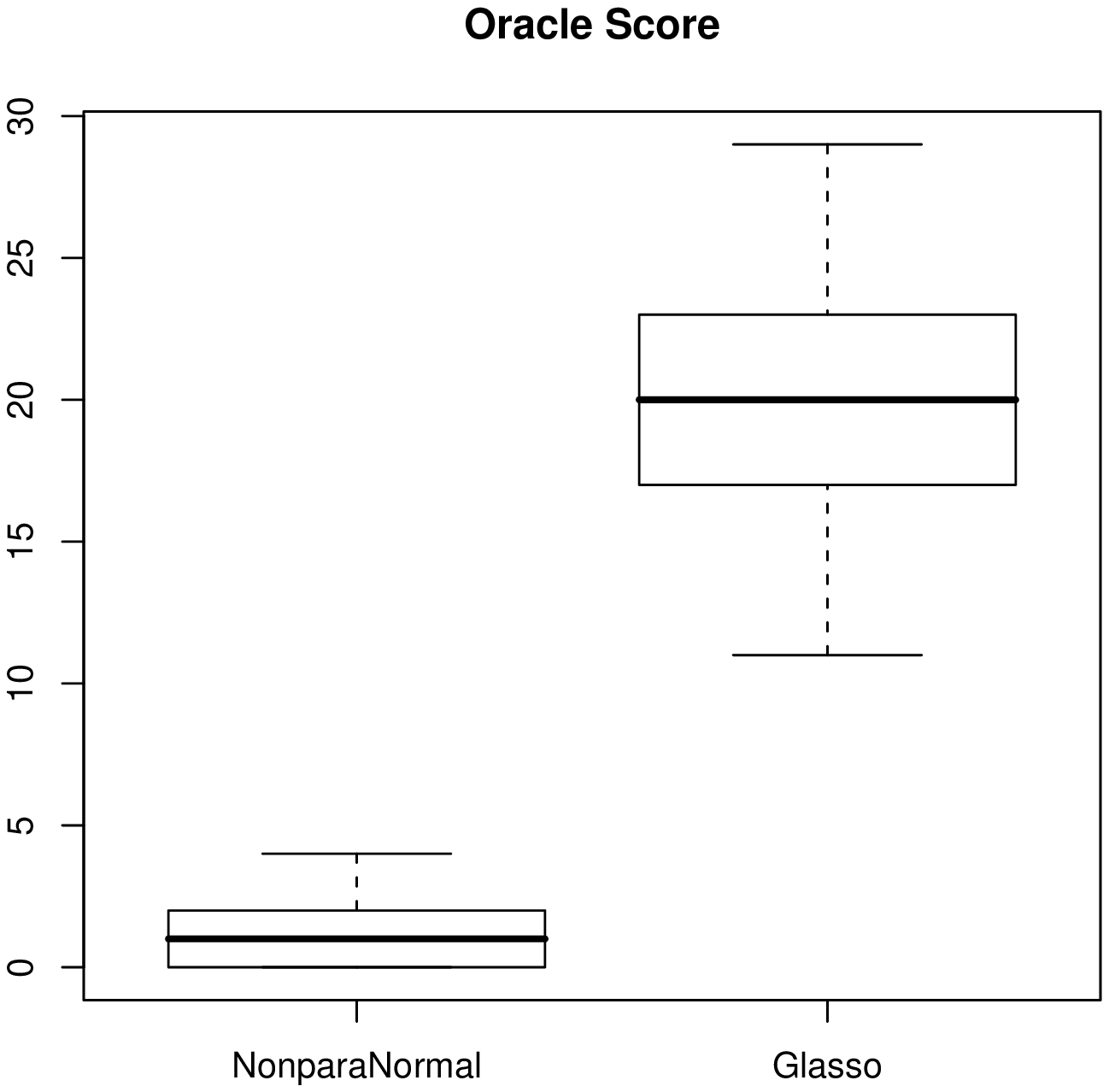} &
\hskip-5pt
\includegraphics[width=.3\textwidth,angle=-90]{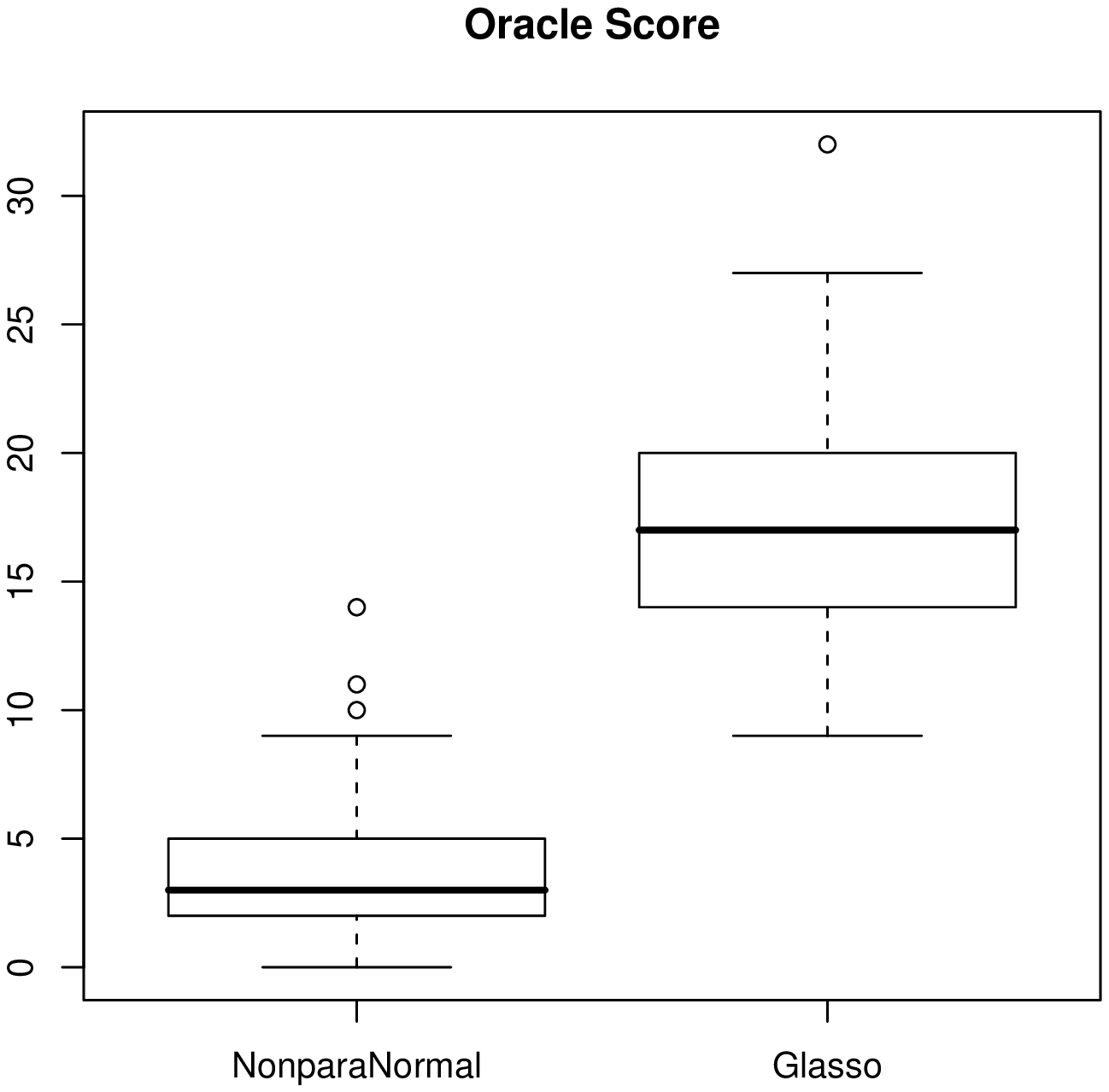}& 
\hskip-5pt
\includegraphics[width=.3\textwidth,angle=-90]{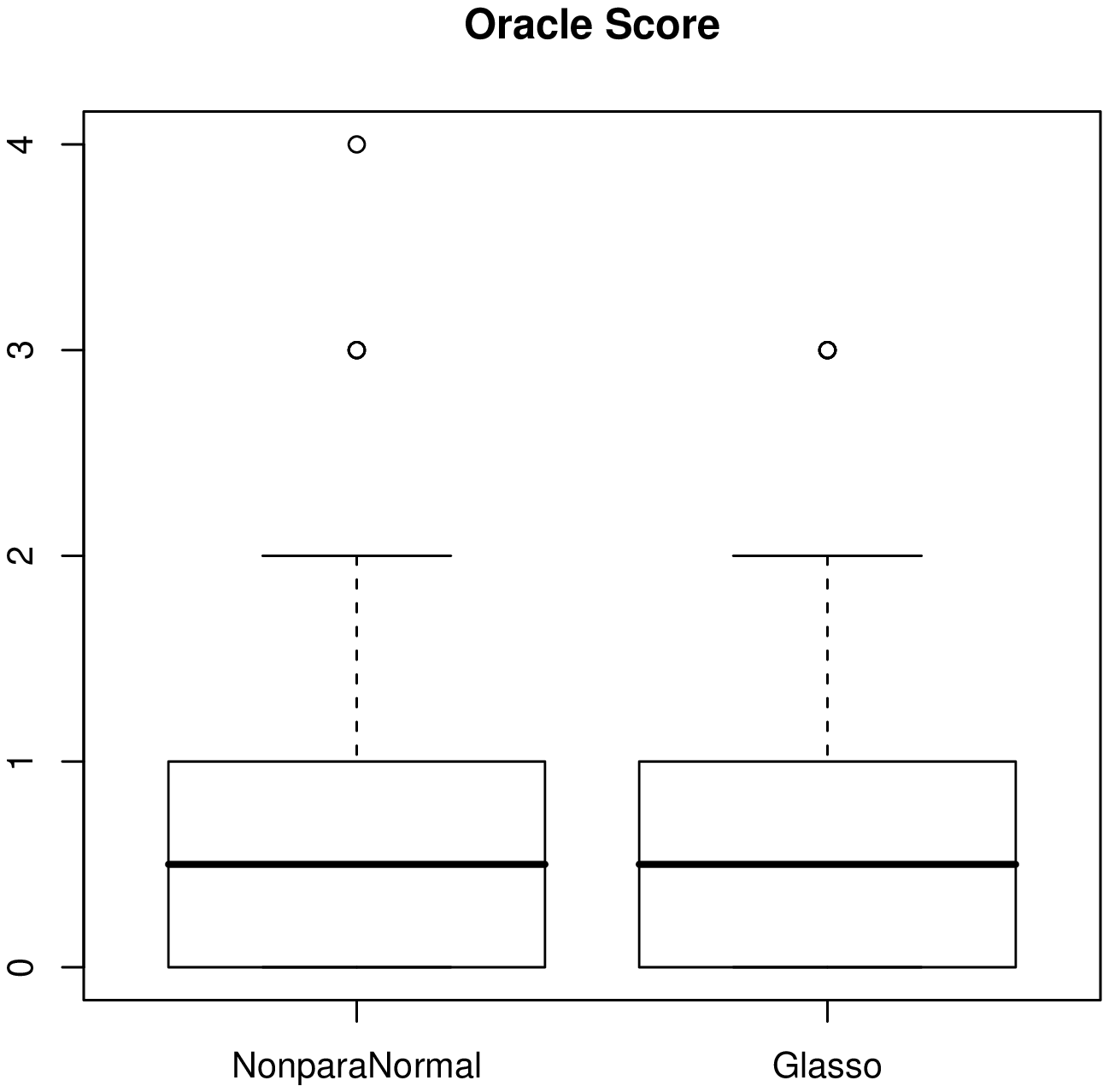}
\\[-20pt]
\hskip-5pt
\includegraphics[width=.3\textwidth,angle=-90]{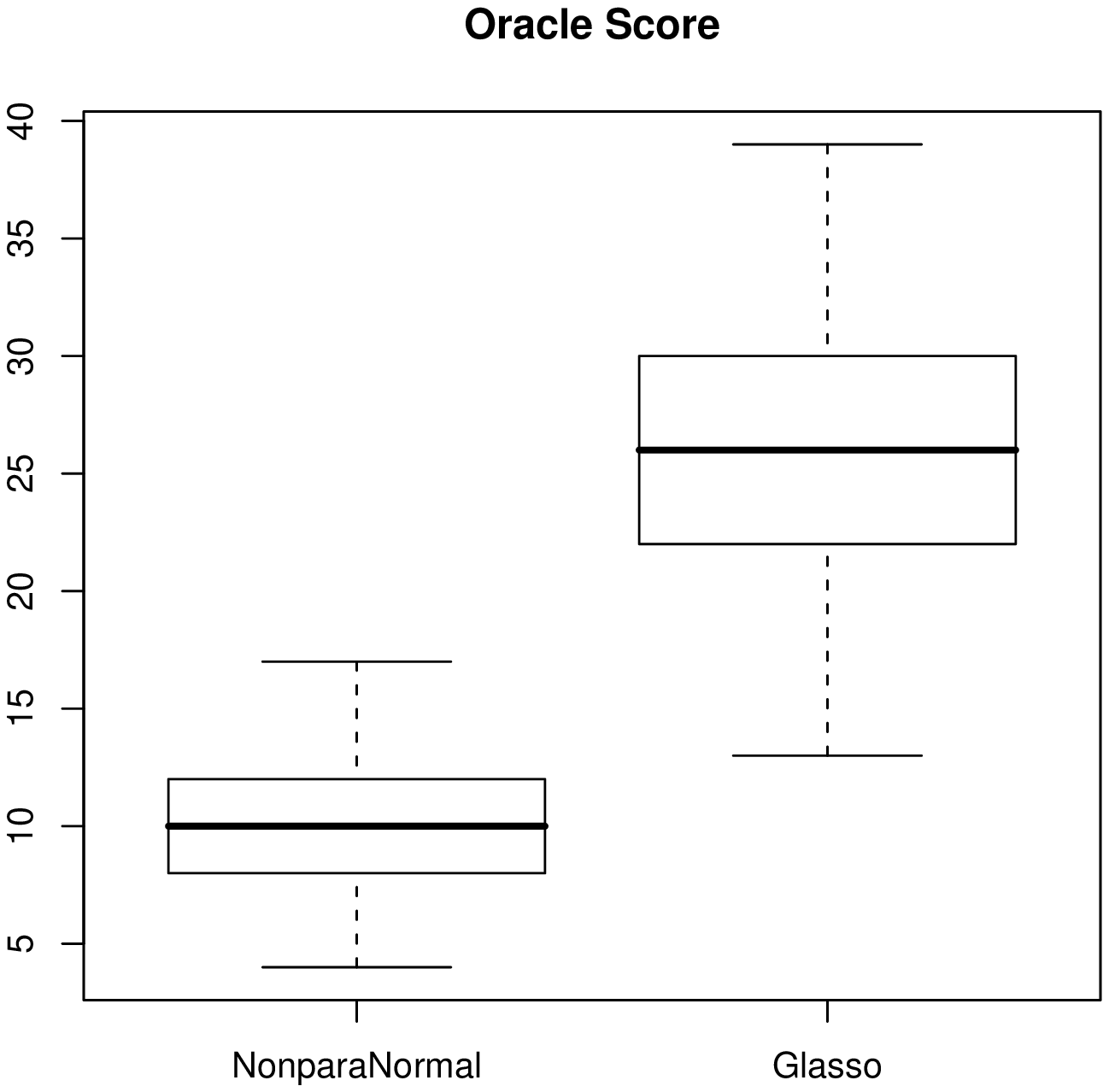} &
\hskip-5pt
\includegraphics[width=.3\textwidth,angle=-90]{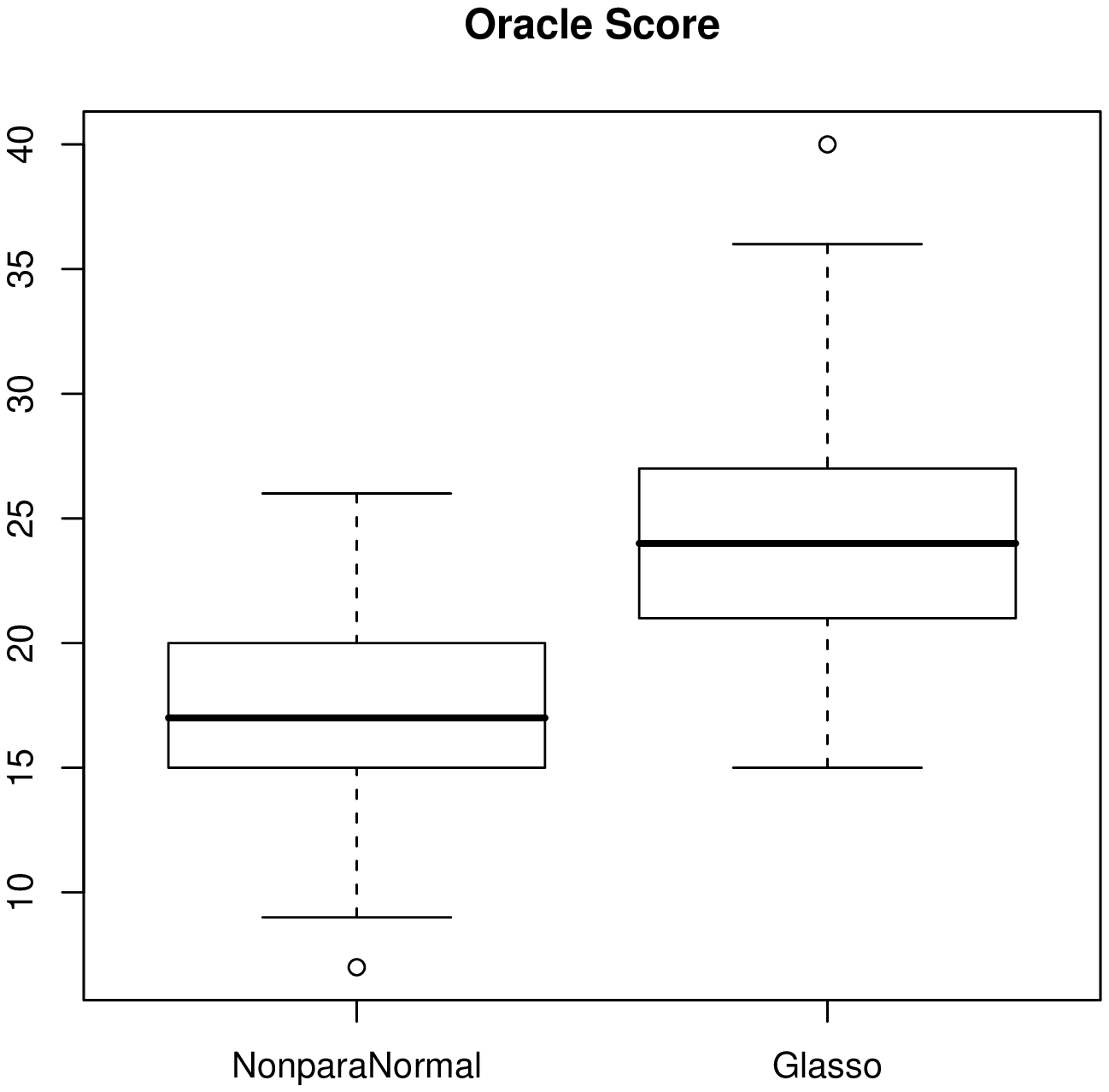}& 
\hskip-5pt
\includegraphics[width=.3\textwidth,angle=-90]{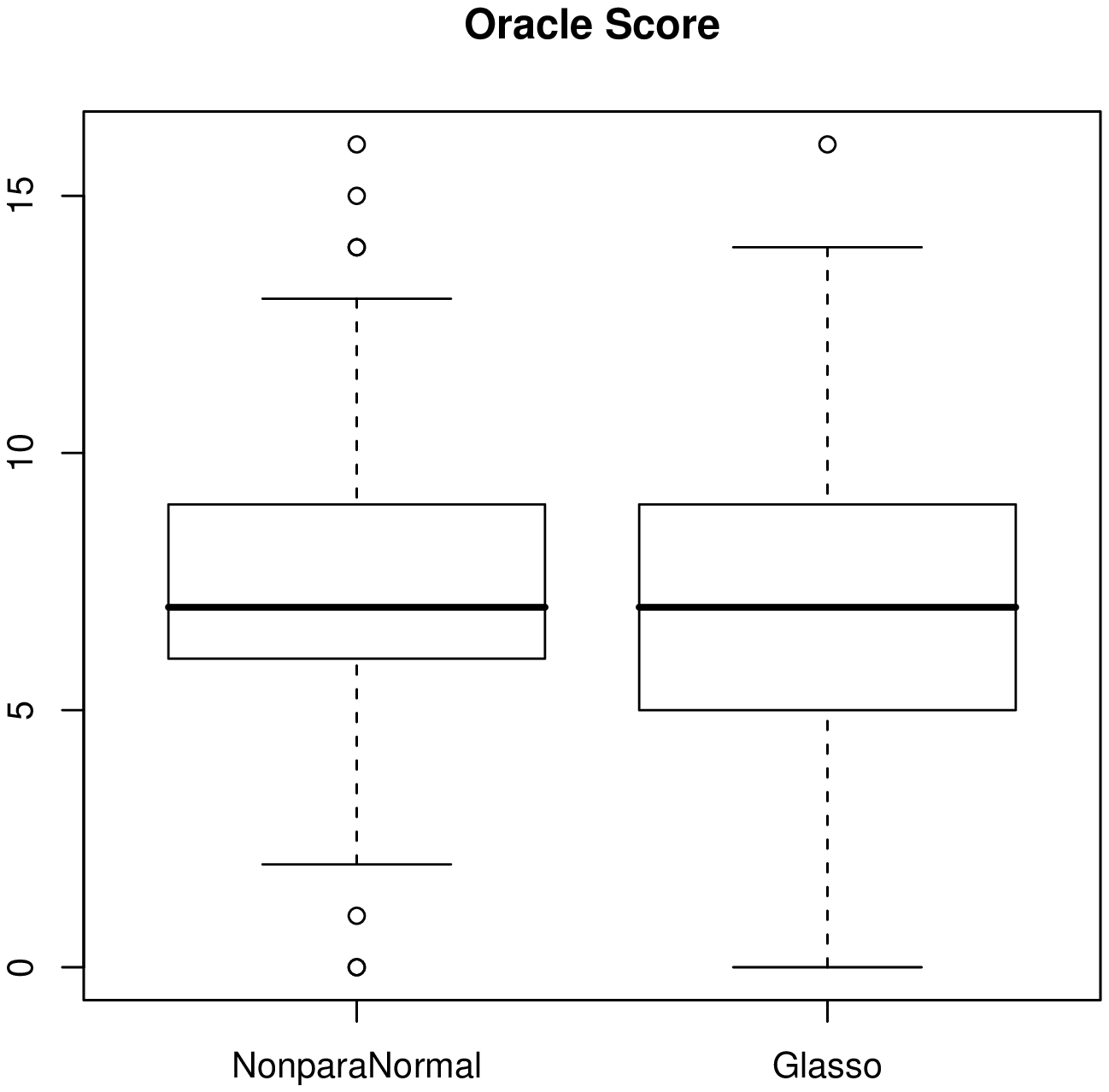}
\\[-15pt]
\end{tabular}
 \end{center}
\caption{\small Boxplots of the oracle scores for $n=1000, 500, 200$ (top, center,
  bottom).}
\vskip-5pt
\label{fig.Boxplots}
\end{figure}

\subsubsection{Quantitative comparison}

To evaluate the performance for structure estimation quantitatively,
we use false positive and false negative rates.
Let $G=(V,E)$ be a $p$-dimensional graph (which has at most
$\binom{p}{2}$ edges) in which there are $|E|=r$
edges, and let $\hat{G}^\lambda = (V, \hat{E}^\lambda)$ be an estimated graph using the
regularization parameter $\lambda$. The number of false positives at
$\lambda$ is 
\begin{eqnarray}
\mathrm{FP(\lambda)} \equiv \text{number of edges
in $\hat{E}^\lambda$ not in $E$}
\nonumber
\end{eqnarray}
The number of false negatives at
$\lambda$ is defined as
\begin{eqnarray}
\mathrm{FN(\lambda)} \equiv \text{number of edges in $E$
not in $\hat{E}^\lambda$}. \nonumber
\end{eqnarray}
The oracle regularization level $\lambda^*$ is then
\begin{eqnarray}
\lambda^* = \argmin_{\lambda \in \Lambda}\left\{ \mathrm{FP(\lambda)}
 + \mathrm{FN(\lambda)}
\right\}. \nonumber
\end{eqnarray}
The oracle score is $\mathrm{FP}(\lambda^*)+\mathrm{FN}(\lambda^*)$.
Figure~\ref{fig.Boxplots} shows boxplots of the oracle scores
for the two methods, calculated using 100 simulations.

To illustrate the overall performance of these two methods over the
full paths, ROC curves are shown in Figure \ref{fig.Rocs}, using
\begin{eqnarray}
\left(1-\frac{\mathrm{FN}(\lambda)}{r},
1-\frac{\mathrm{FP}(\lambda)}{\binom{p}{2}-r}\right).
\nonumber
\end{eqnarray}
The curves clearly show how the performance of both methods improves with
sample size, and that the nonparanormal is superior to the Gaussian
model in most cases.

\begin{figure}[htp]
\begin{center}
\begin{tabular}{ccc}
\\[-30pt]
\scriptsize \bf cdf & \scriptsize \bf power &\scriptsize \bf linear \\[-15pt]
\vspace{0in} 
\hskip-5pt
\includegraphics[width=.3\textwidth,angle=-90]{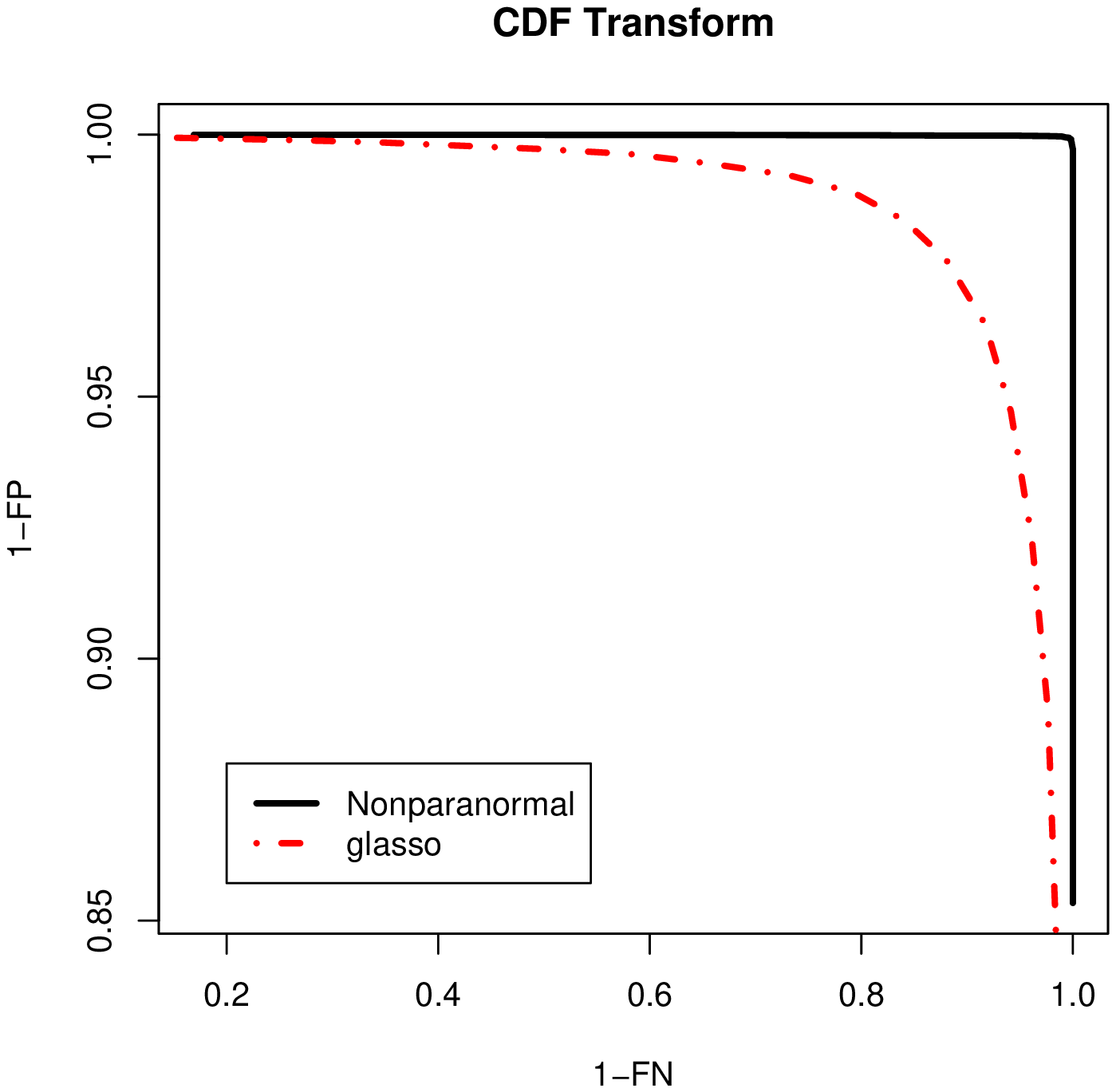} &
\hskip-5pt
\includegraphics[width=.3\textwidth,angle=-90]{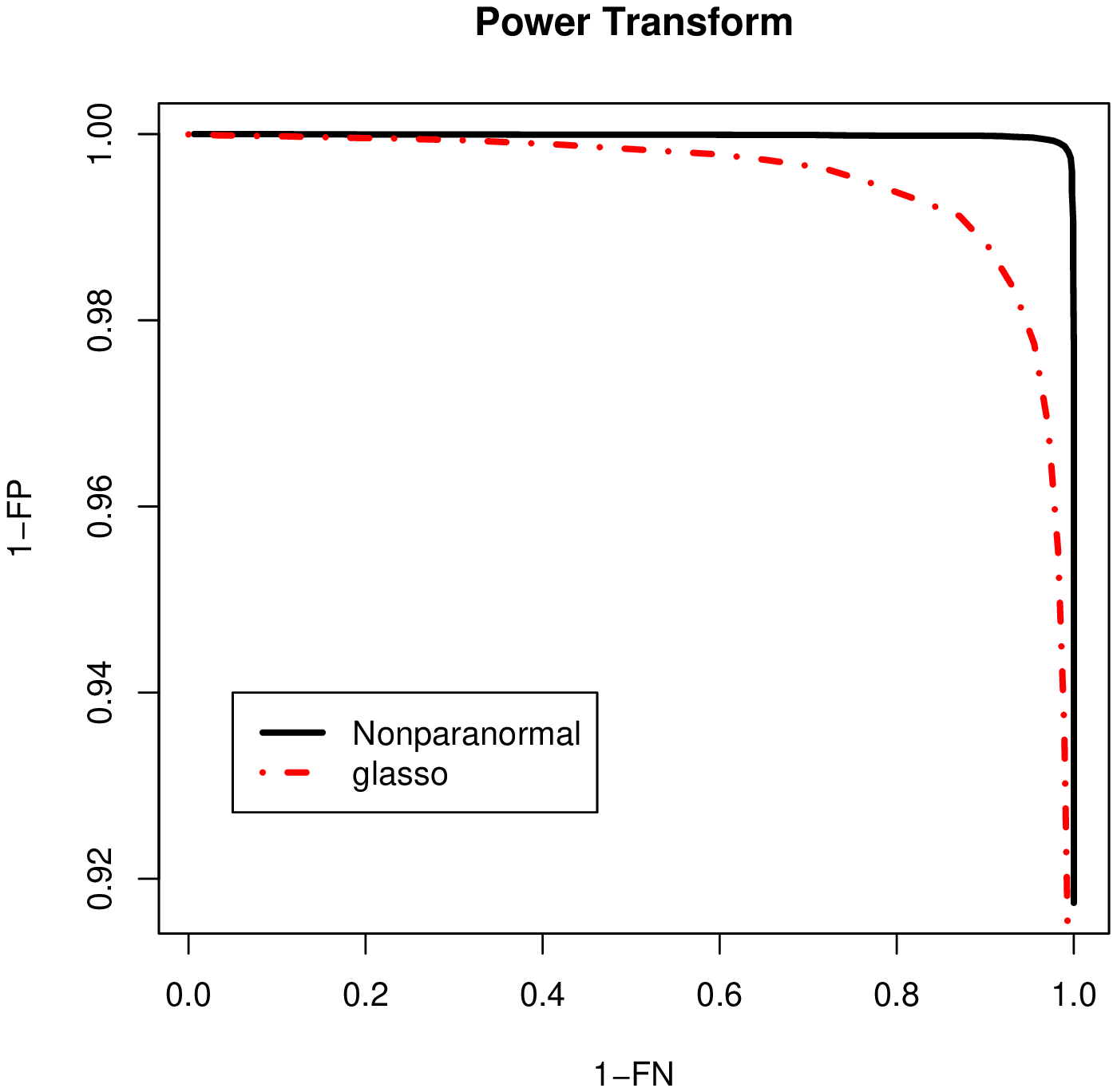}&
\hskip-5pt
\includegraphics[width=.3\textwidth,angle=-90]{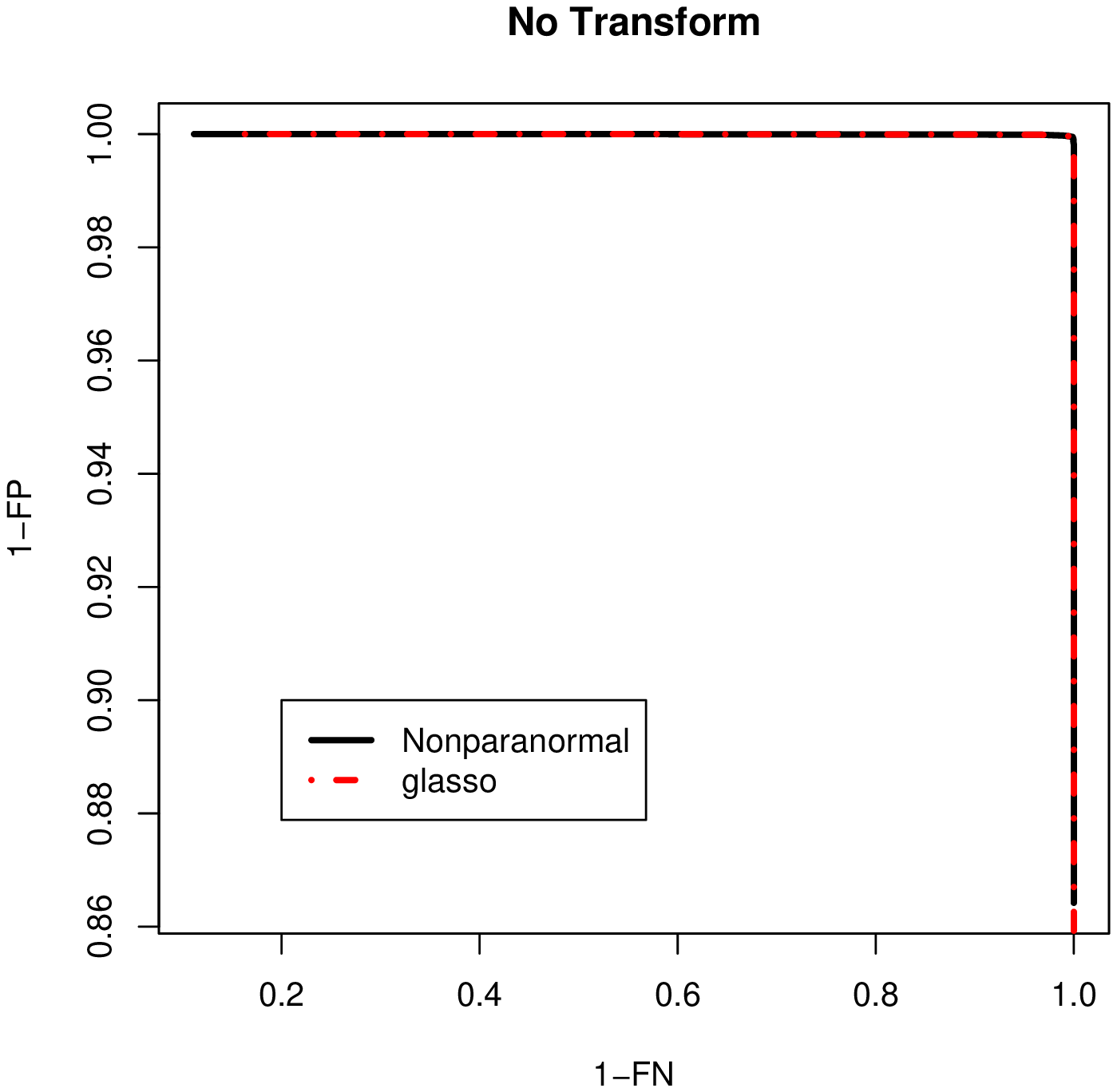} 
\\[-20pt]
\hskip-5pt
\includegraphics[width=.3\textwidth,angle=-90]{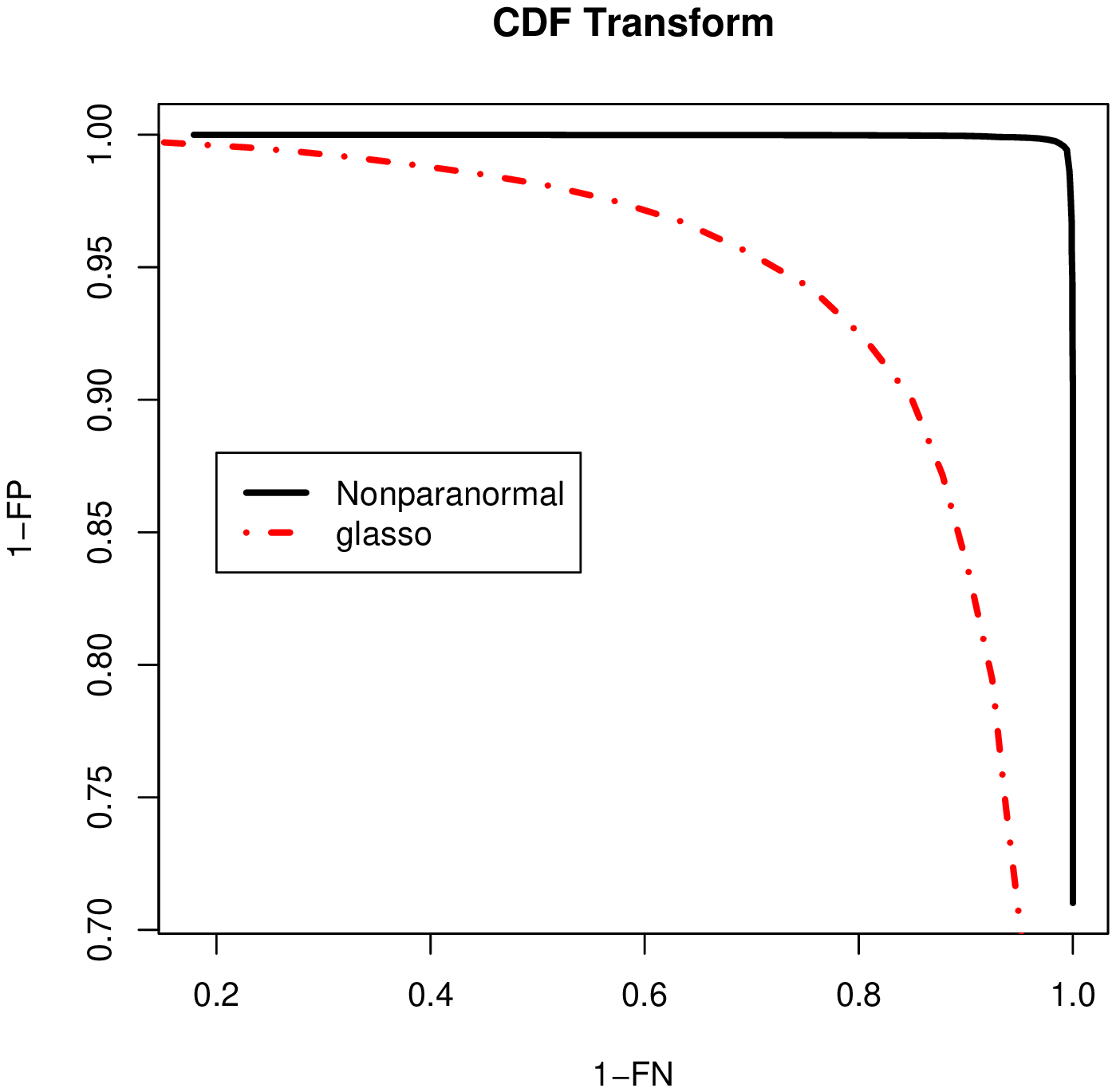} &
\hskip-5pt
\includegraphics[width=.3\textwidth,angle=-90]{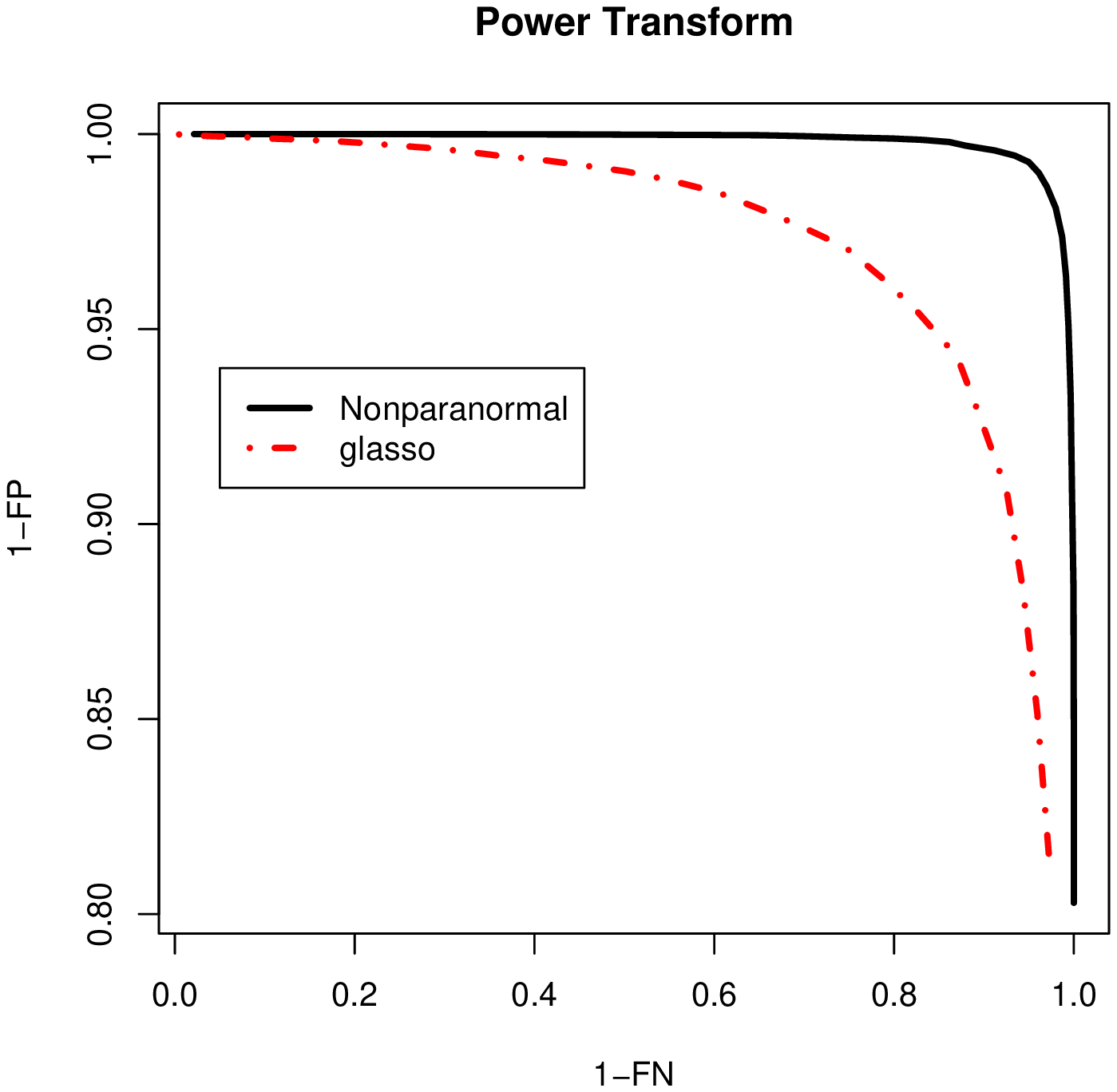}&
\hskip-5pt
\includegraphics[width=.3\textwidth,angle=-90]{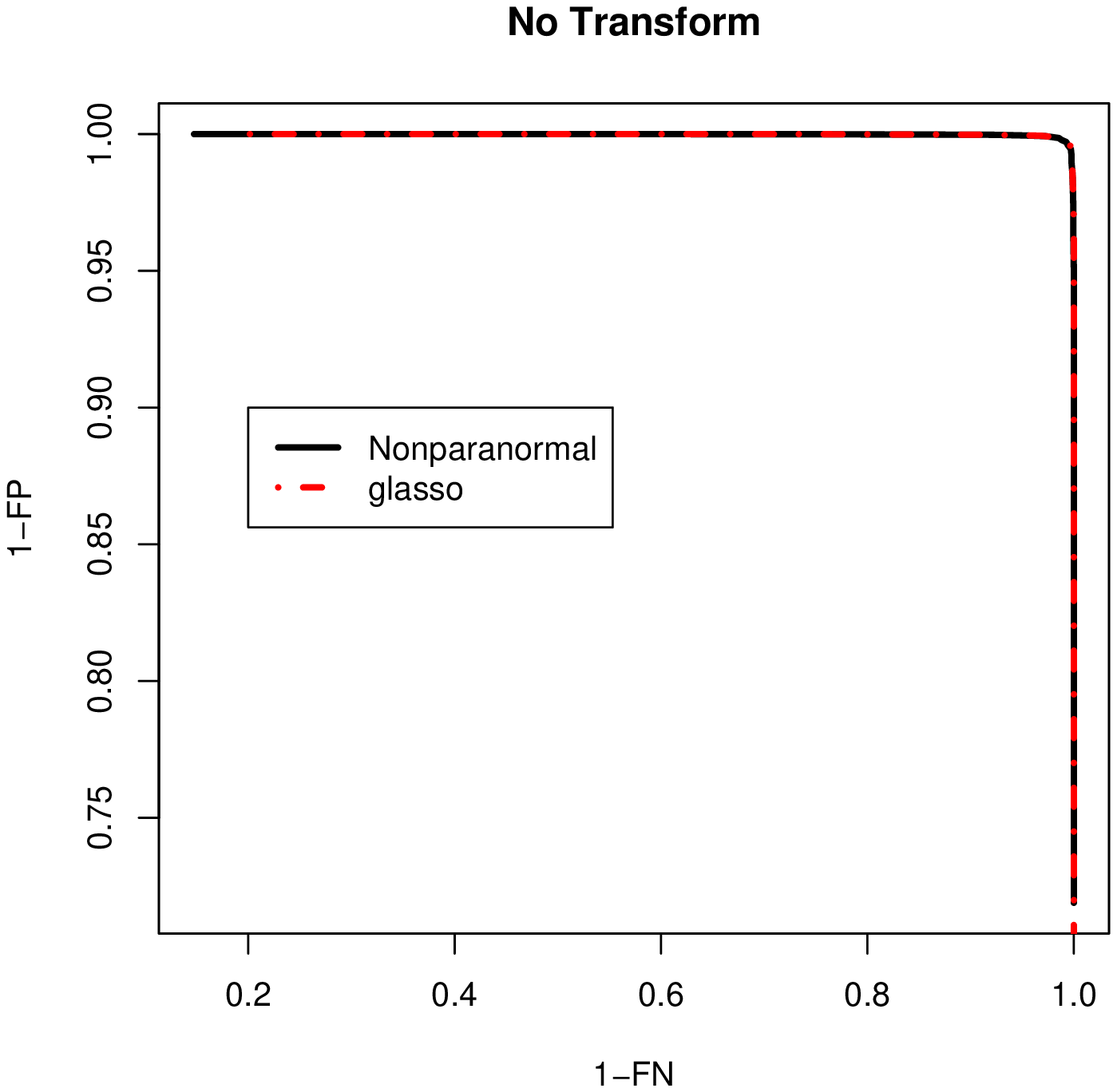}
\\[-20pt]
\hskip-5pt
\includegraphics[width=.3\textwidth,angle=-90]{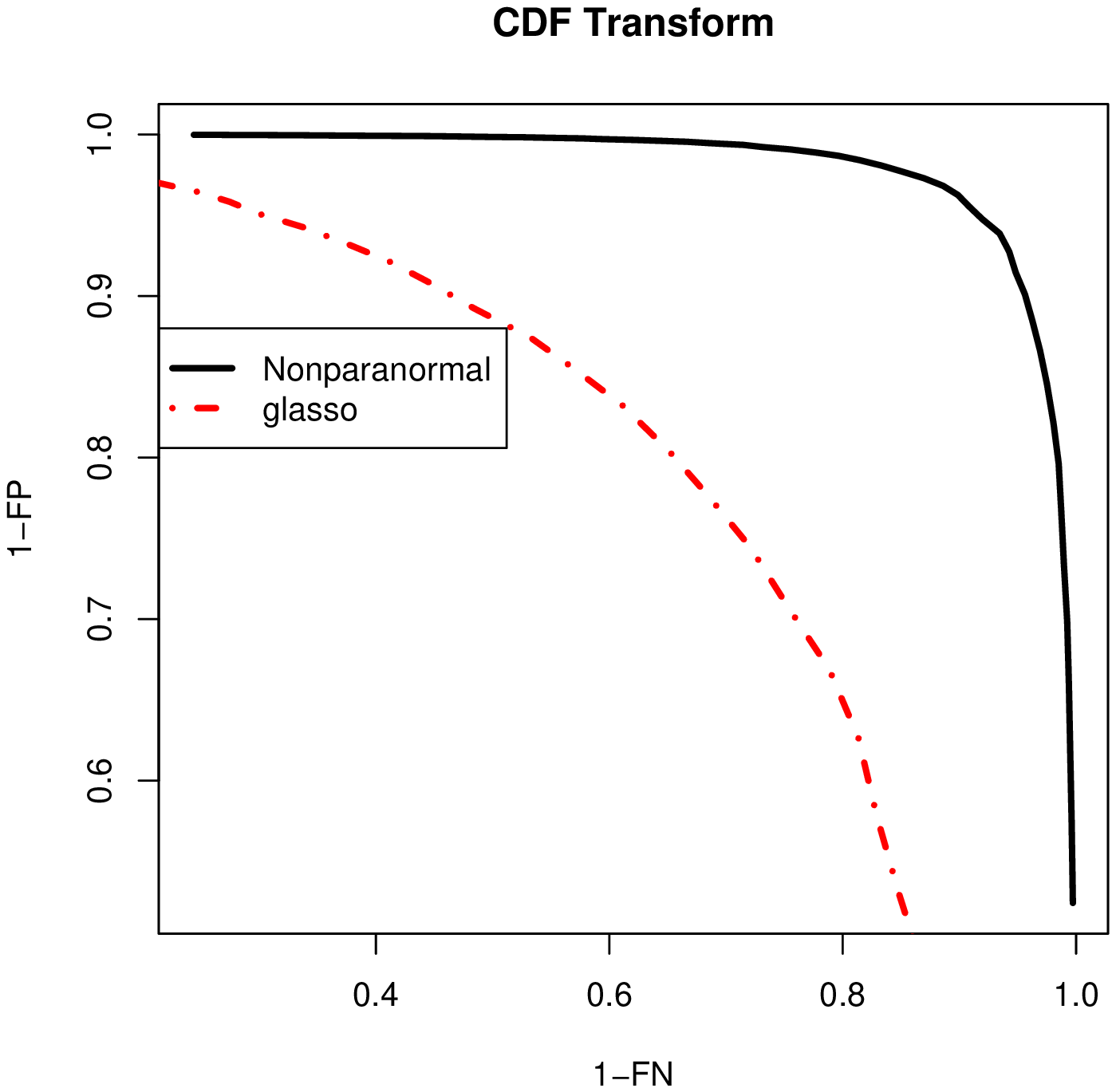} &
\hskip-5pt
\includegraphics[width=.3\textwidth,angle=-90]{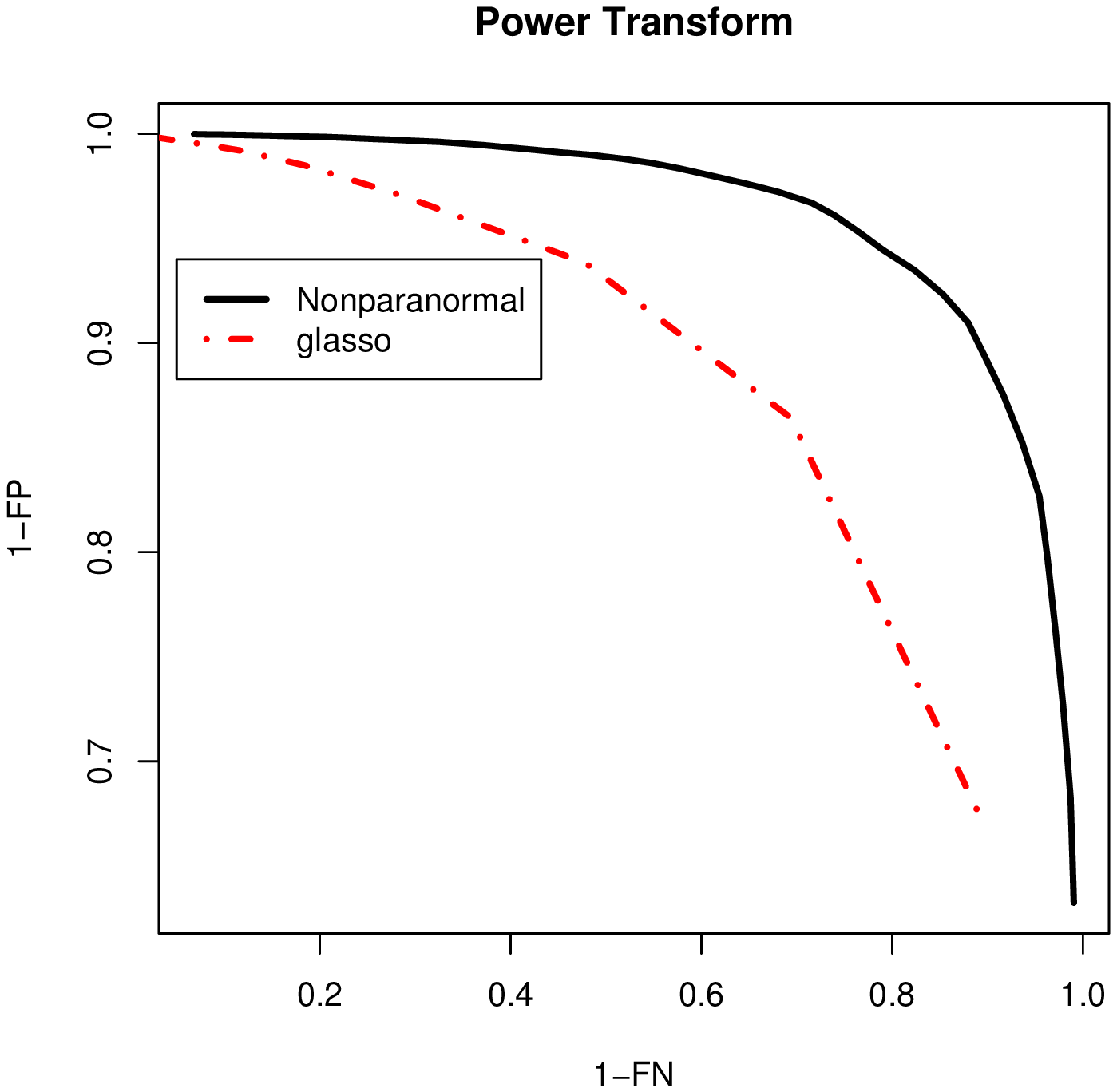}&
\hskip-5pt
\includegraphics[width=.3\textwidth,angle=-90]{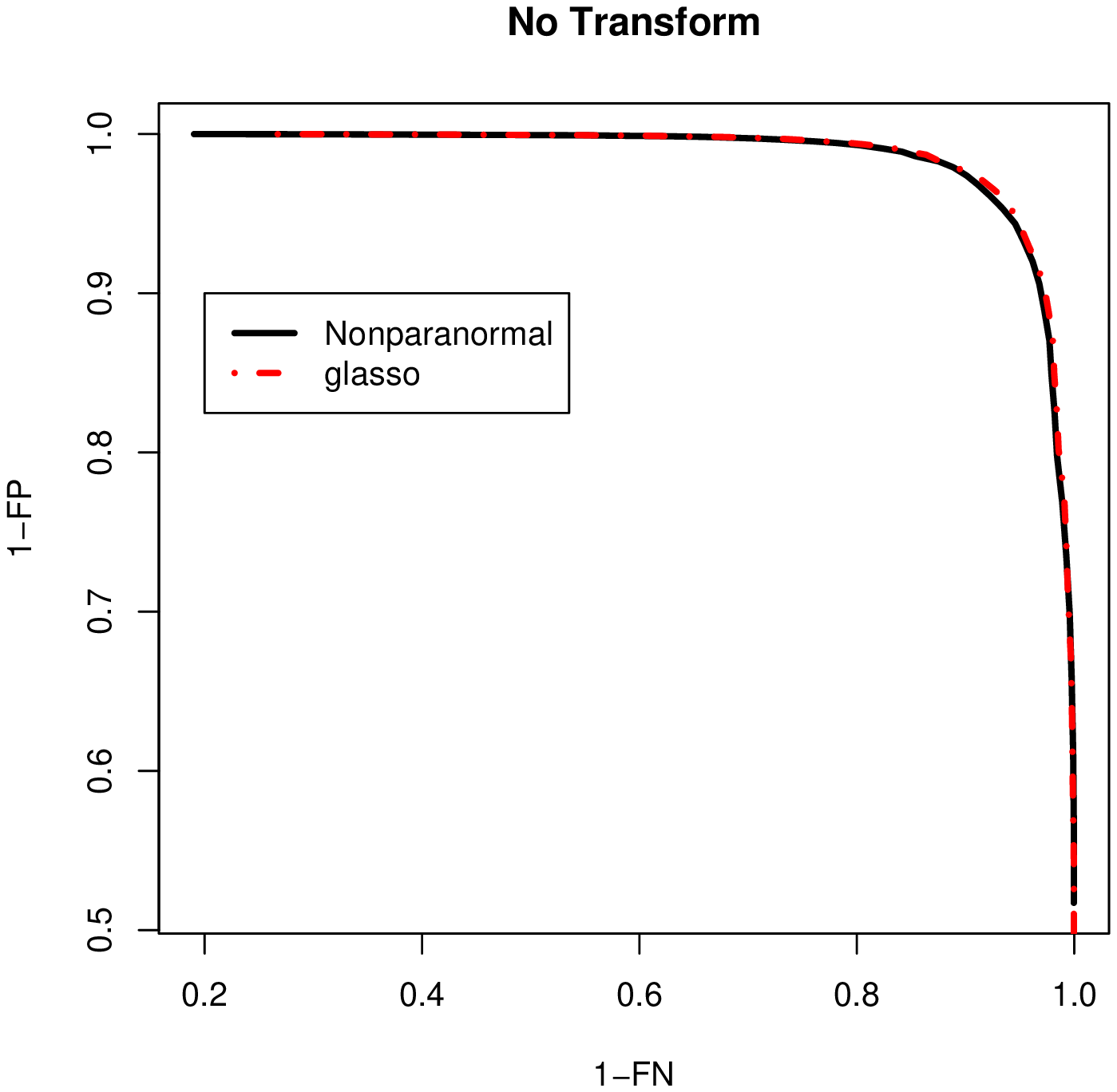}
\\[-20pt]
\end{tabular}
 \end{center}
\caption{\small ROC curves for sample sizes $n=1000, 500, 200$ (top,
  middle, bottom).}\label{fig.Rocs}
\end{figure}

Let   $\mathrm{FPE} \equiv \mathrm{FP}(\lambda^{*})$ and  $\mathrm{FNE} \equiv \mathrm{FN}(\lambda^{*})$,  Tables \ref{tab.CDFFPE}, \ref{tab.PowerFPE}, and \ref{tab.GaussFPE}
provide numerical comparisons of both methods on datasets
with different transformations, where we repeat the experiments 100
times and report the average {\rm FPE} and $\mathrm{FNE}$
values with the corresponding standard deviations. It's clear from
the tables that the nonparanormal achieves significantly smaller
errors than the glasso if the true distribution of the data is not multivariate
Gaussian and achieves comparable performance as the glasso when the
true distribution is exactly multivariate Gaussian.

 \newcolumntype {Q}{>{$\displaystyle}l<{$}}
 \newcolumntype {A}{>{$}c <{$}}
 \begin{table}[htp]
 \caption{Quantitative comparison on the dataset using the cdf
 transformation}\label{tab.CDFFPE}
 \begin{center}
{\sf \small
 \begin{tabular}{QAAAAAAAA}
 \toprule %
 \multicolumn{8}{c}{ Nonparanormal ~~~~~~} {glasso}  \\
 \cmidrule(r){2-5}\cmidrule(r){6-9}
      n       & {\rm FPE} &  {\rm (sd(FPE))} & {\rm FNE} & {\rm (sd(FNE))}  & {\rm FPE} &  {\rm (sd(FPE))} & {\rm FNE} & {\rm (sd(FNE))}  \\
 \midrule
 1000 & 0.10 & (0.3333) & 0.05 & (0.2190) & 3.73 & (2.3904) & 7.24 & (3.2910)\\
 \midrule
  900 & 0.18 & (0.5389) & 0.16 & (0.4197) & 3.31 & (2.4358) & 8.94 & (3.2808)\\
 \midrule
  800 & 0.16 & (0.5069) & 0.23 & (0.5659) & 3.80 & (2.9439) & 9.91 & (3.4789)\\
  \midrule
  700 & 0.26 & (0.6295) & 0.43 & (0.7420) & 3.45 & (2.5519) & 12.26 & (3.5862)\\
  \midrule
  600 & 0.33 & (0.6039) & 0.41 & (0.6371) & 3.31 & (2.8804) & 14.25 & ( 4.0735)\\
  \midrule
  500 & 0.58 & (0.9658) & 1.10 & (1.0396) & 3.18 & (2.9211) & 17.54 & (4.4368)\\
  \midrule
  400 & 0.71 & (1.0569) & 1.52 & (1.2016) & 1.58 & (2.3535) & 21.18 & (4.9855)\\
  \midrule
  300 & 1.37 & (1.4470) & 2.97 & (2.0123) & 0.67 & (1.6940) & 23.14 & (5.0232)\\
   \midrule
  200 & 2.03 & (1.9356) & 7.13 & (3.4514) & 0.01 & (0.1000) & 24.03 & ( 4.9816)\\
\bottomrule
 \end{tabular}\label{table.coef} }
 \end{center}
 \end{table}

 \begin{table}[htp]
 \caption{Quantitative comparison on the dataset using the power
 transformation}\label{tab.PowerFPE}
 \begin{center}
{\sf \small
 \begin{tabular}{QAAAAAAAA}
 \toprule %
 \multicolumn{8}{c}{ Nonparanormal ~~~~~~} {glasso}  \\
 \cmidrule(r){2-5}\cmidrule(r){6-9}
      n       & {\rm FPE} &  {\rm (sd(FPE))} & {\rm FNE} & {\rm (sd(FNE))}  & {\rm FPE} &  {\rm (sd(FPE))} & {\rm FNE} & {\rm (sd(FNE))}  \\
 \midrule
 1000 & 0.27 & (0.7086) & 0.35 & (0.6571) & 2.89 & (1.9482) & 4.97 & (2.9213)\\
 \midrule
  900 & 0.38 & (0.6783) & 0.41 & (0.6210) & 2.98 & (2.3697) & 5.99 & (3.0467)\\
 \midrule
  800 & 0.25 & (0.5751) & 0.73 & (0.8270) & 4.10 & (2.7834) & 6.39 & (3.3571)\\
  \midrule
  700 & 0.69 & (0.9067) & 0.90 & (1.0200) & 4.42 & (2.8891) & 8.80 & (3.9848)\\
  \midrule
  600 & 0.92 & (1.2282) & 1.59 & (1.5314) & 4.64 & (3.3830) & 10.58 & (4.2168)\\
  \midrule
  500 & 1.17 & (1.3413) & 2.56 & (2.3325) & 4.00 & (2.9644) & 13.09 & (4.4903)\\
  \midrule
  400 & 1.88 & (1.6470) & 4.97 & (2.7687) & 3.14 & (3.4699) & 17.87 & (4.7750)\\
  \midrule
  300 & 2.97 & (2.4181) & 7.85 & (3.5572) & 1.36 & (2.3805) & 21.24 & (4.7505)\\
   \midrule
  200 & 2.82 & (2.6184) & 14.53 & ( 4.3378) & 0.37 & (0.9914) & 24.01& (5.0940)\\
\bottomrule
 \end{tabular}\label{table.coef2} }
 \end{center}
 \end{table}

 \begin{table}[htp]
 \caption{Quantitative comparison on the dataset without any
 transformation}\label{tab.GaussFPE}
 \begin{center}
{\sf \small
 \begin{tabular}{QAAAAAAAA}
 \toprule %
 \multicolumn{8}{c}{ Nonparanormal ~~~~~~} {glasso}  \\
 \cmidrule(r){2-5}\cmidrule(r){6-9}
      n       & {\rm FPE} &  {\rm (sd(FPE))} & {\rm FNE} & {\rm (sd(FNE))}  & {\rm FPE} &  {\rm (sd(FPE))} & {\rm FNE} & {\rm (sd(FNE))}  \\
 \midrule
 1000 & 0.10 & (0.3333) & 0.05 & (0.2190) & 0.09 & (0.3208) & 0.06 & (0.2386)\\
 \midrule
  900 & 0.24 & (0.7537) & 0.14 & (0.4025) & 0.22 & (0.6447) & 0.15 & (0.4113)\\
 \midrule
  800 & 0.17 & (0.4277) & 0.16 & (0.3949) & 0.16 & (0.4431) & 0.19 & (0.4191)\\
  \midrule
  700 & 0.25 & (0.6871) & 0.33 & (0.8534) & 0.29 & (0.8201) & 0.27 & (0.7501)\\
  \midrule
  600 & 0.37 & (0.7740) & 0.36 & (0.7456) & 0.36 & (0.7722) & 0.37 & (0.6459)\\
  \midrule
  500 & 0.28 & (0.5874) & 0.46 & (0.7442) & 0.25 & (0.5573) & 0.45 & (0.6571)\\
  \midrule
  400 & 0.55 & (0.8453) & 1.37 & (1.2605) & 0.47 & (0.7713) & 1.35 & (1.2502)\\
  \midrule
  300 & 1.24 & (1.3715) & 3.07 & (1.7306) & 0.98 & (1.2058) & 3.04 & (1.8905)\\
   \midrule
  200 & 1.62 & (1.7219) & 5.89 & (2.7373) & 1.55 & (1.6779) & 5.62 & (2.6620)\\
\bottomrule
 \end{tabular}\label{table.coef3} }
 \end{center}
 \end{table}

\subsubsection{Visualization of typical runs}

Figure \ref{fig.CDFtypicalRun}
shows typical runs for the cdf and power transformations.
It's clear that when the glasso estimates the graph incorrectly, the mistakes
include both false positives and~negatives.

\begin{figure}[htp]
\vspace{-0.1in}
\begin{center}
\begin{tabular}{c|c}
\scriptsize \bf cdf & \scriptsize \bf power \\[-5pt]
\vspace{0in} 
\hskip-25pt
\includegraphics[width=3.5in,height=3.3in,angle=-90]{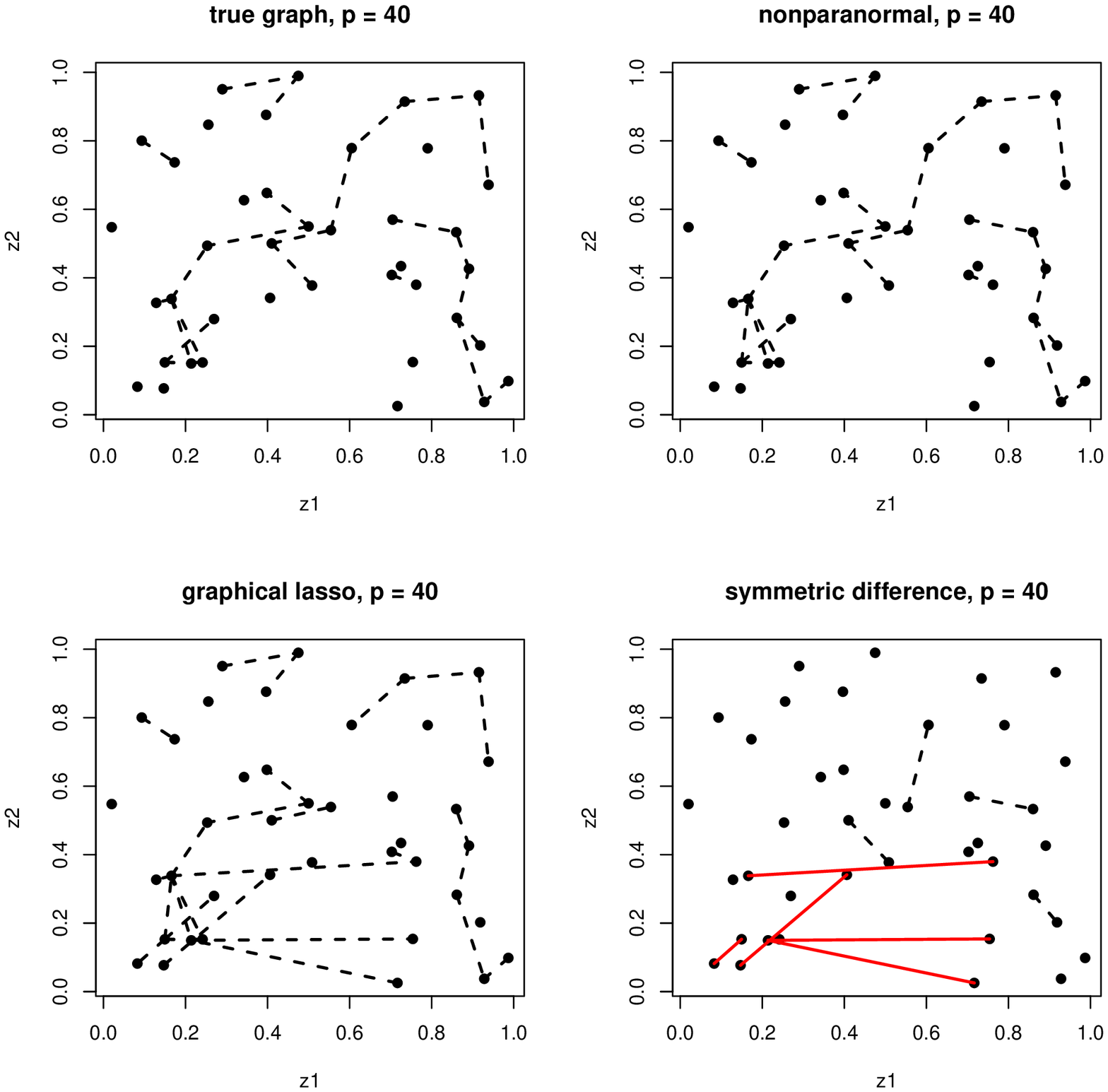} &\vspace{-0.02in}
\hspace{-0.20in}
\includegraphics[width=3.5in,height=3.3in,angle=-90]{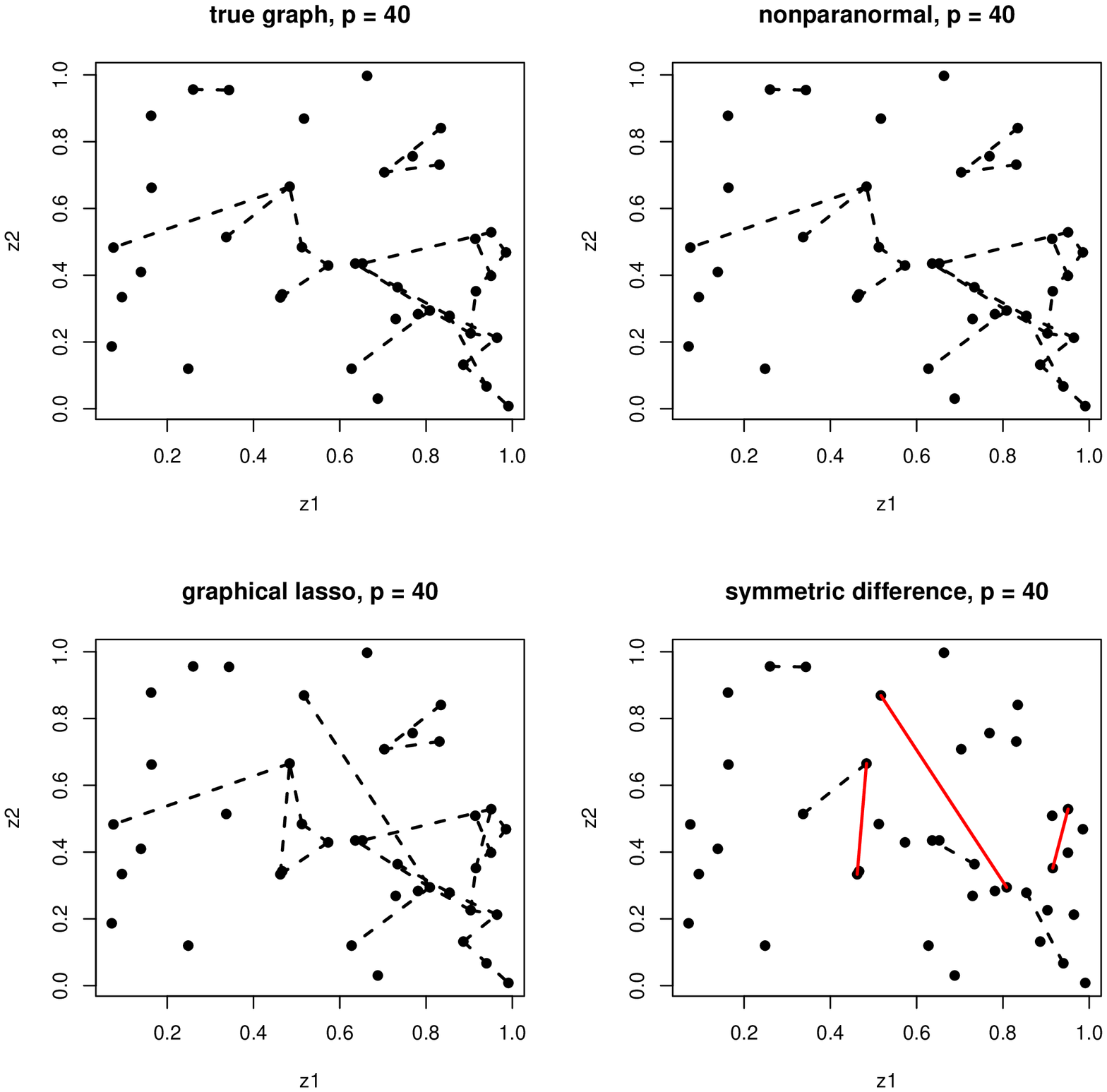}
\vspace{-0.0in}
\\
\hskip-20pt
\includegraphics[width=3.5in,height=3.3in,angle=-90]{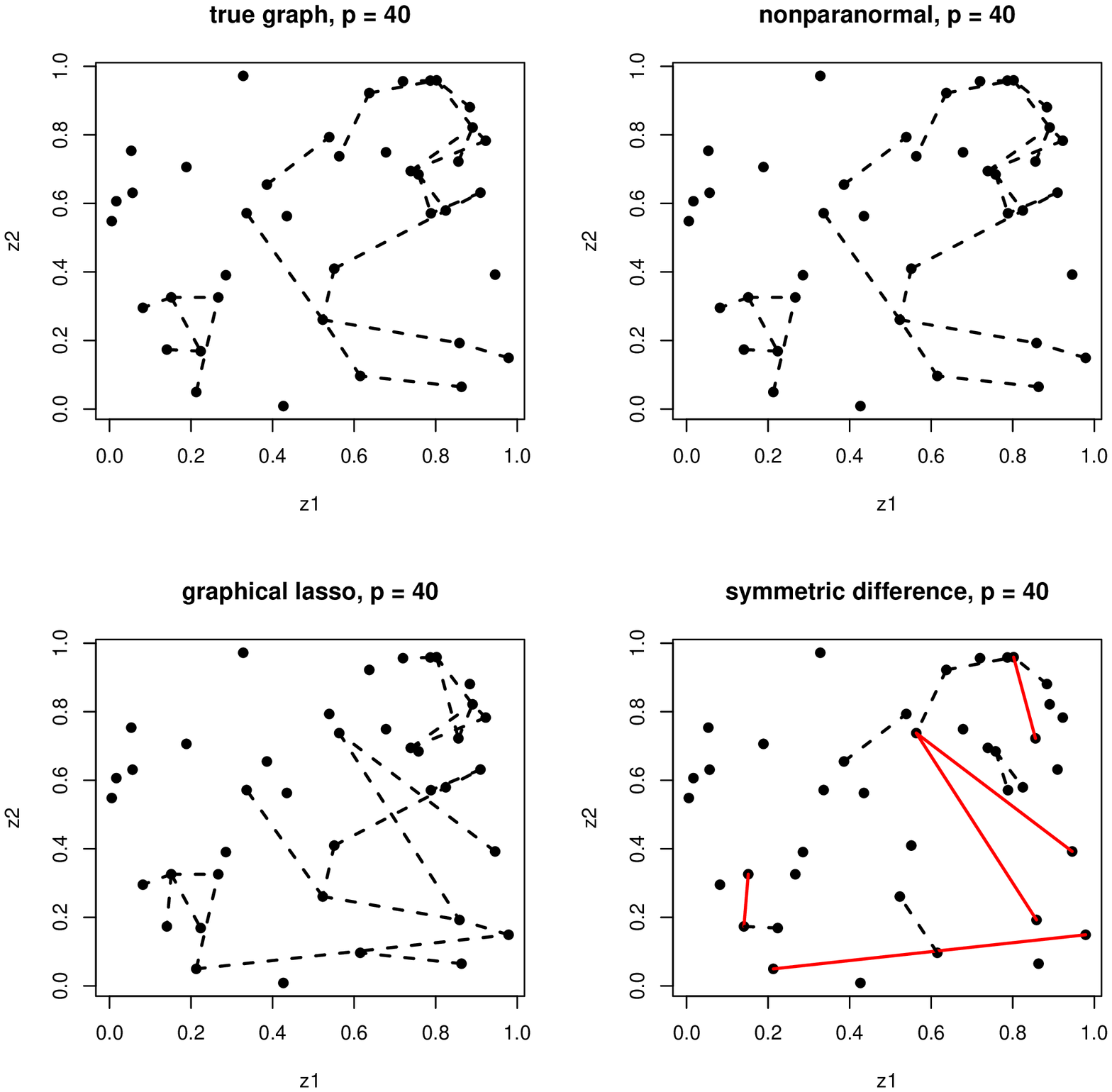} &\vspace{-0.0in}
\hspace{-0.15in}
\includegraphics[width=3.5in,height=3.3in,angle=-90]{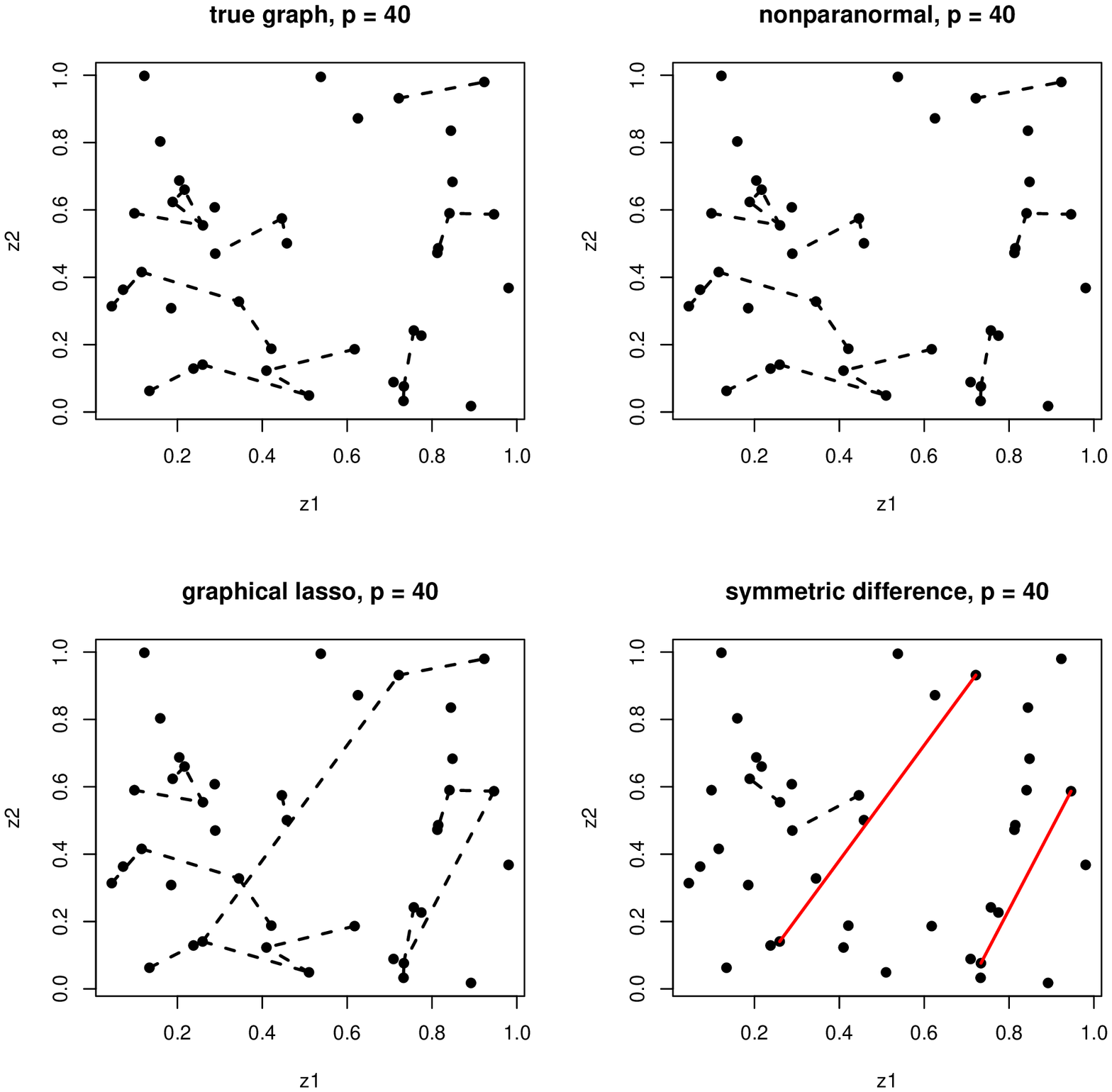}
\vspace{-0.0in}
\end{tabular}
\end{center}
\caption{\small Typical runs for the two methods for $n=1000$ using 
the cdf and power transformations. The dashed (black) lines in the
symmetric difference plots indicate edges found by the glasso but
not the nonparanormal, and vice-versa for the solid (red) lines.}
\label{fig.CDFtypicalRun}
\end{figure}

\subsection{Gene microarray data}

In this study, we consider a dataset based on
Affymetrix GeneChip microarrays 
for the plant \textit{Arabidopsis thaliana},
\citep{wille:04}.  The sample size is $n=118$.  
The expression levels for each chip are pre-processed by log-transformation
and standardization.   A subset of 40 genes from the isoprenoid
pathway are chosen, and we study the associations among them using
both the paranormal and nonparanormal models. Even though these data
are generally treated as multivariate Gaussian in the
previous analysis \citep{wille:04}, our study shows that the results
of the nonparanormal and the glasso are very different over a 
wide range of regularization parameters. This suggests the
nonparanormal could support different scientific conclusions.

\subsubsection{Comparison of the regularization paths}

We first compare the regularization paths of the two methods, in Figure~\ref{fig.GeneFullPath}.
To generate the paths, we select  50 regularization
parameters on an evenly spaced grid in the interval $[0.16, 1.2]$.
Although the paths for the two methods look similar, there are
some subtle differences.  In particular, variables become nonzero in a different order,
especially when the regularization parameter is in the range $\lambda
\in [0.2, 0.3]$.  As shown below, these subtle differences in the paths lead to
different model selection behaviors.

\begin{figure}[htp]
\vspace{-0.1in}
\begin{center}
\begin{tabular}{cc}
\hskip-10pt
\includegraphics[width=.5\textwidth,angle=-90]{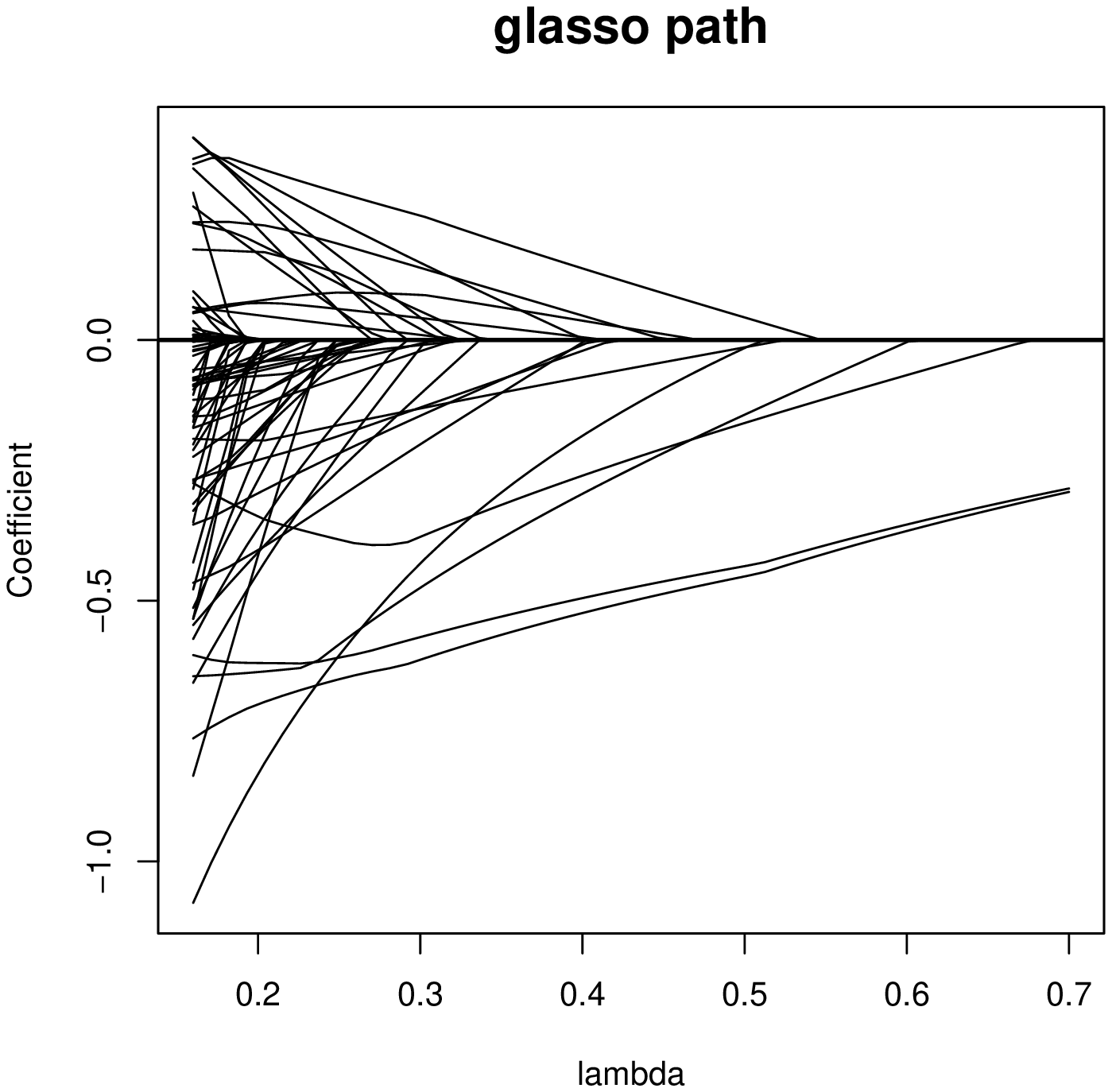}
&
\hskip-10pt
\includegraphics[width=.5\textwidth,angle=-90]{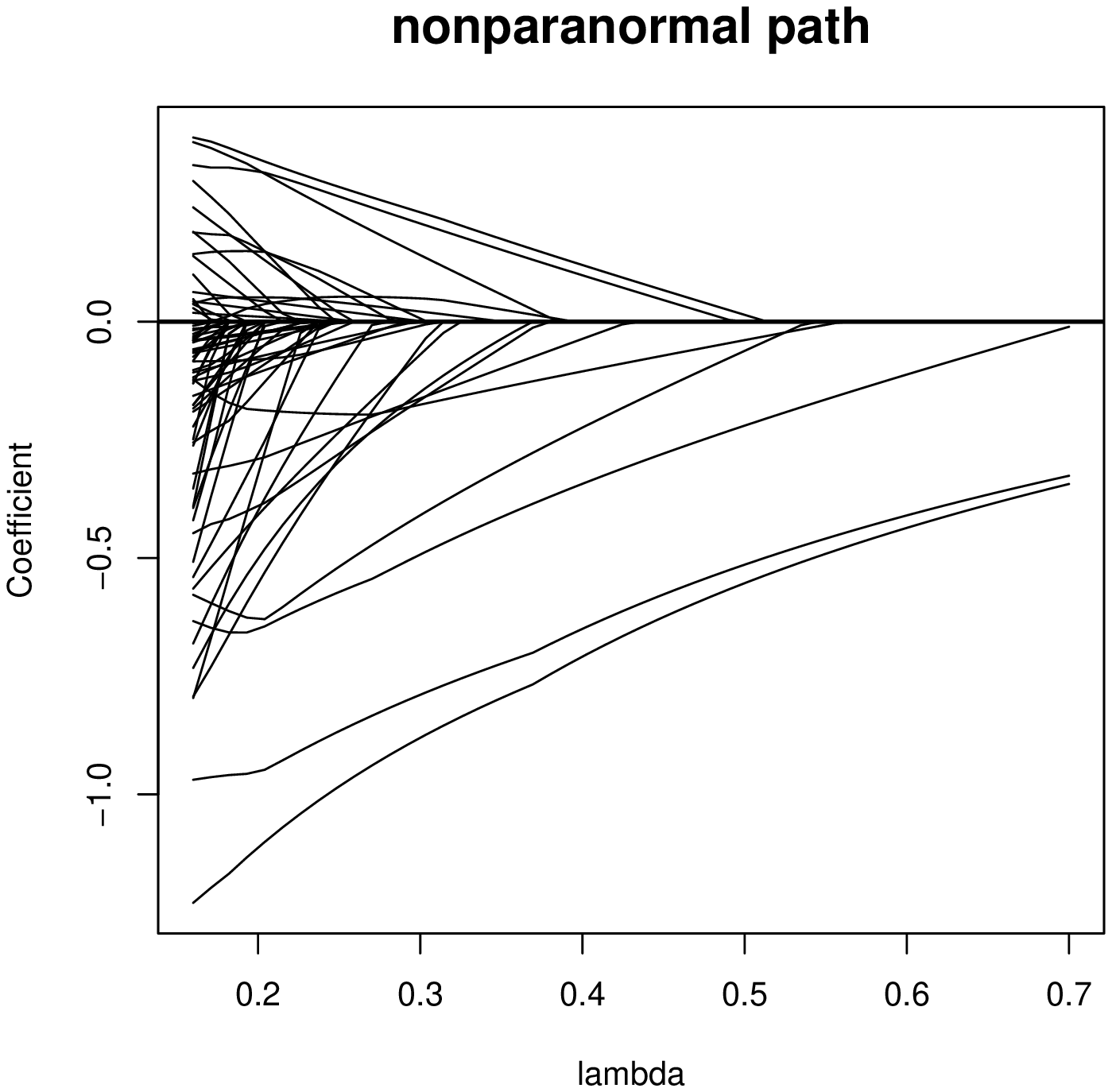}
\\[-30pt]
\end{tabular}
 \end{center}
\caption{\small The regularization paths of both methods on the
microarray dataset.}\label{fig.GeneFullPath}
\end{figure}

\subsubsection{Comparison of the selected graphs}

Figure \ref{fig.GeneSelectedGraph1} compares the estimated graphs for
the two methods at several values of the regularization parameter
$\lambda$ in the range $[0.16, 0.37]$.  For each $\lambda$, we show
the estimated graph from the nonparanormal in the first column. In the
second column we show the graph obtained by scanning the full
regularization path of the glasso fit and finding the graph having the
smallest symmetric difference with the nonparanormal graph.  The
symmetric difference graph is shown in in the third column.  The
closest glasso fit is different, with edges selected by the glasso not
selected by the nonparanormal, and vice-versa.  Several estimated
transformations are plotted in Figure \ref{fig.Genecomponent1}, which are
are nonlinear.  Interestingly, several of the differences
between the fitted graphs are related to these variables.

\begin{figure}[htp]
\vspace{-0.1in}
\begin{center}
\begin{tabular}{ccc}
\\[-50pt]
\hspace{-30pt}
\includegraphics[width=.37\textwidth,angle=-90]{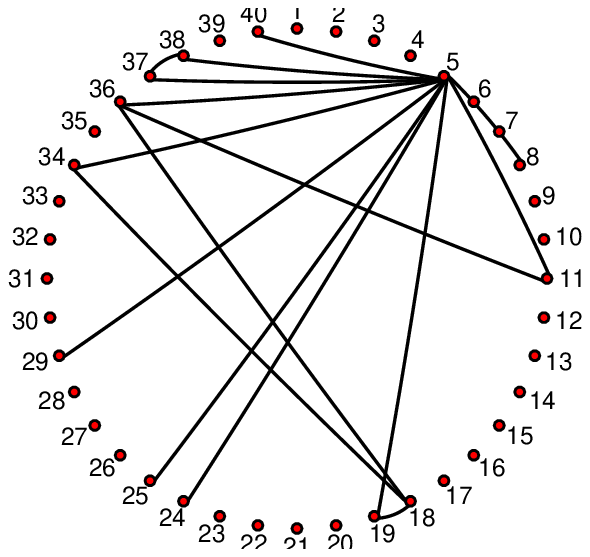} &
\hspace{-0.58in}
\includegraphics[width=.37\textwidth,angle=-90]{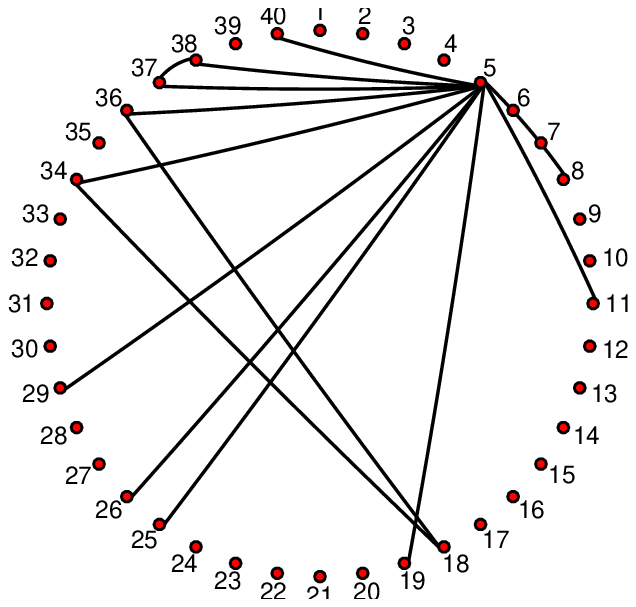} &
\hspace{-0.58in}
\includegraphics[width=.37\textwidth,angle=-90]{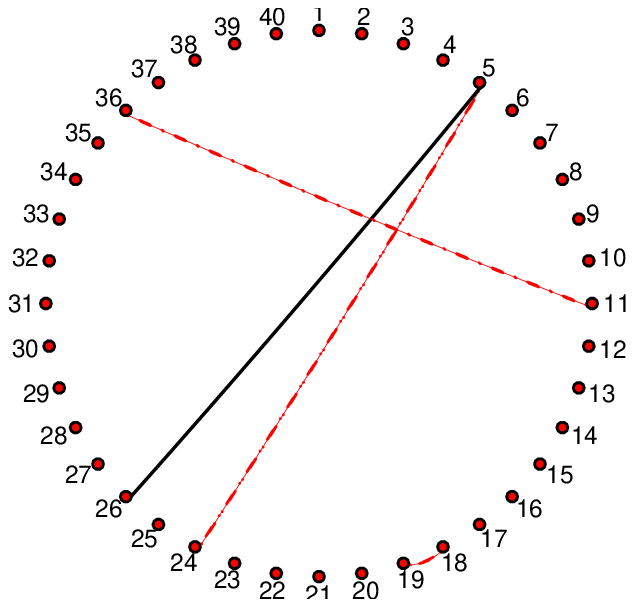} 
\\[-70pt]
\hspace{-30pt}
\includegraphics[width=.37\textwidth,angle=-90]{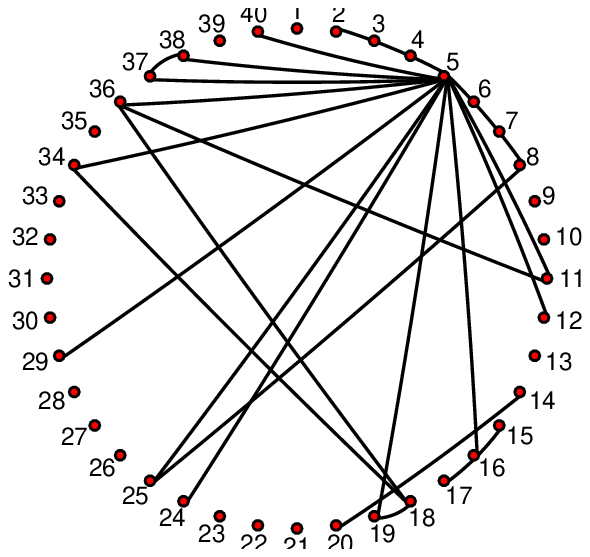} &
\hspace{-0.58in}
\includegraphics[width=.37\textwidth,angle=-90]{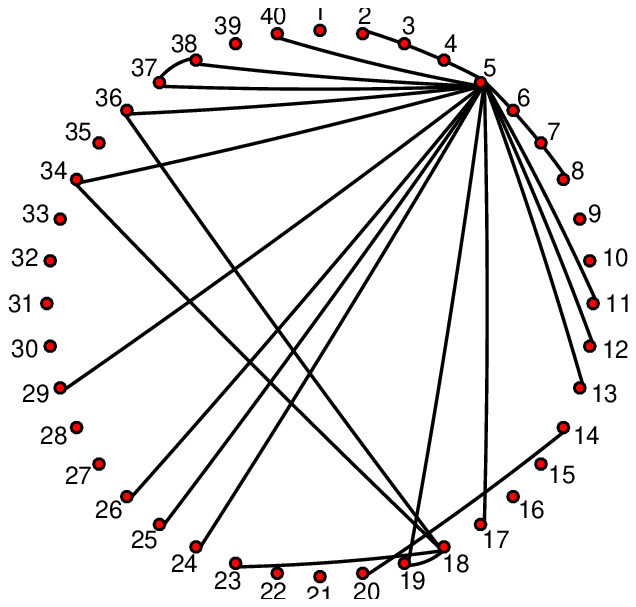} &
\hspace{-0.58in}
\includegraphics[width=.37\textwidth,angle=-90]{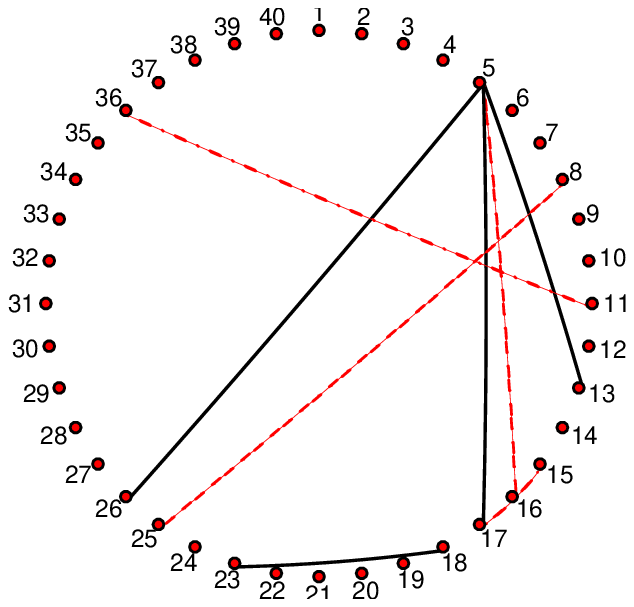} 
\\[-70pt]
\hspace{-30pt}
\includegraphics[width=.37\textwidth,angle=-90]{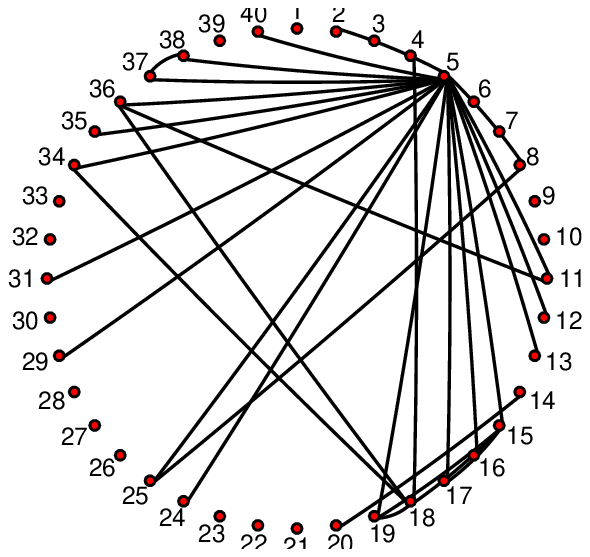} &
\hspace{-0.58in}
\includegraphics[width=.37\textwidth,angle=-90]{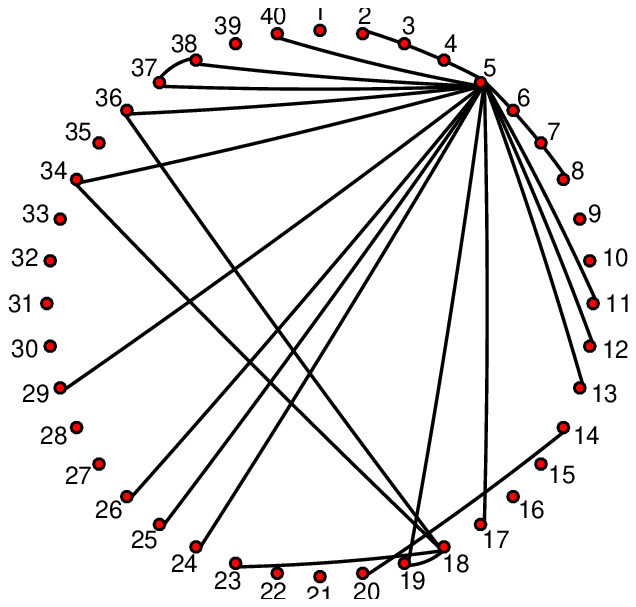} &
\hspace{-0.58in}
\includegraphics[width=.37\textwidth,angle=-90]{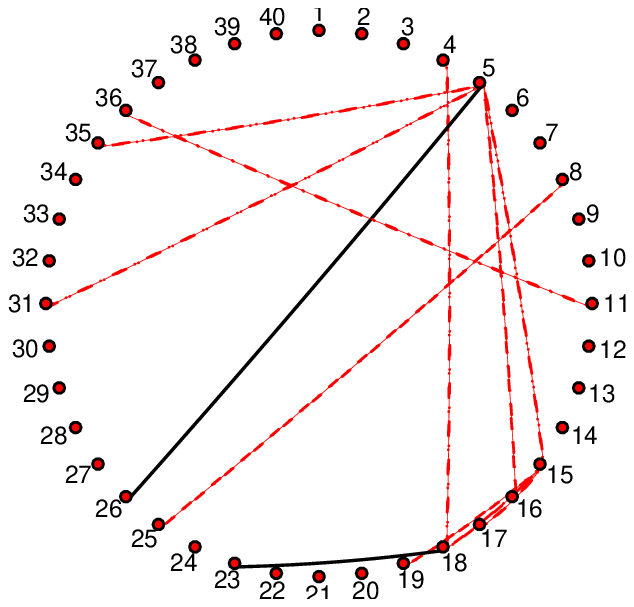} 
\\[-30pt]
\end{tabular}
\end{center}
\caption{\small The nonparanormal estimated graph for three  values of $\lambda =0.2448, 0.2661,  0.30857$ (left column), the glasso
  estimated graph (middle) and the symmetric difference graph (right).}
\label{fig.GeneSelectedGraph1}
\begin{center}
\begin{tabular}{cccc}
\\[-30pt]
\hskip-13pt
\includegraphics[width=.26\textwidth,angle=-90]{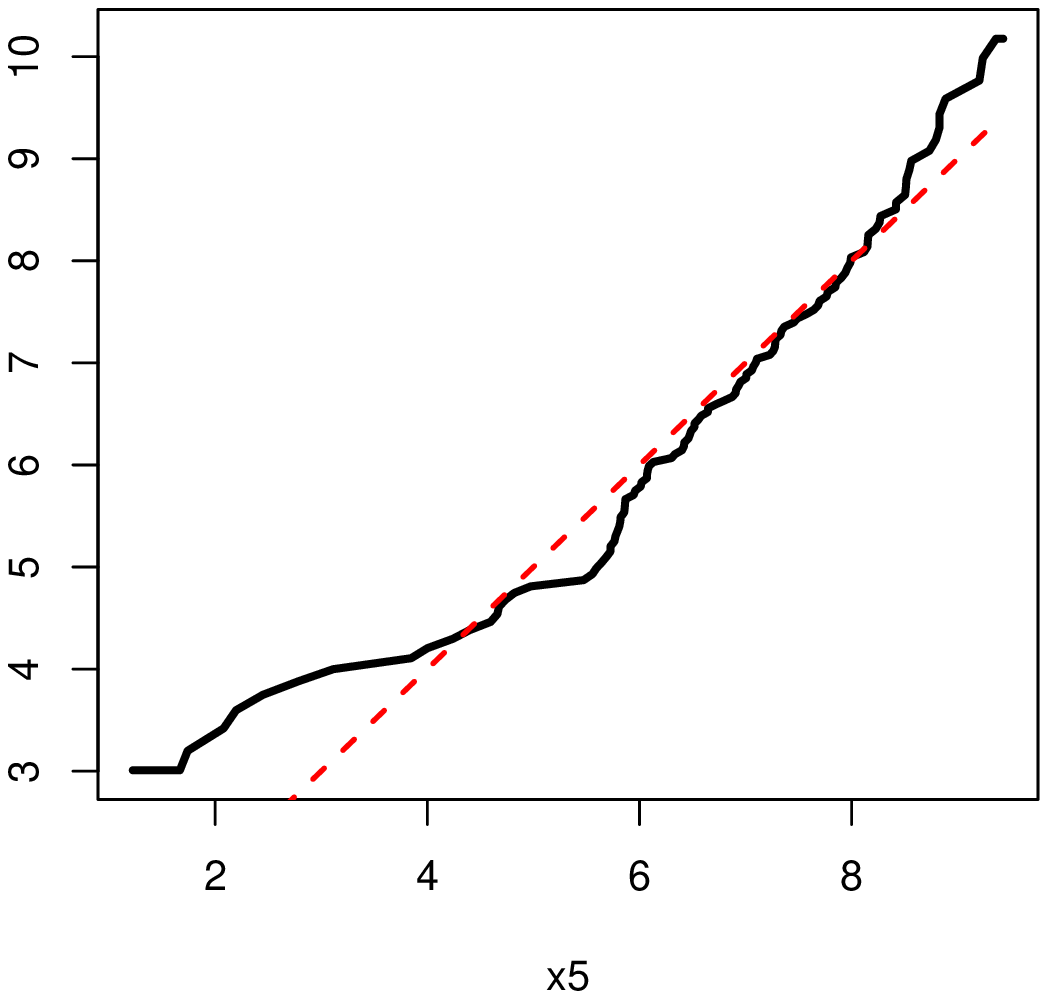} &
\hskip-13pt
\includegraphics[width=.26\textwidth,angle=-90]{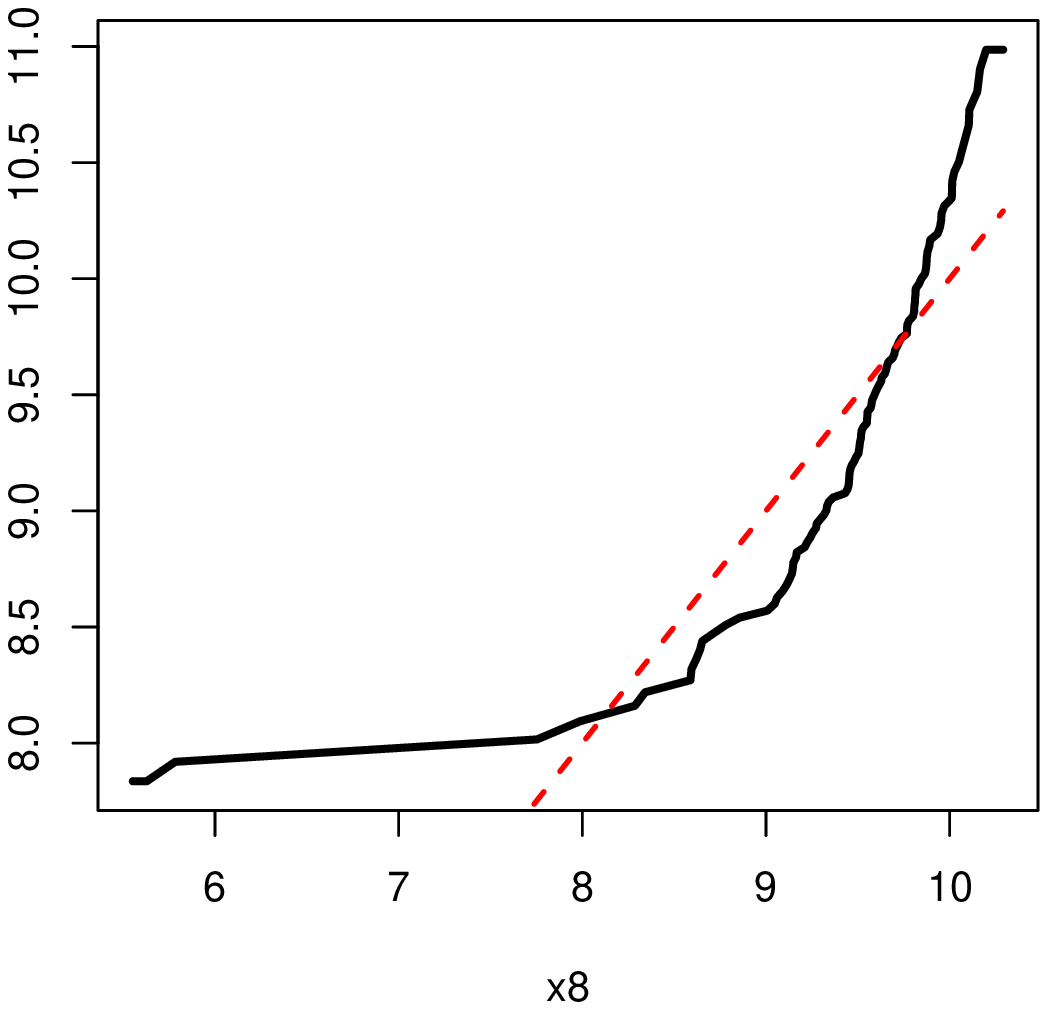} &
\hskip-13pt
\includegraphics[width=.26\textwidth,angle=-90]{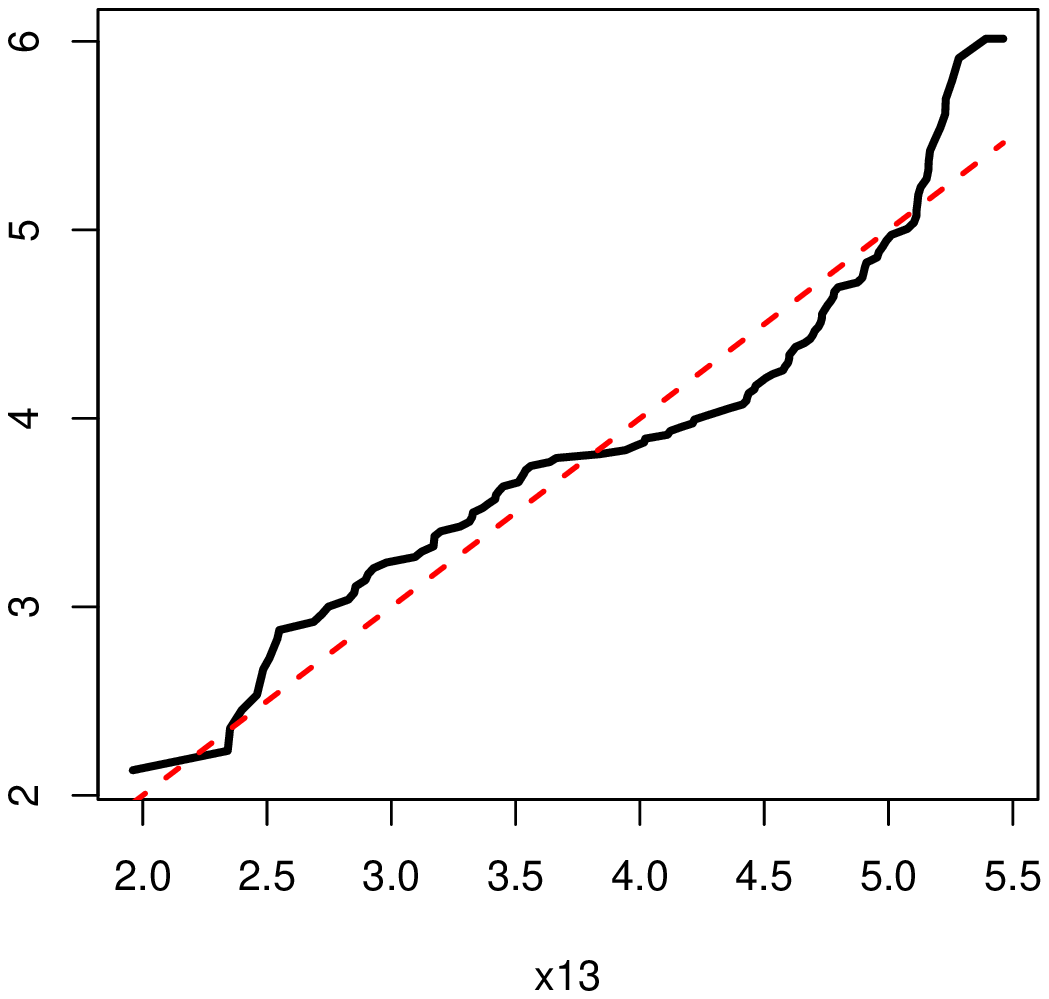} &
\hskip-13pt
\includegraphics[width=.26\textwidth,angle=-90]{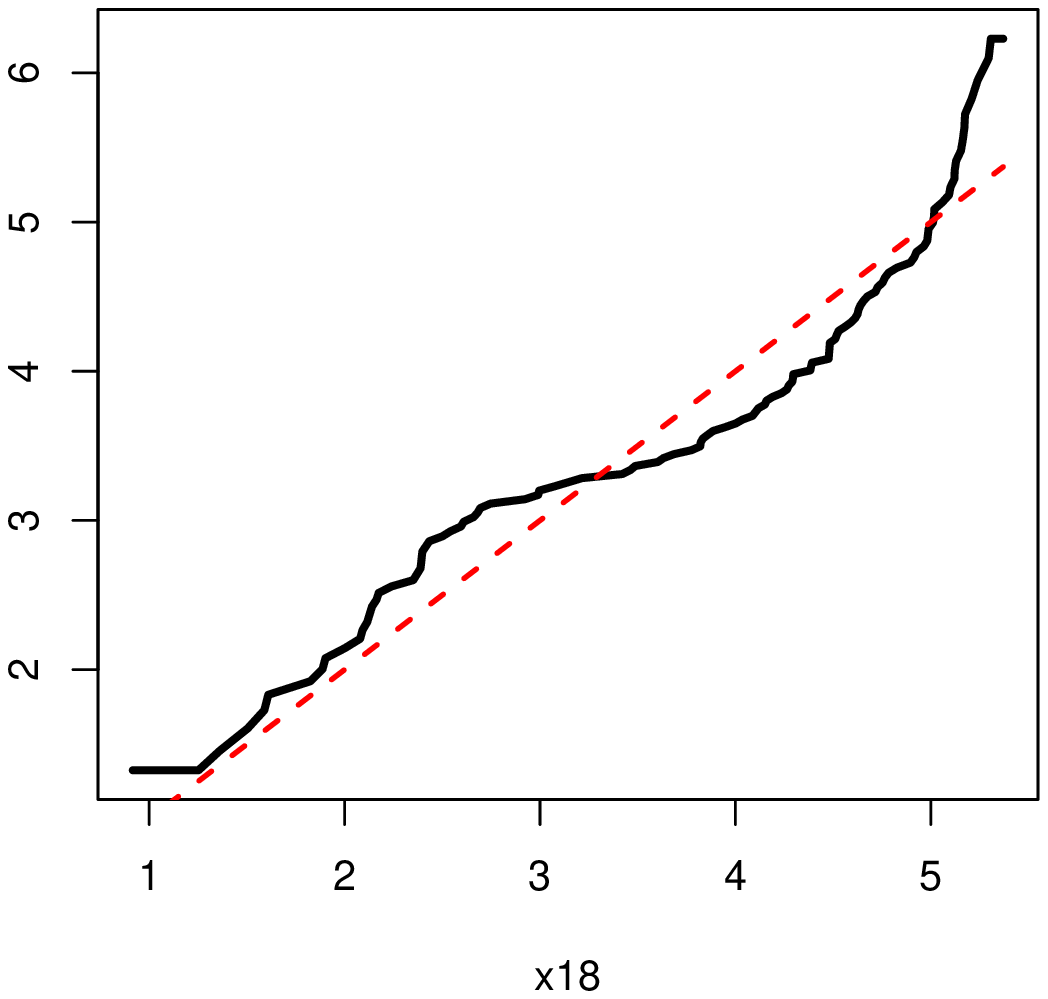}
\\[-10pt]
\end{tabular}
\end{center}
\caption{\small Estimated transformations for the microarray
dataset, indicating non-Gaussian marginals. The corresponding genes are among the
nodes appearing in the symmetric difference graphs of Figure~\ref{fig.GeneSelectedGraph1}.}
\label{fig.Genecomponent1}
\end{figure}

\newpage
\section{Proofs}
\label{sec:proofs}

We assume, without loss of generality from Lemma~\ref{lemma:h}, that
$\mu_{j}=0$  and $\sigma_{j}=1$ for all $j=1,\ldots, p$. Thus, define
$\tilde{f}_j(x) \equiv \Phi^{-1}(\tilde{F}_j(x))$ and ${f}_j(x) \equiv
\Phi^{-1}({F}_j(x))$, and let $g_j\equiv f_j^{-1}$.

\subsection{Proof of Theorem~\ref{thm.keylemma}}


We start with some useful lemmas;
the first is from \cite{Abramovich:06}.
\begin{lemma}\!(Gaussian Distribution function vs.~Quantile function) 
\label{lemma:quantile}
Let $\Phi$ and
$\phi$ denote the distribution and density functions of a
standard Gaussian random variable. Then
\begin{eqnarray}\label{eq.cdfIneq}
\frac{\phi(t)}{2t}\leq 1 - \Phi(t) \leq
\frac{\phi(t)}{t}~~\mathrm{if}~t\geq 1
\end{eqnarray}
and
\begin{eqnarray}\label{eq.quantileDeriv}
(\Phi^{-1})'(\eta) =\frac{1}{\phi\left( \Phi^{-1}(\eta) \right)}.
\end{eqnarray}
Also, for $\eta \geq 0.99$, we have
\begin{eqnarray}\label{eq.evaluatePhiInv}
\Phi^{-1}(\eta) = \ds \sqrt{2\log\left(\frac{1}{1-\eta} \right)} -
r(\eta)
\end{eqnarray}
where $r(\eta) \in [0,1.5]$.
\end{lemma}

\begin{lemma}\!(Distribution function of the transformed random
  variable)\label{lemma:phiinv}
For any $\alpha \in (-\infty, \infty)$
\begin{eqnarray}
\Phi^{-1}\left(F_j\left(g_j (\alpha\sqrt{\log n})\right)\right) =
\alpha\sqrt{\log n}.  \nonumber
\end{eqnarray}
\end{lemma}

\begin{proof}
The statement follows from
\begin{eqnarray}
F_j(t) = \mathbb{P}(X_j \leq t) = \mathbb{P}(g_j(Z_j)\leq t )
=\mathbb{P}(Z_j\leq g^{-1}_j(t) ) =\Phi\left(g_j^{-1}(t) \right). \label{eq.cdfFunction}
\end{eqnarray}
which holds for any $t$.
\end{proof}

\begin{lemma}\!(Gaussian maximal inequality)
\label{lemma.GaussianMaxima}
Let $W_{1},\ldots, W_{n}$ be independently and identically distributed
standard Gaussian random variables.
Then for any $\alpha>0$
\begin{eqnarray}
\mathbb{P}\left( \max_{1\leq i\leq n} W_{i}>\sqrt{\alpha\log n}\right) \leq \frac{1}{n^{\alpha/2-1}}. \nonumber
\end{eqnarray}
\end{lemma}

\begin{proof} 
Using Mill's inequality, we have
\begin{eqnarray}
\mathbb{P}\left( \max_{1\leq i\leq n} W_{i} > \sqrt{\alpha\log n}\right) \leq \sum_{i=1}^{n} \mathbb{P}\left(W_{i} >\sqrt{\alpha\log n}\right) \leq n\frac{\phi(\sqrt{\alpha\log n})}{\sqrt{\alpha \log n }} = \frac{1}{n^{\alpha/2-1}\sqrt{2\pi\alpha\log n}}, \nonumber
\end{eqnarray}
from which the result follows.
\end{proof}

\begin{lemma}\label{lemma.truncationHoeffding}
For any $\alpha >0$ that satisfies $1 -\delta_{n}-\Phi\left(\sqrt{\alpha \log n} \right) > 0$ for all $n$, we have
\begin{eqnarray}\label{eq.tailbound}
\mathbb{P}\left[\hat{F}_{j}  \left( g_{j}\left(\sqrt{\alpha \log n} \right) \right)  > 1 -\delta_{n}\right] \leq \exp\left\{-2n\left( 1 -\delta_{n}-\Phi\left(\sqrt{\alpha\log n} \right)\right)^{2} \right\}.
\end{eqnarray}
and
\begin{eqnarray}\label{eq.lowertailbound}
\mathbb{P}\left[\hat{F}_{j}  \left( g_{j}\left(-\sqrt{\alpha \log n}
    \right) \right)  <\delta_{n}\right] 
\leq \exp\left\{-2n\left( 1 -\delta_{n}-\Phi\left(\sqrt{\alpha\log n} \right)\right)^{2} \right\}.
\end{eqnarray}
\end{lemma}

\begin{proof}
Using Hoeffding's inequality,
\begin{eqnarray}
\lefteqn{
\P\left[\hat{F}_{j}  \left( g_{j}\left(\sqrt{\alpha \log n} \right)
  \right) > 1 -\delta_{n}\right]} \nonumber && \\
\nonumber
& = & \P\left[\hat{F}_{j} \left( g_{j}\left(\sqrt{\alpha \log
        n} \right) \right)  - {F}_{j}\left( g_{j}\left(\sqrt{\alpha
        \log n} \right) \right) 
> 1 -\delta_{n}-{F}_{j} \left( g_{j}\left(\sqrt{\alpha \log n}  \right) \right)\right]  \\
& \leq & \exp\left\{-2n\left( 1 -\delta_{n}-{F}_{j} \left( g_{j}\left(\sqrt{\alpha \log n} \right) \right)\right)^{2} \right\}. \nonumber
\end{eqnarray}
Equation \eqref{eq.tailbound} then follows from equation
\eqref{eq.cdfFunction}. The proof of equation
\eqref{eq.lowertailbound} uses the same argument.
\end{proof}

\def\M{{\mathcal M}}
\def\E{{\mathcal E}}

Now let $M > 2$ be some constant and set $\ds \beta = \frac{1}{2}$.
We split the interval 
$$\left[g_j(-\sqrt{M\log n}), g_j(\sqrt{M\log n})\right]$$
into two parts, the middle
\begin{eqnarray}
\M_{n}\equiv  \left(g_j\left(-\sqrt{\beta \log n}\right),g_j\left(\sqrt{\beta \log n}\right)\right) \nonumber
\end{eqnarray}
and ends
\begin{eqnarray}
\E_n \equiv \left[g_j\left(-\sqrt{M\log n}\right), g_j\left(-\sqrt{\beta \log n}\right) \right] \nonumber
\cup \left[g_j\left(\sqrt{\beta \log n}\right),  g_j\left(\sqrt{M\log n}\right) \right].
\end{eqnarray}
The behaviors of the function estimates in these two regions is
different, and so we first establish bounds on the probability that a sample
can fall in the end region $\E_n$. 

\begin{lemma}\label{lemma.proportion}
Let $\ds A\equiv \sqrt{\frac{2}{\pi}}(\sqrt{M} - \sqrt{\beta})$.  Then
\begin{eqnarray}
\mathbb{P}\left( X_{1j} \in \E_n \right) \leq A \sqrt{\frac{\log n}{n^{\beta}}},~~\forall j \in \{1,\ldots, p\}. \nonumber
\end{eqnarray}
\end{lemma}

\begin{proof}
Using equation \eqref{eq.cdfFunction} and the mean value theorem, we have
 \begin{eqnarray}
\lefteqn{\mathbb{P}\left( X_{1j} \in \E_n \right)} \nonumber \\
&=&
\mathbb{P}\left(
X_{1j} \in\left[g_j(\sqrt{\beta \log n}), g_j(\sqrt{M\log n})\right] \right) + \mathbb{P}\left(
X_{1j} \in\left[g_j(-\sqrt{M\log n}), g_j(-\sqrt{\beta \log n})\right] \right)  \nonumber \\
& = & F_j\left(g_j(\sqrt{M\log n} ) \right) -
F_j\left(g_j(\sqrt{\beta \log
n})\right) + F_j\left(g_j(-\sqrt{\beta\log n} ) \right) -
F_j\left(g_j(-\sqrt{M \log
n})\right) \nonumber \\
& = & 2\left(\Phi(\sqrt{M\log n} ) - \Phi(\sqrt{\beta \log n}) \right) \nonumber \\
& \leq & 2\phi\left(\sqrt{\beta \log n}\right) \left(\sqrt{M\log n} -
\sqrt{\beta \log n} \right). \nonumber 
\end{eqnarray}
The result of the lemma follows directly.
\end{proof}

We next bound the error of the Winsorized estimate of a component
function over the end region.
\begin{lemma}\label{eq.lemma.tails}
For all $n$, we have
\begin{eqnarray} 
\sup_{t\in\E_n}\left|\Phi^{-1}(\tilde{F}_j(t)) - \Phi^{-1}\left( F_j(t)\right) \right| <
\sqrt{2(M+2)\log n},~~\forall j \in \{1,\ldots, p\}. \nonumber 
\end{eqnarray}
\end{lemma}

\begin{proof}
From Lemma \ref{lemma:phiinv} and the definition of $\E_{n}$, we have
$$\sup_{t\in\E_{n}}\left|\Phi^{-1}\left( F_j(t)\right)\right| \in \left[0, \sqrt{M\log n}\right].$$ 
Given the fact that $\delta_{n} = \ds \frac{1}{4n^{1/4}\sqrt{\pi\log n}}$, we have $
\tilde{F}_j(t) \in\ds  \left(\frac{1}{n}, 1-\frac{1}{n}\right)$. Therefore, from equation \eqref{eq.evaluatePhiInv},
\begin{eqnarray}
\sup_{t\in\E_{n}} \left|\Phi^{-1}\left(\tilde{F}_j(t)\right)\right| \in \left[0, \sqrt{2\log n}\right). \nonumber
\end{eqnarray}
The result follows
from the triangle inequality and $\sqrt{M}+\sqrt{2} \leq
\sqrt{2(M+2)}$. 
\end{proof}

Now for any $\epsilon > 0$, we have
\begin{eqnarray}
\lefteqn{\mathbb{P}\left( \max_{j,k} \left| S_n(\tilde{f})_{jk} - 
S_n(f)_{jk}\right| > \epsilon \right)}\nonumber \\
&=& \mathbb{P}\left( \max_{j,k}\left| \frac{1}{n}\sum_{i=1}^n
\left\{\tilde{f}_j(X_{ij})\tilde{f}_k(X_{ik}) - f_j(X_{ij})f_k(X_{ik}) -
\mu_n(\tilde{f}_j)\mu_n(\tilde{f}_k) + \mu_n(f_j)\mu_n(f_k)\right\} \right|>\epsilon
\right) \nonumber \\
&\leq & \mathbb{P}\left( \max_{j,k}\left| \frac{1}{n}\sum_{i=1}^n
\left(\tilde{f}_j(X_{ij})\tilde{f}_k(X_{ik}) - f_j(X_{ij})f_k(X_{ik})\right)
\right|>\frac{\epsilon}{2} \right) \nonumber \\
& &  +\; \mathbb{P}\left(
\max_{j,k}\left| \mu_n(\tilde{f}_j)\mu_n(\tilde{f}_k) - \mu_n(f_j)\mu_n(f_k)
\right|>\frac{\epsilon}{2}\right). \label{eq.decomposition1}  \nonumber
\end{eqnarray}
We only need to analyze the rate for the first term above,
since the second one is of higher order \citep{Cai:08}.
Let
\begin{eqnarray}
\Delta_i(j,k) \equiv \tilde{f}_j(X_{ij})\tilde{f}_k(X_{ik}) -
f_j(X_{ij}) f_k(X_{ik}) \nonumber
\end{eqnarray}
and
\begin{eqnarray}
\Theta_{t,s}(j,k) \equiv \tilde{f}_j(t)\tilde{f}_k(s) - f_j(t) f_k(s). \nonumber
\end{eqnarray}
We define the event $\mathcal{A}_{n}$ as
\begin{eqnarray}
\mathcal{A}_{n}\equiv \left\{ g_{j}\left(-\sqrt{M\log n}\right)\leq X_{1j}, \ldots, X_{nj} \leq g_{j}\left(\sqrt{M\log n}\right), j=1,\ldots, p\right\}. \nonumber
\end{eqnarray}
Then
\begin{eqnarray}
\mathbb{P}(\mathcal{A}_{n}^{c}) \leq p\mathbb{P}\left( \min_{i} f_{j}\left(X_{ij}\right) <-\sqrt{M\log n}\right) + p\mathbb{P}\left( \max_{i} f_{j}\left(X_{ij}\right) > \sqrt{M\log n}\right) \leq \frac{c_{1}n^{\xi}}{n^{M/2-1}}.\nonumber
\end{eqnarray}
with $c_{1}$ as a generic positive constant.
Therefore
\begin{eqnarray}
\mathbb{P}\left(\max_{j,k}\left|\frac{1}{n}\sum_{i=1}^n
\Delta_i(j,k) \right| >\epsilon \right) 
 \leq \mathbb{P}\left(\max_{j,k}\left|\frac{1}{n}\sum_{i=1}^n
\Delta_i(j,k) \right| >\epsilon, \mathcal{A}_{n}\right)  &+& \mathbb{P}(\mathcal{A}_{n}^{c}) \nonumber \\
 \leq \mathbb{P}\left(\max_{j,k}\left|\frac{1}{n}\sum_{i=1}^n
\Delta_i(j,k) \right| >\epsilon, \mathcal{A}_{n}\right)  &+&  \frac{c_{1}n^{\xi}}{n^{M/2-1}}. \nonumber
\end{eqnarray}
Thus, we only need to carry out our analysis on the event
$\mathcal{A}_{n}$.  On this event, we have the following decomposition:
\begin{eqnarray}
\lefteqn{\mathbb{P}\left(\max_{j,k}\left|\frac{1}{n}\sum_{i=1}^n
\Delta_i(j,k) \right| >\epsilon, \,\mathcal{A}_n \right)   } \nonumber \\
& \leq & \mathbb{P}\left(\max_{j,k} \frac{1}{n}\sum_{X_{ij} \in
\M_n, X_{ik}\in \M_n}\left| \Delta_i(j,k) \right|
>\frac{\epsilon}{4} \right) + \mathbb{P}\left(\max_{j,k} \frac{1}{n}\sum_{X_{ij} \in
\E_n, X_{ik}\in \E_n}\left| \Delta_i(j,k) \right|
>\frac{\epsilon}{4} \right) \nonumber \\
&  & +\; 2 \mathbb{P}\left(\max_{j,k} \frac{1}{n}\sum_{X_{ij} \in
\M_n, X_{ik}\in \E_n}\left| \Delta_i(j,k) \right|
>\frac{\epsilon}{4} \right).  \nonumber
\end{eqnarray}

We now analyze each of these terms separately.
\begin{lemma} \label{lemma.term1} 
On the event $\mathcal{A}_{n}$,  let $\beta=1/2$ and  $\ds \epsilon \geq C(M,\xi)\sqrt{\frac{\log p \log^2 n}{n^{1/2}}}$, then
\begin{eqnarray}
\mathbb{P}\left(\max_{j,k}
\frac{1}{n}\sum_{X_{ij} \in \E_n, X_{ik}\in
\E_n}\left| \Delta_i(j,k) \right|
>\frac{\epsilon}{4} \right) = o(1). \nonumber
\end{eqnarray}
\end{lemma}

\begin{proof}
We define
\begin{eqnarray}
\theta_{1} \equiv  \frac{n^{\beta/2}\epsilon}{8A\sqrt{\log n}}
\nonumber
\end{eqnarray}
with the same parameter $A$ as in Lemma \ref{lemma.proportion}. 
Such a $\theta_{1}$ guarantees that
\begin{eqnarray}
\frac{n\epsilon}{4\theta_1} - nA\sqrt{\frac{\log n}{n^{\beta}}} = nA\sqrt{\frac{\log n}{n^{\beta}}}  > 0. \nonumber
\end{eqnarray}
By Lemma \ref{lemma.proportion}, we have
\begin{eqnarray}
\lefteqn{\mathbb{P}\left(\frac{1}{n}\sum_{i=1}^n\mathbf{1}_{\{ X_{ij} \in
\E_n, X_{ik}\in \E_n \}} >
\frac{\epsilon}{4\theta_1} \right) \;\leq\;
\mathbb{P}\left(\sum_{i=1}^n\mathbf{1}_{\{ X_{ij} \in
	\E_n\}} > \frac{n\epsilon}{4\theta_1} \right) } \qquad\qquad  && \nonumber\\
&=& \mathbb{P}\left(\sum_{i=1}^n\left(\mathbf{1}_{\{ X_{ij} \in
	\E_n\}} - \mathbb{P} \left(X_{1j}\in\E_{n} \right)\right) > \frac{n\epsilon}{4\theta_1} - n\mathbb{P} \left(X_{1j}\in\E_{n} \right)\right)\nonumber\\	
&\leq & \mathbb{P}\left(\sum_{i=1}^n\left(\mathbf{1}_{\{ X_{ij} \in
	\E_n\}} - \mathbb{P} \left(X_{1j}\in\E_{n} \right)\right) > \frac{n\epsilon}{4\theta_1} - nA\sqrt{\frac{\log n}{n^{\beta}}}\right).	 \nonumber
\end{eqnarray}

Using the Bernstein's inequality, for $\beta=\ds \frac{1}{2}$,
\begin{eqnarray}
\mathbb{P}\left(\frac{1}{n}\sum_{i=1}^n\mathbf{1}_{\{ X_{ij} \in
\E_n, X_{ik}\in \E_n \}} >
\frac{\epsilon}{4\theta_1} \right) & \leq &\mathbb{P}\left(\sum_{i=1}^n\left(\mathbf{1}_{\{ X_{ij} \in
	\E_n\}} - \mathbb{P} \left(X_{1j}\in\E_{n} \right)\right) > nA\sqrt{\frac{\log n}{n^{\beta}}}\right) \nonumber\\
	& \leq & \exp\left(-\frac{c_{1}n^{2-\beta}\log n}{c_{2}n^{1-\beta/2}\sqrt{\log n} + c_{3}n^{1-\beta/2}\sqrt{\log n}}\right) =o(1), \nonumber
\end{eqnarray}
where $c_1, c_{2}, c_{3} >0$ are generic constants.  

Therefore,
\begin{eqnarray}
\lefteqn{\mathbb{P}\left(\max_{j,k} \frac{1}{n}\sum_{X_{ij} \in
\E_n, X_{ik}\in \E_n}\left| \Delta_i(j,k) \right|
>\frac{\epsilon}{4} \right)}\nonumber \\
& = & \mathbb{P}\left(\max_{j,k} \frac{1}{n}\sum_{X_{ij} \in
\E_n, X_{ik}\in \E_n}\left| \Delta_i(j,k) \right|
>\frac{\epsilon}{4}, \max_{j,k}\sup_{t\in \E_n, s\in\E_n} \left|\Theta_{t,s}(j,k) \right| > \theta_1
\right) \nonumber \\
& &~+ \mathbb{P}\left(\max_{j,k} \frac{1}{n}\sum_{X_{ij} \in
\E_n, X_{ik}\in \E_n}\left| \Delta_i(j,k) \right|
>\frac{\epsilon}{4}, \max_{j,k}\sup_{t\in \E_n, s\in\E_n} \left|\Theta_{t,s}(j,k) \right| \leq \theta_1
\right) \nonumber \\
& \leq & \mathbb{P}\left(\max_{j,k}\sup_{t\in \E_n,
s\in\E_n} \left|\Theta_{t,s}(j,k) \right| > \theta_1
\right) + \mathbb{P}\left(\frac{1}{n}\sum_{i=1}^n\mathbf{1}_{\{
X_{ij} \in \E_n, X_{ik}\in \E_n \}} >
\frac{\epsilon}{4\theta_1} \right) \nonumber \\
&= &  \mathbb{P}\left(\max_{j,k}\sup_{t\in \E_n,
s\in\E_n} \left|\Theta_{t,s}(j,k) \right| > \theta_1
\right)  + o(1). \nonumber
\end{eqnarray}

Now, we analyze the first term
\begin{eqnarray}
 \mathbb{P}\left(\max_{j,k}\sup_{t\in \E_n,
s\in\E_n} \left|\Theta_{t,s}(j,k) \right| > \theta_1
\right) &\leq & p^2  \mathbb{P}\left(\sup_{t\in \E_n,
s\in\E_n} \left|\Theta_{t,s}(j,k) \right| > \theta_1
\right) \nonumber \\
& =&
p^2\mathbb{P}\left(\sup_{t\in\E_n,s\in\E_n}|\tilde{f}_j(t)\tilde{f}_k(s)
- f_j(t) f_k(s)|
> \theta_1\right). \nonumber
\end{eqnarray}
By adding and subtracting terms $f_j(t)$ and $f_s(t)$, we have
\begin{eqnarray}
\lefteqn{\mathbb{P}\left(\sup_{t\in\E_n,s\in\E_n}|\tilde{f}_j(t)\tilde{f}_k(s)
- f_j(t) f_k(s)|
> \theta_1\right)
} ~~~~~~~~~~~~~~~~~~~~~~~~~\nonumber\\
&\leq&
\mathbb{P}\left(\sup_{t\in\E_n,s\in\E_n}|(\tilde{f}_j(t)
-
f_j(t))(\tilde{f}_k(s)-f_k(s))| > \frac{\theta_1}{3}\right)\nonumber \\
& &+ ~\mathbb{P}\left(\sup_{t\in\E_n,
s\in\E_n}|(\tilde{f}_j(t) - f_j(t))|\cdot|f_k(s)|
> \frac{\theta_1}{3}\right) \nonumber \\
& & + ~\mathbb{P}\left(\sup_{t\in\E_n,
s\in\E_n}|(\tilde{f}_k(s) - f_k(s))|\cdot|f_j(t)|
> \frac{\theta_1}{3}\right). \nonumber
\end{eqnarray}
The first term can further be decomposed to be
\begin{eqnarray}
\lefteqn{
\mathbb{P}\left(\sup_{t\in\E_n,s\in\E_n}|(\tilde{f}_j(t)
- f_j(t))(\tilde{f}_k(s)-f_k(s))| > \frac{\theta_1}{3}\right)} ~~~\nonumber\\
& \leq & \mathbb{P}\left( \sup_{t\in\E_n}|(\tilde{f}_j(t) -
f_j(t))|>\sqrt{\frac{\theta_1}{3}}\right) + \mathbb{P}\left(
\sup_{s\in\E_n}|(\tilde{f}_k(s) -
f_k(s))|>\sqrt{\frac{\theta_1}{3}}\right). \label{eq.keyterm1} \nonumber
\end{eqnarray}
Also, from the definition of $\mathcal{E}_{n}$, we have
\begin{eqnarray}
\sup_{t\in\E_n}|f_j(t)| =
\sup_{t\in\E_n}\left|g_j^{-1}(t)\right| \leq \sqrt{M\log n}. \nonumber
\end{eqnarray}
Since $\ds \epsilon \geq C(M,\xi)\sqrt{\frac{\log p \log^2 n}{n^{1/2}}}$, we have 
\begin{eqnarray}
\frac{\theta_{1}}{3} =  \frac{n^{\beta/2}\epsilon}{24A\sqrt{\log n}} \geq \frac{ C(M,\xi)\sqrt{\log p \log^2 n}}{24A\sqrt{\log n}}=2(M+2)\log n.
\nonumber
\end{eqnarray}
This implies that
\begin{eqnarray}
\sqrt{\frac{\theta_{1}}{3}} \geq \sqrt{2(M+2)\log n}~~\mathrm{and}~~\frac{\theta_{1}}{3\sqrt{M\log n}} \geq \sqrt{2(M+2)\log n}. \nonumber
\end{eqnarray}
Then, from Lemma \ref{eq.lemma.tails}, we get
\begin{eqnarray}
 \mathbb{P}\left( \sup_{t\in\E_n}|(\tilde{f}_j(t) -
f_j(t))|>\sqrt{\frac{\theta_1}{3}}\right) = 0 \nonumber
\end{eqnarray}
and
\begin{eqnarray}
\mathbb{P}\left(\sup_{t\in\E_n,
s\in\E_n}|(\tilde{f}_j(t) - f_j(t))|\cdot|f_k(s)|
> \frac{\theta_1}{3}\right)=0. \nonumber
\end{eqnarray}
The claim of the lemma then follows directly.
\end{proof}

\begin{remark}
From the above analysis, we see that the data in the tails doesn't affect
the rate. Using exactly the same argument, we can also show that 
\begin{eqnarray}
\mathbb{P}\left(\max_{j,k} \frac{1}{n}\sum_{X_{ij} \in
\M_n, X_{ik}\in \E_n}\left| \Delta_i(j,k) \right|
>\frac{\epsilon}{4} \right) = o(1). \nonumber
\end{eqnarray}
\end{remark}

\begin{lemma}\label{lemma.term2}
On the event $\mathcal{A}_{n}$,  let $\beta=1/2$ and  $\ds \epsilon \geq C(M,\xi)\sqrt{\frac{\log p \log^2 n}{n^{1/2}}}$.  There exist generic constants  $c_{1}, c_{2}, c_{3}, c_{4}$, such that 
\begin{eqnarray}
\mathbb{P}\left(\max_{j,k} \frac{1}{n}\sum_{X_{ij} \in
\M_n, X_{ik}\in \M_n}\hskip-10pt \left| \Delta_i(j,k) \right|
>\frac{\epsilon}{4} \right) \leq c_{2}\exp\left(-\frac{c_{1}n^{1-\beta}\epsilon^{2}}{\log p  \log^2 n  }\right) + c_{3}\exp\left(-\frac{c_{4}n^{1-\beta}}{\log p(\log n)} \right). \nonumber
\end{eqnarray}
\end{lemma}

\begin{proof}
We have
\begin{eqnarray}
\lefteqn{\mathbb{P}\left(\max_{j,k} \frac{1}{n}\sum_{X_{ij} \in
\M_n, X_{ik}\in \M_n}\left| \Delta_i(j,k) \right|
>\frac{\epsilon}{4} \right)  \leq p^2 \mathbb{P}\left(\sup_{t\in\M_n,s\in\M_n}|\tilde{f}_j(t)\tilde{f}_k(s)
- f_j(t)f_k(s)|
> \frac{\epsilon}{4}\right) } ~~~~~~~~~~~~~~~~~~~~~~~~~~~~~~~~~\nonumber \\
& &
\leq p^{2}\mathbb{P}\left(\sup_{t\in\M_n,s\in\M_n}|(\tilde{f}_j(t)
-
f_j(t))(\tilde{f}_k(s)-f_k(s))| > \frac{\epsilon}{12}\right) \nonumber \\
& &+ ~2p^{2}\mathbb{P}\left(\sup_{t\in\M_n,
s\in\M_n}|(\tilde{f}_j(t) - f_j(t))|\cdot|f_k(s)|
> \frac{\epsilon}{12}\right). \nonumber 
\end{eqnarray}

Further, since
\begin{eqnarray}
\sup_{t\in\M_n}|f_j(t)| =
\sup_{t\in\M_n}\left|g_j^{-1}(t)\right| = \sqrt{\beta\log n} \nonumber
\end{eqnarray}
and $\ds\sup_{t\in\M_n,s\in\M_n}|(\tilde{f}_j(t)
- f_j(t))(\tilde{f}_k(s)-f_k(s))|  $ is of higher order than $\sup_{t\in\M_n,
s\in\M_n}|(\tilde{f}_j(t) - f_j(t))|\cdot|f_k(s)|$,   we only need to analyze the term $ \mathbb{P}\left(
\sup_{t\in\M_n}|(\tilde{f}_j(t) - f_j(t))|>\ds\frac{\epsilon}{12\sqrt{\beta\log n}}\right)$.

Since $ \delta_{n} = \ds  \frac{1}{4n^{\beta/2}\sqrt{2\pi\beta \log
    n}} $, using Mill's inequality we have
\begin{eqnarray}
2\delta_{n} = \frac{\phi(\sqrt{\beta\log n})}{2\sqrt{\beta\log n}} \leq 1 - \Phi(\sqrt{\beta\log n}). \nonumber
\end{eqnarray}
This implies that 
\begin{eqnarray}
1 -\delta_{n} - \Phi(\sqrt{\beta\log n}) \geq  \delta_{n}  >0. \nonumber
\end{eqnarray}
Using Lemma \ref{lemma.truncationHoeffding}, we have
\begin{eqnarray}\label{eq.ineq1}
p^{2}\mathbb{P} \left( \hat{F}_{j}  \left( g_{j}\left(\sqrt{\beta \log n} \right) \right)  > 1 -\delta_{n}\right) \leq  \exp\left(-\frac{n\delta_{n}^{2}}{2\log p} \right) =  \exp\left(-\frac{n^{1-\beta}}{\log p(64\pi\beta\log n)} \right)
\end{eqnarray}
and
\begin{eqnarray}\label{eq.ineq2}
p^{2}\mathbb{P}\left(\hat{F}_{j}  \left( g_{j}\left(-\sqrt{\beta \log n} \right) \right)<\delta_{n}\right)  \leq  \exp\left(-\frac{n^{1-\beta}}{\log p(64\pi\beta\log n)} \right).
\end{eqnarray}

Define an event $\mathcal{B}_{n}$ as
\begin{eqnarray}
\mathcal{B}_{n} \equiv \left\{\delta_{n} \leq \hat{F}_{j}  \left( g_{j}\left(\sqrt{\beta \log n} \right) \right)\leq1 -\delta_{n}, j=1,\ldots, p\right\}.  \nonumber
\end{eqnarray}
From \eqref{eq.ineq1} and \eqref{eq.ineq2}, it is easy to see that
\begin{eqnarray}
\mathbb{P}\left(\mathcal{B}_{n}^{c} \right) \leq c_{3}\exp\left(-\frac{c_{4}n^{1-\beta}}{\log p(\log n)} \right) \nonumber
\end{eqnarray}
where $c_{3}$ and $c_{4}$ are generic positive constants.

From the definition of $\tilde{F}_{j}$, we have 
\begin{eqnarray}
\lefteqn{p^{2}\mathbb{P}\left(\sup_{t\in\M_n}|(\tilde{f}_j(t) -
f_j(t))|>\frac{\epsilon}{12\sqrt{\beta\log n}} \right)}  \nonumber\\
&\leq &p^{2} \mathbb{P}\left(\sup_{t \in \M_n}
\left|\Phi^{-1}\left(\tilde{F}_j(t)\right)- \Phi^{-1}\left( F_j(t)
\right)  \right|
>
\frac{\epsilon}{12\sqrt{\beta\log n}}, \mathcal{B}_{n}\right)+\mathbb{P}\left( \mathcal{B}_{n}^{c}\right). \nonumber \\
&\leq & p^{2} \mathbb{P}\left(\sup_{t \in \M_n}
\left|\Phi^{-1}\left(\hat{F}_j(t)\right)- \Phi^{-1}\left( F_j(t)
\right)  \right|
>
\frac{\epsilon}{12\sqrt{\beta\log n}}\right)+c_{3}\exp\left(-\frac{c_{4}n^{1-\beta}}{\log p(\log n)} \right). \nonumber
\end{eqnarray}

Define
\begin{eqnarray}
T_{1n}\equiv \max\left\{F_j\left(g_j\left(\sqrt{\beta\log
n}\right)\right), 1-\delta_{n}\right\}~~\mathrm{and}~~T_{2n}\equiv 1-\min\left\{F_j\left(g_j\left(-\sqrt{\beta\log
n}\right)\right), \delta_{n} \right\}.\nonumber
\end{eqnarray}
From equation \eqref{eq.cdfFunction} and the fact that $1-\delta_{n} \geq  \Phi\left(\sqrt{\beta\log
n}\right) $, we have  that
\begin{eqnarray}
T_{1n} = T_{2n} = 1-\delta_{n}. \nonumber
\end{eqnarray}
Thus, by the mean value theorem,
\begin{eqnarray}
\lefteqn{\mathbb{P}\left(\sup_{t \in \M_n}
\left|\Phi^{-1}\left(\hat{F}_j(t)\right) - \Phi^{-1}\left( F_j(t)
\right)  \right|
>
\frac{\epsilon}{12\sqrt{\beta\log n}}\right)}\nonumber \\
& \leq &
\mathbb{P}\left((\Phi^{-1})'\left(\max\left\{T_{1n}, T_{2n}\right\}\right)\sup_{t \in
\M_n}\left|\hat{F}_j(t) - F_j(t) \right|
>
\frac{\epsilon}{12\sqrt{\beta\log n}}\right)  \nonumber \\
& =&
\mathbb{P}\left((\Phi^{-1})'\left(1-\delta_{n}\right)\sup_{t \in
\M_n}\left|\hat{F}_j(t) - F_j(t) \right|
>
\frac{\epsilon}{12\sqrt{\beta\log n}}\right).  \nonumber
\end{eqnarray}

Finally, using the  Dvoretzky-Kiefer-Wolfowitz inequality,
\begin{eqnarray}
\lefteqn{\mathbb{P}\left(\sup_{t \in \M_n}
\left|\Phi^{-1}\left(\hat{F}_j(t)\right) - \Phi^{-1}\left( F_j(t)
\right)  \right|
>
\frac{\epsilon}{12\sqrt{\beta\log n}}\right) } ~~~~~~~~~~~~~~~~~~~~~~~~~~~~~~~~~~~~\nonumber \\
& \leq &\mathbb{P}\left(\sup_{t \in \M_n}\left|\hat{F}_j(t) -
F_j(t) \right|
>
\frac{\epsilon}{(\Phi^{-1})'\left(1-\delta_{n} \right)12\sqrt{\beta\log n} }\right) \nonumber \\
& \leq &  
2\exp\left(-2\frac{n\epsilon^{2}}{144\beta\log n \left[( \Phi^{-1})'\left(1-\delta_{n}\right) \right]^{2} }\right) \nonumber.
\end{eqnarray}
Furthermore,  by Lemma \ref{lemma:quantile},
\begin{eqnarray}
(\Phi^{-1})'\left(1-\delta_{n}\right) = \frac{1}{\phi\left(\Phi^{-1} (1-\delta_{n})\right)} \leq \frac{1}{\phi\left(\sqrt{2\log \ds \frac{1}{\delta_{n}}}\right)} = \sqrt{2\pi}\left( \frac{1}{\delta_{n}}\right)= 8\pi n^{\beta/2}\sqrt{\beta\log n}. \nonumber
\end{eqnarray} 
This implies that
\begin{eqnarray}
p^{2}\mathbb{P}\left(\sup_{t \in \M_n}
\left|\Phi^{-1}\left(\hat{F}_j(t)\right) - \Phi^{-1}\left( F_j(t)
\right)  \right|
>
\frac{\epsilon}{12\sqrt{\beta\log n}}\right) \leq
2\exp\left(-\frac{c_{1}n^{1-\beta}\epsilon^{2}}{\log p  \log^2 n  }\right)\nonumber
\end{eqnarray}
where $c_{1}$ is a generic constant.

In summary, we have 
\begin{eqnarray*}
\mathbb{P}\left(\max_{j,k} \frac{1}{n}\sum_{X_{ij} \in
\M_n, X_{ik}\in \E_n}\!\left| \Delta_i(j,k) \right|
>\frac{\epsilon}{4} \right) \leq c_{2}\exp\left(-\frac{c_{1}n^{1-\beta}\epsilon^{2}}{\log p  \log^2 n  }\right) + c_{3}\exp\left(-\frac{c_{4}n^{1-\beta}}{\log p(\log n)} \right)
\end{eqnarray*}
where $c_{1}, c_{2}, c_{3}, c_{4}$ are generic constants.
\end{proof}

The conclusion of Theorem \ref{thm.keylemma} follows from Lemma \ref{lemma.term1} and Lemma \ref{lemma.term2}.

\subsection{Proof of Theorem~\ref{thm:persist}}

\begin{proof}
First note that the population and sample risks are
\begin{eqnarray*}
R(f, \Omega) &=&  \frac{1}{2}\left\{\tr\left[\Omega \mathbb{E}(f(X) f(X)^T\right] - \log
  | \Omega | - p\log (2\pi)\right\} \\
\hat R(f, \Omega) &=&  \frac{1}{2}\left\{\tr\left[\Omega S_n(f)\right] - \log
  | \Omega | - p\log (2\pi)\right\}.
\end{eqnarray*}
Therefore, for all $(f,\Omega)\in \cM_n^p\oplus\cE_n$, we have
\begin{eqnarray*}
|\hat{R}(f,\Omega) - R(f,\Omega)| &=&  
\frac{1}{2} \left|  \tr\left[\Omega\left(\mathbb{E}[f f^T] - S_n(f)\right)\right] \right| \\
&\leq&  \frac{1}{2} \left\|\Omega\right\|_1
\max_{jk}  \sup_{f_j, f_k \in \cM_n} |\mathbb{E}(f_j(X_j) f_k(X_k) - S_n(f)_{jk}| \\
&\leq&  \frac{L_n}{2} 
\max_{jk}  \sup_{f_j, f_k \in \cM_n} |\mathbb{E}(f_j(X_j) f_k(X_k) - S_n(f)_{jk}|.
\end{eqnarray*}

Now, if $\cF$ is a class of functions, we have 
\begin{equation}\label{eq::emp}
\mathbb{E}
\Biggl(\sup_{g\in\cF}  |\hat\mu(g) - \mu(g)|\Biggr) \leq
\frac{C\, J_{[\,]}(\left\|F\right\|_\infty,\cF)}{\sqrt{n}}
\end{equation}
for some $C>0$, where
$F(x) = \sup_{g\in cF} |g(x)|$,
$\mu(g) = \mathbb{E}(g(X))$ and
$\hat\mu(g) = n^{-1}\sum_{i=1}^n g(X_i)$ (see Corollary 19.35 of \cite{vaar:1998}).  Here
the bracketing integral is defined to be
\begin{equation}
J_{[\,]}(\delta,\cF) =\int_0^\delta \sqrt{\log N_{[\,]}(u,\cF)} \,du \nonumber
\end{equation}
where $\log N_{[\,]}(\epsilon, \cF)$ is the bracketing entropy.
For the class of one dimensional, bounded and monotone
functions, the bracketing entropy satisfies
\begin{equation}\label{eq::ent-of-T}
\log N_{[\,]}(\epsilon, \cM) \leq K \left(\frac{1}{\epsilon}\right) \nonumber
\end{equation}
for some $K>0$ \citep{van:well:1996}.

Now, let $\cP_{n,p}$ be the class of all functions of the
form $m(x) = f_j(x_j) f_k(x_k)$ for $j,k \in \{1,\ldots, p\}$, where
$f_j \in \cM_n$ for each $j$. 
Then the bracketing entropy satisfies
\begin{equation*}
\log N_{[\,]}(C\sqrt{\log n},\cP_{n,p})  \leq 2\log p +
K\left(\frac{1}{\epsilon}\right)
\end{equation*}
and the  bracketing integral satisfies $J_{[\,]}(C\sqrt{\log n},\cP_{n,p}) =
O(\sqrt{\log n \log p})$.  
It follows from \eqref{eq::emp} and Markov's inequality
that 
\begin{equation*}
\max_{jk}\sup_{f_j, f_k \in \cM_n}
 | S_n(f)_{jk} - \mathbb{E}(f_j(X_j) f_k(X_k)| =
O_P\left(\sqrt{\frac{\log n\log p}{n}}\right) = 
O_P\left(\sqrt{\frac{\log n}{n^{1-\xi}}}\right).
\end{equation*}
Therefore,
\begin{equation*}
\sup_{(f,\Omega)\in \cM_n^p\oplus \cE_n}|\hat{R}(f,\Omega) - R(f,\Omega)| =
O_P\left(\frac{L_n \sqrt{\log n}}{n^{(1-\xi)/2}}\right).
\end{equation*}
As a consequence, we have
\begin{eqnarray*}
R(f^*,\Omega^*) 
&\leq & R(\tilde{f}_n, \hat\Omega_n)  \\
&\leq & \hat R(\tilde{f}_n, \hat\Omega_n) + O_P\left(\frac{L_n\sqrt{\log n}}{n^{(1-\xi)/2}}\right)\\
& \leq & \hat{R}(f^*, \Omega^*) + O_P\left(\frac{L_n\sqrt{\log n}}{n^{(1-\xi)/2}}\right)\\
&\leq & R(f^*, \Omega^*) + O_P\left(\frac{L_n\sqrt{\log n}}{n^{(1-\xi)/2}}\right)
\end{eqnarray*}
and the conclusion follows.
\end{proof}

\section{Concluding Remarks}

In this paper we have introduced the nonparanormal, a type of Gaussian
copula with nonparametric marginals that is suitable for estimating
high dimensional undirected graphs.  The nonparanormal can be viewed
as an extension of sparse additive models to the setting of graphical
models.  We proposed an estimator for the component functions that is
based on thresholding the tails of the empirical distribution function
at appropriate levels.  A theoretical analysis was given to bound the
difference between the sample covariance with respect to these
estimated functions and the true sample covariance.  This analysis
was leveraged with the recent work of \cite{Ravikumar:Gauss:09} and
\cite{Rothman:08} to obtain consistency results for the
nonparanormal.  Computationally, fitting a high dimensional
nonparanormal is no more difficult than estimating a multivariate
Gaussian, and indeed one can exploit existing software for the
graphical lasso.  Our experimental results indicate that the
sparse nonparanormal can give very different results than a
sparse Gaussian graphical model, suggesting that it may
be a useful tool for relaxing the normality assumption, which is often
made only for convenience.

\bibliography{npn}

\begin{thebibliography}{20}
\expandafter\ifx\csname natexlab\endcsname\relax\def\natexlab#1{#1}\fi
\expandafter\ifx\csname url\endcsname\relax
  \def\url#1{\texttt{#1}}\fi
\expandafter\ifx\csname urlprefix\endcsname\relax\def\urlprefix{URL }\fi

\bibitem[{Abramovich et~al.(2006)Abramovich, Benjamini, Donoho and
  Johnstone}]{Abramovich:06}
\textsc{Abramovich, F.}, \textsc{Benjamini, Y.}, \textsc{Donoho, D.~L.} and
  \textsc{Johnstone, I.~M.} (2006).
\newblock Adapting to unknown sparsity by controlling the false discovery rate.
\newblock \textit{The Annals of Statistics} \textbf{34} 584--653.

\bibitem[{Banerjee et~al.(2008)Banerjee, Ghaoui and d'Aspremont}]{Banerjee:08}
\textsc{Banerjee, O.}, \textsc{Ghaoui, L.~E.} and \textsc{d'Aspremont, A.}
  (2008).
\newblock Model selection through sparse maximum likelihood estimation.
\newblock \textit{Journal of Machine Learning Research} \textbf{9} 485--516.

\bibitem[{Cai et~al.(2008)Cai, Zhang and Zhou}]{Cai:08}
\textsc{Cai, T.}, \textsc{Zhang, C.-H.} and \textsc{Zhou, H.~H.} (2008).
\newblock Optimal rates of convergence for covariance matrix estimation.
\newblock Tech. rep., Wharton School, Statistics Department, University of
  Pennsylvania.

\bibitem[{Drton and Perlman(2007)}]{Drton:07}
\textsc{Drton, M.} and \textsc{Perlman, M.~D.} (2007).
\newblock Multiple testing and error control in {G}aussian graphical model
  selection.
\newblock \textit{Statistical Science} \textbf{22} 430--449.

\bibitem[{Drton and Perlman(2008)}]{Drton:08}
\textsc{Drton, M.} and \textsc{Perlman, M.~D.} (2008).
\newblock A {SIN}ful approach to {G}aussian graphical model selection.
\newblock \textit{Journal of Statistical Planning and Inference} \textbf{138}
  1179--1200.

\bibitem[{Friedman et~al.(2007)Friedman, Hastie and Tibshirani}]{FHT:07}
\textsc{Friedman, J.}, \textsc{Hastie, T.} and \textsc{Tibshirani, R.} (2007).
\newblock Sparse inverse covariance estimation with the graphical lasso.
\newblock \textit{Biostatistics} \textbf{9} 432--441.

\bibitem[{Hastie and Tibshirani(1999)}]{Hast:Tibs:1999}
\textsc{Hastie, T.} and \textsc{Tibshirani, R.} (1999).
\newblock \textit{Generalized additive models}.
\newblock Chapman \& Hall Ltd.

\bibitem[{Mallows(1990)}]{Tukey:VI:90}
\textsc{Mallows, C.~L.} (ed.)  (1990).
\newblock \textit{The collected works of {J}ohn {W.} {T}ukey. {V}olume {VI}:
  {M}ore mathematical, 1938--1984}.
\newblock Wadsworth \& Brooks/Cole.

\bibitem[{Meinshausen and B\"{u}hlmann(2006)}]{Meinshausen:2006}
\textsc{Meinshausen, N.} and \textsc{B\"{u}hlmann, P.} (2006).
\newblock High dimensional graphs and variable selection with the {L}asso.
\newblock \textit{The Annals of Statistics} \textbf{34} 1436--1462.

\bibitem[{Ravikumar et~al.(2008{\natexlab{a}})Ravikumar, Lafferty, Liu and
  Wasserman}]{Ravikumar:08}
\textsc{Ravikumar, P.}, \textsc{Lafferty, J.}, \textsc{Liu, H.} and
  \textsc{Wasserman, L.} (2008{\natexlab{a}}).
\newblock Sparse additive models.
\newblock \textit{Journal of the Royal Statistical Society, Series B,
  Methodological} To appear.

\bibitem[{Ravikumar et~al.(2008{\natexlab{b}})Ravikumar, Liu, Lafferty and
  Wasserman}]{Ravikumar:07}
\textsc{Ravikumar, P.}, \textsc{Liu, H.}, \textsc{Lafferty, J.} and
  \textsc{Wasserman, L.} (2008{\natexlab{b}}).
\newblock {SpAM}: Sparse additive models.
\newblock In \textit{Advances in {N}eural {I}nformation {P}rocessing {S}ystems
  20}. MIT Press, Cambridge, MA, 1201--1208.

\bibitem[{Ravikumar et~al.(2009)Ravikumar, Wainwright, Raskutti and
  Yu}]{Ravikumar:Gauss:09}
\textsc{Ravikumar, P.}, \textsc{Wainwright, M.}, \textsc{Raskutti, G.} and
  \textsc{Yu, B.} (2009).
\newblock Model selection in {G}aussian graphical models: {H}igh-dimensional
  consistency of $\ell_1$-regularized {MLE}.
\newblock In \textit{Advances in Neural Information Processing Systems 22}. MIT
  Press, Cambridge, MA.

\bibitem[{Rothman et~al.(2008)Rothman, Bickel, Levina and Zhu}]{Rothman:08}
\textsc{Rothman, A.~J.}, \textsc{Bickel, P.~J.}, \textsc{Levina, E.} and
  \textsc{Zhu, J.} (2008).
\newblock Sparse permutation invariant covariance estimation.
\newblock \textit{Electronic Journal of Statistics} \textbf{2} 494--515.

\bibitem[{Sklar(1959)}]{Sklar:59}
\textsc{Sklar, A.} (1959).
\newblock Fonctions de r\'epartition \`a $n$ dimensions et leurs marges.
\newblock \textit{Publications de l'Institut de Statistique de L'Universit\'e
  de Paris 8}  229--231.

\bibitem[{Tibshirani(1996)}]{Tibshirani:96}
\textsc{Tibshirani, R.} (1996).
\newblock Regression shrinkage and selection via the lasso.
\newblock \textit{Journal of the Royal Statistical Society, Series B,
  Methodological} \textbf{58} 267--288.

\bibitem[{Tsukahara(2005)}]{Tsukahara:05}
\textsc{Tsukahara, H.} (2005).
\newblock Semiparametric estimation in copula models.
\newblock \textit{Canadian Journal of Statistics} \textbf{33} 357--375.

\bibitem[{van~der Vaart(1998)}]{vaar:1998}
\textsc{van~der Vaart, A.~W.} (1998).
\newblock \textit{Asymptotic Statistics}.
\newblock Cambridge University Press.

\bibitem[{van~der Vaart and Wellner(1996)}]{van:well:1996}
\textsc{van~der Vaart, A.~W.} and \textsc{Wellner, J.~A.} (1996).
\newblock \textit{Weak Convergence and Empirical Processes: {W}ith Applications
  to Statistics}.
\newblock Springer-Verlag.

\bibitem[{Wille(2004)}]{wille:04}
\textsc{Wille, A.} (2004).
\newblock Sparse {G}aussian graphical modelling of the isoprenoid gene network
  in \it {A}rabidopsis thaliana\rm.
\newblock \textit{Genome Biology} \textbf{5} R92.

\bibitem[{Yuan and Lin(2007)}]{Yuan:Lin:07}
\textsc{Yuan, M.} and \textsc{Lin, Y.} (2007).
\newblock Model selection and estimation in the {G}aussian graphical model.
\newblock \textit{Biometrika} \textbf{94} 19--35.

\end{thebibliography}

\end{document}